\documentclass{article}

\PassOptionsToPackage{numbers, compress}{natbib}
\usepackage[preprint]{neurips_2026}  % uncomment for arXiv

\usepackage[utf8]{inputenc} % allow utf-8 input
\usepackage[T1]{fontenc}    % use 8-bit T1 fonts
\usepackage[colorlinks=true,linkcolor={blue!70!black},citecolor={blue!70!black},urlcolor={blue!70!black},hypertexnames=false]{hyperref}
\usepackage{url}            % simple URL typesetting
\usepackage{booktabs}       % professional-quality tables
\usepackage{amsfonts}       % blackboard math symbols
\usepackage{nicefrac}       % compact symbols for 1/2, etc.
\usepackage{microtype}      % microtypography
\usepackage{xcolor}         % colors
\usepackage{graphicx}       % includegraphics
\usepackage{subcaption}     % subfigures with (a), (b) labels
\usepackage{amsmath,amssymb,amsthm} % math
\usepackage{algorithm}
\usepackage{algorithmic}

\newtheorem{theorem}{Theorem}

\newtheorem{proposition}[theorem]{Proposition}
\newtheorem{corollary}[theorem]{Corollary}
\theoremstyle{definition}

\theoremstyle{remark}

\title{Stochastic Attention via Langevin Dynamics on the Modern Hopfield Energy}

\author{%
  Abdulrahman Alswaidan and Jeffrey D. Varner \\
  R.F. Smith School of Chemical and Biomolecular Engineering\\
  Cornell University, Ithaca, NY 14850 \\
  \texttt{\{aa2725, jdv27\}@cornell.edu} \\
}

\begin{document}

\maketitle

\begin{abstract}
Attention heads retrieve: given a query, they return a weighted average of stored values. We showed that this computation is one step of gradient descent on the modern Hopfield energy, and that Langevin sampling from the corresponding Boltzmann distribution yielded \emph{stochastic attention}, a training-free sampler controlled by a single temperature parameter. Lowering the temperature gave exact retrieval; raising it gave open-ended generation. Because the energy gradient equals the attention map, no score network, training loop, or learned model was required, making the approach particularly suited to the low-data regime where learned generative models are starved of training signal. We derived an entropy inflection condition that identified the retrieval-to-generation transition temperature for any memory geometry and validated the sampler on five domains spanning two orders of magnitude in dimension. A single Boolean mask on the attention softmax, identical to the causal mask used in transformers but applied along the memory axis rather than the sequence axis, turned the sampler into a zero-shot class-conditional generator on Olivetti faces with no retraining and no learned classifier. On MNIST digit images, stochastic attention produced samples that were markedly more novel and more diverse than the best learned baseline while matching a Metropolis-corrected gold standard. On protein sequences from a small Pfam family, the generation regime preserved amino acid composition far more faithfully than a variational autoencoder at matched novelty, indicating that the training-free score function retained family-level fidelity that learned models lost. A denoising diffusion baseline failed across all memory sizes tested, producing samples indistinguishable from isotropic noise. The approach required no architectural changes to the underlying attention mechanism.
\end{abstract}

\section{Introduction}
% introduction.tex -- Introduction
% Three-paragraph structure: (1) problem, (2) landscape, (3) contribution.

Attention is the central computational primitive of modern deep learning~\cite{vaswaniAttentionAllYou2017}. Given a query, it computes a softmax-weighted average of stored values, a powerful operation but a fundamentally \emph{deterministic} one. The same query always produces the same output, so attention retrieves but does not generate. Yet generation from a structured memory is precisely what many downstream tasks require: producing novel but plausible continuations, interpolating between stored prototypes, or exploring the space of patterns consistent with partial evidence. This limitation is sharpest in the low-data regime (small protein families, rare event catalogs, specialized image collections), where the memory bank contains tens to hundreds of examples, far too few to train a score network or diffusion model, yet enough to define a structured energy landscape. A natural question arises: \emph{can the attention mechanism itself be made stochastic, in a principled way, so that it samples from the space of memories rather than merely returning their weighted average?}

The two pieces of an answer have been developed separately and never combined. On one side, modern Hopfield networks~\cite{ramsauerHopfieldNetworksAll2021}, building on classical associative memory~\cite{hopfieldNeuralNetworksPhysical1982} and dense associative memories~\cite{krotovDenseAssociativeMemory2016,krotovLargeAssociativeMemory2021}, have revealed that each attention head implicitly performs gradient descent on a smooth, confining energy whose minima are stored patterns. On the other side, Langevin dynamics~\cite{robertsTweedieExponentialConvergence1996,wellingBayesianLearningStochastic2011,durmusNonasymptoticAnalysisUnadjusted2017} converts any such energy into a sampler for the corresponding Boltzmann distribution by adding calibrated noise to the gradient update. Energy-based models~\cite{lecunTutorialEnergyBased2006,songHowTrainEnergyBased2021} and score-based diffusion methods~\cite{songGenerativeModelingEstimating2019,songScoreBasedGenerativeModeling2021} have demonstrated the power of Langevin-type sampling for generation, but rely on black-box neural networks whose scores must be learned. The Energy Transformer~\cite{hooverEnergyTransformer2024} brings Hopfield energies into a deep architecture, yet remains a discriminative model that descends the energy rather than sampling from it. The classical retrieval-generation duality, in which Hopfield networks retrieve while Boltzmann machines sample over the same energy~\cite{ackleyLearningAlgorithmBoltzmann1985}, has not been lifted to the modern continuous setting.

In this paper, we show that applying the unadjusted Langevin algorithm to the modern Hopfield energy yields a \emph{stochastic attention} (SA) update: a contraction toward the origin, a softmax attention pull toward stored memories, and an isotropic Gaussian perturbation governed by a temperature parameter. The inverse temperature $\beta$ interpolates between exact retrieval ($\beta\to\infty$) and open-ended generation ($\beta\to 0$), requiring no learned score network, no training loop, and no contrastive objective. Because the Hopfield energy gradient is exactly the identity minus the attention map, every step costs the same as a single attention head; in particular, $\mathbf{X}$ can be the key matrix of any pretrained attention layer, making stochastic attention a zero-shot stochastic decoding layer that requires no retraining. The energy's analytic structure (smoothness, Lipschitz gradients, quadratic confinement) delivers convergence guarantees unavailable to generic energy-based models. The contribution lies in the synthesis and the analysis it enables. The closed-form score function eliminates the score matching or contrastive divergence required by standard energy-based and score-based generative models. Beyond this, by analyzing the Hopfield energy \emph{as a sampling target} rather than for retrieval alone, we derive an entropy inflection condition (Proposition~\ref{prop:entropy-derivative}) that identifies the critical $\beta^*$ at which the sampler transitions from retrieval to generation, with a scaling law $\beta^*\sim\sqrt{d}$ yielding a dimension-dependent signal-to-noise ratio (SNR) criterion for temperature selection. Finally, we characterize the smoothness, dissipativity, and condition number of the energy as a function of $\beta$ and $K/d$ for sampling guarantees (Proposition~\ref{prop:regularity}). In practice, the same algorithm realized two operating regimes, \emph{structured retrieval} at high $\beta$ and \emph{genuine generation} at low $\beta$, producing, e.g., $6.9{\times}$ lower amino acid composition divergence than a VAE on protein sequences at matched novelty (Table~\ref{tab:protein-main}), despite using zero training.

\section{Related Work}
% relatedwork.tex -- Related Work
The classical Hopfield network~\cite{hopfieldNeuralNetworksPhysical1982} stores binary patterns as local minima of a quadratic energy; retrieval proceeds by gradient descent but the storage capacity scales only as $\mathcal{O}(d)$~\cite{amitStoringInfiniteNumbers1985}. Dense associative memories~\cite{krotovDenseAssociativeMemory2016,krotovLargeAssociativeMemory2021} replace the quadratic interaction with higher-order polynomials, breaking the linear barrier. At the top of this hierarchy, the log-sum-exp energy of modern Hopfield networks~\cite{ramsauerHopfieldNetworksAll2021} achieves exponential capacity and is equivalent to transformer attention: each head performs one step of gradient descent on the energy. All of these works focus on \emph{retrieval}; sampling from the Boltzmann distribution of the same energy is not explored. The duality between retrieval and generation dates to the Boltzmann machine~\cite{ackleyLearningAlgorithmBoltzmann1985}, which shares the Hopfield energy but samples via Gibbs dynamics over binary states. Our formulation lifts this duality to the continuous modern Hopfield setting where Langevin dynamics scales naturally to high dimensions.

Energy-based models (EBMs) define $p(\mathbf{x})\propto\exp(-E(\mathbf{x}))$ with a neural-network energy trained by contrastive divergence~\cite{hintonTrainingProductsExperts2002}, score matching~\cite{hyvarinen2005estimation}, or noise contrastive estimation~\cite{gutmannHyvaerinenNoiseContrastive2010,songHowTrainEnergyBased2021}. Because the energy is a black box, its gradient must be learned. By contrast, the Hopfield energy has a closed-form gradient (identity minus softmax attention), requires no training, and provides analytic guarantees (smoothness, Lipschitz gradients, quadratic confinement). The Energy Transformer~\cite{hooverEnergyTransformer2024} brings Hopfield energies into a deep architecture for discriminative tasks but does not sample from the Boltzmann distribution. Score-based generative models~\cite{songGenerativeModelingEstimating2019,songScoreBasedGenerativeModeling2021} and diffusion models~\cite{hoDenoisingDiffusionProbabilistic2020,sohlDicksteinDeepUnsupervisedLearning2015} achieve remarkable sample quality but rely on a learned score network; our score is exact and analytic, at the cost of being restricted to a fixed memory $\mathbf{X}$. Normalizing flows~\cite{rezendeMohamed2015} construct generative models via learned invertible transformations but, like diffusion models, require training data and a learned bijection without exposing a closed-form score over a fixed memory. Deng et al.~\cite{dengLatentAlignment2018} introduce stochasticity into attention by replacing softmax with a latent variable trained via variational inference to model alignment uncertainty; our stochasticity arises from Langevin dynamics on a fixed energy with no variational objective.

\section{Background}
% background.tex -- Background / Preliminaries

We build on three ideas whose intersection has not been explored: attention as a computational primitive, its reinterpretation as energy minimization, and Langevin dynamics as the tool that converts any energy landscape into a sampler. The transformer~\cite{vaswaniAttentionAllYou2017} computes attention via queries $\mathbf{Q}$, keys $\mathbf{K}$, and values $\mathbf{V}$:
\begin{equation}\label{eq:sdp-attention}
    \mathrm{Attention}(\mathbf{Q},\mathbf{K},\mathbf{V})
    = \operatorname{softmax}\!\left(\frac{\mathbf{Q}\mathbf{K}^{\top}}{\sqrt{d}}\right)\mathbf{V},
\end{equation}
where $d$ is the key dimension and the softmax is applied row-wise. Each output row is a convex combination of value vectors, with weights determined by query-key similarity. This operation is \emph{deterministic}: given the same query, it always returns the same weighted average. The deterministic input-output map is exactly what limits attention to retrieval; introducing stochasticity in a principled way requires identifying the energy landscape that attention is implicitly descending so that calibrated noise can be added to the gradient. We first identify that landscape, then convert it into a sampler.

The classical Hopfield network~\cite{hopfieldNeuralNetworksPhysical1982} stores $K$ binary patterns $\mathbf{m}_1,\dots,\mathbf{m}_K\in\{-1,+1\}^d$ via the quadratic energy
\begin{equation}\label{eq:classical-hopfield}
    E_{\mathrm{classical}}(\mathbf{s})
    = -\tfrac{1}{2}\,\mathbf{s}^{\top}\mathbf{W}\mathbf{s}
      - \mathbf{b}^{\top}\mathbf{s},
    \qquad
    \mathbf{W} = \frac{1}{d}\sum_{i=1}^{K}\mathbf{m}_i\mathbf{m}_i^{\top},
\end{equation}
where $\mathbf{b}\in\mathbb{R}^d$ is an external bias (typically zero). Retrieval proceeds by coordinate-wise sign updates that descend $E_{\mathrm{classical}}$, but the storage capacity scales only as $\mathcal{O}(d)$~\cite{amitStoringInfiniteNumbers1985}. \textit{Modern} Hopfield networks~\cite{ramsauerHopfieldNetworksAll2021} overcome this bottleneck by replacing the quadratic energy with a log-sum-exp form. Let $\mathbf{X}=[\mathbf{m}_1,\dots,\mathbf{m}_K]\in\mathbb{R}^{d\times K}$ be the memory matrix, $\beta>0$ an inverse temperature, and $M:=\max_i\|\mathbf{m}_i\|_2$. The energy is
\begin{equation}\label{eq:modern-hopfield-energy}
    E(\boldsymbol{\xi})
    = -\operatorname{lse}_\beta\!\bigl(\mathbf{X}^{\top}\boldsymbol{\xi}\bigr)
      + \tfrac{1}{2}\|\boldsymbol{\xi}\|_2^2
      + \tfrac{1}{\beta}\log K
      + \tfrac{1}{2}M^2,
\end{equation}
where $\operatorname{lse}_\beta(\mathbf{z}):=\beta^{-1}\log\bigl(\sum_{i=1}^{K}e^{\beta z_i}\bigr)$. The retrieval map is
\begin{equation}\label{eq:hopfield-update}
    \mathbf{T}(\boldsymbol{\xi})
    := \mathbf{X}\,\operatorname{softmax}\!\bigl(\beta\,\mathbf{X}^{\top}\boldsymbol{\xi}\bigr),
\end{equation}
and $\nabla E(\boldsymbol{\xi})=\boldsymbol{\xi}-\mathbf{T}(\boldsymbol{\xi})$~\cite{ramsauerHopfieldNetworksAll2021}, so the deterministic iteration $\boldsymbol{\xi}^{t+1}=\mathbf{T}(\boldsymbol{\xi}^t)$ is gradient descent on $E$ with unit step size. Ramsauer et al.~\cite{ramsauerHopfieldNetworksAll2021} observed that single-head attention~\eqref{eq:sdp-attention} with $\mathbf{Q}=\boldsymbol{\xi}^{\top}$, $\mathbf{K}=\mathbf{V}=\mathbf{X}^{\top}$, and $\beta=1/\sqrt{d}$ is \emph{exactly one step} of~\eqref{eq:hopfield-update}: every attention head implicitly minimizes~\eqref{eq:modern-hopfield-energy}.

Two properties of this energy are essential and together unlock Langevin sampling. First, $E$ is $C^{\infty}$ with Lipschitz gradient, so standard convergence theory for the unadjusted Langevin algorithm applies and the discretization bias is bounded by an $O(\alpha)$ term controllable through the step size. Second, $E$ is \emph{confining}:
\begin{equation}\label{eq:energy-lower-bound}
    E(\boldsymbol{\xi}) \;\ge\; \tfrac{1}{2}\bigl(\|\boldsymbol{\xi}\|_2 - M\bigr)^2 \;\ge\; 0,
\end{equation}
so $E$ grows at least quadratically as $\|\boldsymbol{\xi}\|_{2}\to\infty$ and iterates cannot escape; the corresponding Boltzmann measure $p_\beta\propto\exp(-\beta E)$ has finite normalization and concentrates near the convex hull of the stored patterns. These two properties together place the Hopfield energy squarely in the regime where well-developed sampling theory delivers convergence guarantees, in contrast to generic energy-based models whose energies are learned and may violate either condition.

Given a smooth potential $U:\mathbb{R}^d\to\mathbb{R}$ with target density $p(\mathbf{x})\propto\exp(-U(\mathbf{x}))$, the overdamped Langevin SDE is
\begin{equation}\label{eq:langevin-sde}
    d\mathbf{x}_t = -\nabla U(\mathbf{x}_t)\,dt + \sqrt{2}\,d\mathbf{B}_t,
\end{equation}
where $\mathbf{B}_t$ is standard $d$-dimensional Brownian motion. This process has $p$ as its unique stationary distribution provided \textbf{(R1)}~$\nabla U$ is Lipschitz and \textbf{(R2)}~$U$ is dissipative, i.e.\ $\langle \nabla U(\mathbf{x}),\mathbf{x}\rangle \ge a\|\mathbf{x}\|_2^2 - b$ for $a>0$, $b\ge 0$~\cite{robertsTweedieExponentialConvergence1996}. The modern Hopfield energy satisfies both conditions. Discretizing~\eqref{eq:langevin-sde} with step size $\alpha>0$ gives the \emph{unadjusted Langevin algorithm} (ULA):
\begin{equation}\label{eq:ula}
    \mathbf{x}_{t+1}
    = \mathbf{x}_t - \alpha\,\nabla U(\mathbf{x}_t) + \sqrt{2\alpha}\;\boldsymbol{\epsilon}_t,
    \quad \boldsymbol{\epsilon}_t\sim\mathcal{N}(\mathbf{0},\mathbf{I}).
\end{equation}
The noise scale $\sqrt{2\alpha}$ maintains the fluctuation-dissipation relation; for finite $\alpha$ the stationary distribution is biased, but the bias vanishes as $\alpha\to 0$~\cite{durmusNonasymptoticAnalysisUnadjusted2017}. Setting $U=\beta E$ yields a Boltzmann target $p_\beta\propto\exp(-\beta E)$ where the temperature $1/\beta$ controls the trade-off between retrieval and generation.

\section{Method}
% method.tex -- Stochastic Attention via Langevin Dynamics on the Hopfield Energy

The preceding section established that attention is gradient descent on the modern Hopfield energy $E$, that $E$ is smooth and confining, and that Langevin dynamics converts any such energy into a sampler. We now combine these facts to derive \emph{stochastic attention}.

\subsection{Stochastic Attention Update}
The memory matrix $\mathbf{X}=[\mathbf{m}_1,\dots,\mathbf{m}_K]\in\mathbb{R}^{d\times K}$ stores $K$ patterns. We sample from the Boltzmann distribution $p_\beta(\boldsymbol{\xi})\propto\exp(-\beta E(\boldsymbol{\xi}))$ by applying the ULA~\eqref{eq:ula} to $U=\beta E$. Using $\nabla E(\boldsymbol{\xi})=\boldsymbol{\xi}-\mathbf{T}(\boldsymbol{\xi})$ and substituting $\tilde\alpha=\alpha/\beta$ so that the effective gradient step size is $\alpha$, we obtain
\begin{equation}\label{eq:stochastic-attention}
    \boldsymbol{\xi}_{t+1}
    = \boldsymbol{\xi}_t
      - \alpha\bigl(\boldsymbol{\xi}_t - \mathbf{T}(\boldsymbol{\xi}_t)\bigr)
      + \sqrt{\tfrac{2\alpha}{\beta}}\;\boldsymbol{\epsilon}_t,
    \qquad \boldsymbol{\epsilon}_t\sim\mathcal{N}(\mathbf{0},\mathbf{I}),
\end{equation}
which, after collecting terms, takes the form
\begin{equation}\label{eq:stochastic-attention-expanded}
    \boxed{\;
    \boldsymbol{\xi}_{t+1}
    = (1-\alpha)\,\boldsymbol{\xi}_t
      + \alpha\,\mathbf{X}\,\operatorname{softmax}\!\bigl(\beta\,\mathbf{X}^{\top}\boldsymbol{\xi}_t\bigr)
      + \sqrt{\tfrac{2\alpha}{\beta}}\;\boldsymbol{\epsilon}_t
    \;}
\end{equation}
This is the \emph{stochastic attention} update: a contraction toward the origin, a softmax-weighted pull toward stored memories, and isotropic Gaussian noise governed by $1/\beta$. The complete procedure is given in Algorithm~\ref{alg:stochastic-attention}. Because $\mathbf{X}$ is fixed, each step requires only the same primitives as a single attention head (two matrix-vector products, one softmax, one Gaussian draw) at cost $\mathcal{O}(dK)$. No training loop, score network, or learned parameters are needed.

\begin{algorithm}[t]
\caption{Stochastic Attention Sampler}\label{alg:stochastic-attention}
\begin{algorithmic}[1]
\REQUIRE Memory matrix $\mathbf{X}\in\mathbb{R}^{d\times K}$, inverse temperature $\beta>0$, step size $\alpha\in(0,1)$, number of iterations $T$, initial state $\boldsymbol{\xi}_0\in\mathbb{R}^d$
\ENSURE Sample trajectory $\boldsymbol{\xi}_0,\boldsymbol{\xi}_1,\dots,\boldsymbol{\xi}_T$
\FOR{$t = 0,1,\dots,T-1$}
    \STATE $\mathbf{a}_t \leftarrow \operatorname{softmax}\!\bigl(\beta\,\mathbf{X}^{\top}\boldsymbol{\xi}_t\bigr)$ \hfill $\triangleright$ attention weights
    \STATE $\boldsymbol{\epsilon}_t \sim \mathcal{N}(\mathbf{0},\mathbf{I}_d)$
    \STATE $\boldsymbol{\xi}_{t+1} \leftarrow (1-\alpha)\,\boldsymbol{\xi}_t + \alpha\,\mathbf{X}\,\mathbf{a}_t + \sqrt{2\alpha/\beta}\;\boldsymbol{\epsilon}_t$
\ENDFOR
\end{algorithmic}
\end{algorithm}

\subsection{Properties and Limiting Behavior}
When $\beta\to\infty$ and noise vanishes, the softmax sharpens to $\arg\max$ and the update reduces to deterministic retrieval (\emph{hard attention}). At finite $\beta$ with zero noise, it yields the standard softmax-weighted average (\emph{soft attention}). With finite $\beta$ and noise, the sampler generates novel patterns shaped by the memory geometry (\emph{stochastic attention}). The step size $\alpha\in(0,1)$ controls discretization fidelity versus mixing speed.

The score function $\nabla\log p_\beta = \beta(\mathbf{T}-\boldsymbol{\xi})$ is known in closed form from a single attention operation; no auxiliary network is needed. The following proposition (proved in Appendix~\ref{app:proof-regularity}) makes the regularity conditions explicit.

\begin{proposition}\label{prop:regularity}
Let $\sigma_{\max}=\|\mathbf{X}\|_{\mathrm{op}}$ denote the largest singular value of the memory matrix. The modern Hopfield energy~\eqref{eq:modern-hopfield-energy} has Lipschitz-continuous gradient with constant $L = 1 + \beta\sigma_{\max}^{2}/2$ and satisfies the dissipativity condition $\langle \nabla E(\boldsymbol{\xi}),\boldsymbol{\xi}\rangle \ge \|\boldsymbol{\xi}\|_2^2/2 - M^2/2$ for all $\boldsymbol{\xi}\in\mathbb{R}^d$.
\end{proposition}

The Hessian satisfies $\nabla^2 E \succeq (1-\beta\sigma_{\max}^{2}/2)\mathbf{I}$, so $E$ is strictly convex when $\beta\sigma_{\max}^{2}<2$, yielding a standard convergence result:

\begin{corollary}\label{cor:convergence}
When $\beta\sigma_{\max}^{2}<2$, the potential $U=\beta E$ is $m$-strongly convex with $m=\beta(1-\beta\sigma_{\max}^{2}/2)$ and has $L_{U}$-Lipschitz gradient with $L_{U}=\beta(1+\beta\sigma_{\max}^{2}/2)$. The iterates of Algorithm~\ref{alg:stochastic-attention} converge to~$p_\beta$ in $W_2$ at a geometric rate, with discretization bias vanishing as $\alpha\to 0$~\cite{dalalyanTheoreticalGuaranteesApproximate2017,durmusNonasymptoticAnalysisUnadjusted2017}.
\end{corollary}

The memory matrices in our experiments place the convex regime well below the operating range $\beta\in\{200,2000\}$. We state this boundary explicitly because it delineates where the Hopfield energy transitions from unimodal (one global basin) to multimodal (one local minimum per stored pattern): the condition $\beta\sigma_{\max}^{2}=2$ is the onset of non-convexity, beyond which energy barriers separate stored memories.

For general non-convex energies, mixing time scales exponentially in the barrier height divided by the temperature~\cite{bovier2004metastability}. In the Hopfield energy at high~$\beta$, barriers between patterns are $O(\beta)$ while the temperature is $1/\beta$, so inter-basin mixing is exponentially slow. This is consistent with our empirical observation that a single chain at $\beta{=}2000$ does not cross between basins (single-chain diversity $0.282$; Table~\ref{tab:single-chain}). However, within each basin the energy is locally strongly convex (the Hessian near each minimum is positive definite), so \emph{intra-basin} mixing is fast. The multi-chain protocol used in our experiments exploits this: by initializing chains near distinct stored patterns, we obtain cross-basin coverage through initialization and rely on fast local mixing within each basin. At $\beta{=}200$, the barriers are low enough that single chains do cross basins (single-chain diversity $0.796$, exceeding the multi-chain $\beta{=}2000$ value of $0.600$; Table~\ref{tab:single-chain}).

The confining bound~\eqref{eq:energy-lower-bound} keeps iterates concentrated near $\mathrm{conv}\{\mathbf{m}_1,\dots,\mathbf{m}_K\}$ with excursions controlled by $\sqrt{2\alpha/\beta}$. Two guarantees hold for all $\beta>0$ regardless of convexity: the continuous-time SDE is geometrically ergodic by dissipativity~\cite{robertsTweedieExponentialConvergence1996}, and our experiments confirm that ULA discretization bias is negligible in practice via the empirical MALA acceptance rate.

\begin{figure}[t!]
\centering
\includegraphics[width=\textwidth]{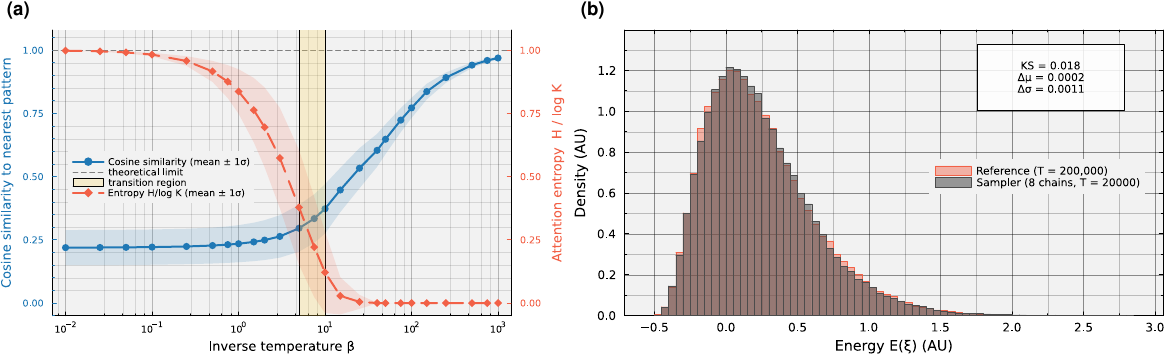}
\caption{Synthetic experiments.
\textbf{(a)}~Phase behavior as a function of inverse temperature $\beta$ ($d=64$, $K=16$). Left axis (blue): mean cosine similarity to the nearest stored pattern; right axis (coral): scaled entropy $H(\mathbf{a})/\log K$. Both diagnostics reveal a smooth transition centered near $\beta\approx 5\text{--}10$ (gold band).
\textbf{(b)}~Convergence validation ($d=8$, $K=4$, $\beta=5$). Pooled energy density from eight independent chains (gray) overlaid on a long-run reference distribution (coral); inset reports the Kolmogorov--Smirnov statistic and moment differences.}
\label{fig:experiments}
\end{figure}

\subsection{Phase Transition and Temperature Selection}\label{sec:phase-transition}

The attention entropy $H(\beta) = -\sum_k p_k \log p_k$, where $p_k = \operatorname{softmax}(\beta\mathbf{X}^{\top}\boldsymbol{\xi})_k$, quantifies selectivity for a given state~$\boldsymbol{\xi}$. At $\beta=0$ the distribution is uniform ($H=\log K$); as $\beta\to\infty$ it concentrates on the nearest memory ($H\to 0$).

\begin{proposition}\label{prop:entropy-derivative}
For fixed $\boldsymbol{\xi}\in\mathbb{R}^d$, let $e_k = \mathbf{m}_k^{\top}\boldsymbol{\xi}$ and $\mathbf{p}=\operatorname{softmax}(\beta\mathbf{e})$. Then
\begin{equation}\label{eq:entropy-deriv}
\frac{dH}{d\beta} \;=\; -\beta\,\mathrm{Var}_{\mathbf{p}}(e),
\end{equation}
where $\mathrm{Var}_{\mathbf{p}}(e) = \sum_k p_k\,e_k^2 - \bigl(\sum_k p_k\,e_k\bigr)^{\!2}$.
\end{proposition}

\begin{proof}
Write $\log p_k = \beta e_k - \log Z$ with $Z=\sum_j\exp(\beta e_j)$, so $H = -\beta\langle e\rangle_{\mathbf{p}} + \log Z$. Differentiating and using $d(\log Z)/d\beta = \langle e\rangle_{\mathbf{p}}$ and $d\langle e\rangle_{\mathbf{p}}/d\beta = \mathrm{Var}_{\mathbf{p}}(e)$ yields the result.
\end{proof}

Since $\mathrm{Var}_{\mathbf{p}}(e)\ge 0$, entropy decreases monotonically. The derivative vanishes at both $\beta=0$ (prefactor) and $\beta\to\infty$ (variance collapse), so the maximum rate of entropy loss occurs at an inflection point $\beta^{*}$ satisfying $d^{2}H/d\beta^{2}=0$:
\begin{equation}\label{eq:inflection}
\mathrm{Var}_{\mathbf{p}}(e)\big|_{\beta^{*}} \;=\; -\beta^{*}\;\mu_{3}\big|_{\beta^{*}},
\end{equation}
where $\mu_3 = \langle(e - \langle e\rangle_{\mathbf{p}})^3\rangle_{\mathbf{p}}$ is the third central moment. This inflection marks the retrieval-to-generation phase boundary: below $\beta^{*}$, attention remains diffuse (generation); above it, attention concentrates on individual memories (retrieval).

For random memories on $\mathbb{S}^{d-1}$, the similarities $e_k$ have $\mathrm{Var}(e_k)=1/d$, so the softmax logits have standard deviation $\beta/\sqrt{d}$; the transition occurs when this reaches order unity, giving $\beta^{*}\sim\sqrt{d}$. Substituting into the per-step signal-to-noise ratio (Eq.~\ref{eq:snr}) yields
\begin{equation}\label{eq:snr-star}
\mathrm{SNR}^{*} \;=\; \sqrt{\frac{\alpha\,\beta^{*}}{2d}} \;\sim\; \sqrt{\frac{\alpha}{2\sqrt{d}}}\,.
\end{equation}
At $d{=}64$ and $\alpha{=}0.01$ this evaluates to $\mathrm{SNR}^{*}=0.025$, matching the empirical transition (Fig.~\ref{fig:experiments}a). For structured data the effective variance may differ from $1/d$, but Eq.~\eqref{eq:inflection} provides a fully data-dependent criterion: evaluate $d^{2}H/d\beta^{2}$ over a $\beta$ sweep and locate the zero crossing.

\begin{table}[t]
\centering
\caption{Quantitative comparison on MNIST digit ``3'' ($K=100$, $d=784$). $\mathcal{N}$ and $\bar{\mathcal{D}}$ are higher-is-better ($\uparrow$); $\bar{E}$ is lower-is-better ($\downarrow$). Values are mean $\pm$ SE across 30 chains. $^\dagger$Energy is positive at $\beta{=}200$ because samples explore off the attractor manifold. $^\ddagger$DDPM produces unstructured noise (max-$\cos=0.062$).}
\label{tab:mnist-baselines}
\small
\begin{tabular}{@{}lccc@{}}
\toprule
Method & $\mathcal{N}$ $\uparrow$ & $\bar{\mathcal{D}}$ $\uparrow$ & $\bar{E}$ $\downarrow$ \\
\midrule
Bootstrap (replay)          & $0.000 \pm 0.000$ & $0.459 \pm 0.011$ & $-0.500 \pm 0.000$ \\
Gaussian perturbation       & $0.004 \pm 0.000$ & $0.450 \pm 0.013$ & $-0.496 \pm 0.000$ \\
Random convex combination   & $0.092 \pm 0.000$ & $0.008 \pm 0.000$ & $-0.399 \pm 0.000$ \\
GMM-PCA ($r{=}50$, $C{=}10$) & $0.198 \pm 0.004$ & $0.419 \pm 0.011$ & $-0.303 \pm 0.005$ \\
VAE (latent${=}8$)          & $0.214 \pm 0.005$ & $0.441 \pm 0.008$ & $-0.286 \pm 0.005$ \\
DDPM ($T{=}200$)            & $0.938 \pm 0.002$ & $0.991 \pm 0.002$ & $+0.44\phantom{0} \pm 0.00\phantom{0}$$^\ddagger$ \\
MALA ($\beta{=}2000$)       & $0.151 \pm 0.001$ & $0.598 \pm 0.001$ & $-0.305 \pm 0.001$ \\
SA ($\beta{=}2000$, retrieval) & $0.152 \pm 0.001$ & $0.600 \pm 0.001$ & $-0.303 \pm 0.001$ \\
\textbf{SA ($\beta{=}200$, generation)} & $\mathbf{0.548 \pm 0.002}$ & $\mathbf{0.885 \pm 0.002}$ & $1.467 \pm 0.008^\dagger$ \\
\bottomrule
\end{tabular}
\end{table}

\subsection{Conditional Sampling via Attention Masking}\label{sec:masking}

The same primitive that powers stochastic attention admits a one-line extension to \emph{conditional} sampling. Let $\mathcal{S}\subseteq\{1,\dots,K\}$ index a subset of memories the user wishes to keep, and define $\mathbf{b}\in(\mathbb{R}\cup\{-\infty\})^{K}$ by $b_{k}=0$ for $k\in\mathcal{S}$ and $b_{k}=-\infty$ otherwise. Replacing the softmax in Algorithm~\ref{alg:stochastic-attention} with $\operatorname{softmax}(\beta\mathbf{X}^{\top}\boldsymbol{\xi}+\mathbf{b})$ yields the masked update
\begin{equation}\label{eq:masked-sa}
\boldsymbol{\xi}_{t+1}
= (1-\alpha)\,\boldsymbol{\xi}_t
  + \alpha\,\mathbf{X}\,\operatorname{softmax}\!\bigl(\beta\,\mathbf{X}^{\top}\boldsymbol{\xi}_t+\mathbf{b}\bigr)
  + \sqrt{\tfrac{2\alpha}{\beta}}\;\boldsymbol{\epsilon}_t.
\end{equation}
Because $b_{k}=-\infty$ exactly zeroes the softmax weight on out-of-set memories, the masked sampler is identical to running unmasked SA on the reduced memory matrix $\mathbf{X}_{\mathcal{S}}$ that retains only the columns indexed by $\mathcal{S}$. Consequently, the regularity guarantees above transfer verbatim with $\mathbf{X}$ replaced by $\mathbf{X}_{\mathcal{S}}$: the energy is smooth and confining, the gradient is Lipschitz with $L=1+\beta\sigma_{\max}(\mathbf{X}_{\mathcal{S}})^{2}/2$, and Corollary~\ref{cor:convergence} applies whenever $\beta\sigma_{\max}(\mathbf{X}_{\mathcal{S}})^{2}<2$. The same operation is the well-known causal mask of transformer attention, here repurposed from sequence-position masking to memory-subset selection: a single boolean vector turns SA into a class-conditional sampler with no retraining and no architectural change. A finite-bias generalization, $b_{k}\in\mathbb{R}$, recovers the Hopfield-multiplicity conditioning of Varner~\cite{varnerConditioningProtein2026}.

\section{Experiments}

\begin{table}[t]
\centering
\caption{Protein sequence generation on Pfam PF00076 ($K{=}68$, $d{=}59$). KL: global amino acid composition divergence. Pos-KL: mean per-position KL. MI~$r$: Pearson correlation of pairwise mutual information matrices. HMM: fraction passing the Pfam RRM hidden Markov model (HMM) at $E$-value~$<0.01$.}
\label{tab:protein-main}
\small
\begin{tabular}{@{}lcccccc@{}}
\toprule
Method & $\mathcal{N}$ $\uparrow$ & SeqID & KL $\downarrow$ & Pos-KL $\downarrow$ & MI~$r$ $\uparrow$ & HMM $\uparrow$ \\
\midrule
Bootstrap (replay)                  & $0.000$ & $0.644$ & $0.143$ & $7.52$ & $0.692$ & $100\%$ \\
VAE (latent${=}8$)                  & $0.621$ & $0.532$ & $0.416$ & $9.99$ & $0.525$ & $100\%$ \\
SA ($\beta{=}77$, retrieval)        & $0.243$ & $0.616$ & $0.107$ & $5.66$ & $0.733$ & $100\%$ \\
\textbf{SA ($\beta{=}8$, generation)} & $\mathbf{0.623}$ & $0.538$ & $\mathbf{0.060}$ & $\mathbf{2.92}$ & $\mathbf{0.871}$ & $100\%$ \\
\bottomrule
\end{tabular}
\end{table}

\paragraph{Synthetic validation.}
We constructed memory matrices with columns drawn uniformly from $\mathbb{S}^{d-1}$ (the standard model in the Hopfield capacity literature~\cite{amitStoringInfiniteNumbers1985}). All experiments used step size $\alpha=0.01$ with fixed random seeds. We first evaluated how $\beta$ controlled the retrieval-to-generation transition (Fig.~\ref{fig:experiments}a): with $K=16$ patterns in $d=64$, both cosine similarity and attention entropy exhibited a smooth sigmoidal transition centered near $\beta\approx 5$--$10$, confirming that $\beta$ acted as a continuous order parameter analogous to the classical Hopfield/Boltzmann transition~\cite{ackleyLearningAlgorithmBoltzmann1985}. The variance of both diagnostics peaked in the transition region, where the chain intermittently switched between pattern-aligned and exploratory states. A convergence experiment on a small system ($d=8$, $K=4$, $\beta=5$; Fig.~\ref{fig:experiments}b) confirmed that eight independent chains converged to the correct Boltzmann target, with pooled energies matching a long-chain reference. A load-ratio phase diagram (Fig.~\ref{fig:phase-diagram}) over $K/d\in[0.05,2.0]$ and $\beta\in[0.5,100]$ revealed a clear phase boundary: retrieval at high $\beta$ and low $K/d$, with the boundary shifting rightward as $\beta$ increased, consistent with classical capacity theory.

% ---------- Experiment 4: MNIST Baseline Comparison ----------
\paragraph{Image generation on MNIST.}
We selected $K=100$ images of the digit ``3'' from MNIST~\cite{lecunGradientBasedLearningApplied1998}, flattened to $\mathbb{R}^{784}$ and $\ell_2$-normalized. Because the energy landscape at high $\beta$ is deeply multimodal, we launched 30 independent chains, each initialized near a different stored pattern ($\sigma_{\mathrm{init}}=0.01$), run at $\beta=2000$ with $\alpha=0.01$ for $T=5{,}000$ iterations (burn-in $2{,}000$, thinning every 100th step), retaining $150$ samples spanning 30 basins. The choice $\beta=2000$ follows from the SNR analysis: the per-step signal-to-noise ratio
\begin{equation}
\mathrm{SNR}
\;=\;\sqrt{\frac{\alpha\beta}{2d}}\,,
\label{eq:snr}
\end{equation}
undergoes a critical transition at $\mathrm{SNR}^{*}=\sqrt{\alpha/(2\sqrt{d})}$ (Eq.~\ref{eq:snr-star}). At $d{=}64$ and $\alpha{=}0.01$ this gives $\mathrm{SNR}^{*}=0.025$ ($\beta^{*}\approx 8$), consistent with the observed transition; we set $\beta=2000$ ($\mathrm{SNR}=0.113$) for structured retrieval.

We compared against seven baselines using the same memory matrix: bootstrap resampling, Gaussian perturbation ($\sigma=\sqrt{2\alpha/\beta}$), random convex combination ($\mathbf{w}\sim\mathrm{Dirichlet}(1,\dots,1)$), Gaussian mixture model with PCA preprocessing (GMM-PCA; 50 components, 10 clusters; Appendix~\ref{app:gmm-pca}), variational autoencoder (VAE; latent dim 8; Appendix~\ref{app:vae}), denoising diffusion probabilistic model (DDPM; MLP-based, $T{=}200$ steps; Appendix~\ref{app:ddpm}), and Metropolis-adjusted Langevin algorithm (MALA; same Langevin proposal with accept/reject correction). For each method we generated 150 samples and evaluated novelty~$\mathcal{N}$, diversity~$\bar{\mathcal{D}}$, and mean energy~$\bar{E}$:
\begin{align*}
\mathcal{N}(\hat{\boldsymbol{\xi}}) = 1 - \max_{k}\cos\!\bigl(\hat{\boldsymbol{\xi}},\,\mathbf{m}_{k}\bigr),\qquad
\bar{\mathcal{D}} = \frac{2}{S(S{-}1)}\!\sum_{i<j}\bigl(1-\cos(\hat{\boldsymbol{\xi}}_i,\hat{\boldsymbol{\xi}}_j)\bigr),\qquad
\bar{E} = \frac{1}{S}\sum_{i=1}^{S} E\!\bigl(\hat{\boldsymbol{\xi}}_i\bigr),
\end{align*}
where $\cos(\mathbf{x},\mathbf{y})=\mathbf{x}^{\top}\mathbf{y}/(\lVert\mathbf{x}\rVert\,\lVert\mathbf{y}\rVert)$. Novelty measures departure from stored patterns; diversity is mean pairwise cosine distance; energy measures proximity to the memory manifold.

Each non-Langevin baseline failed on at least one axis (Table~\ref{tab:mnist-baselines}, Fig.~\ref{fig:mnist-baselines}): bootstrap by construction ($\mathcal{N}{=}0$), Gaussian perturbation by negligible noise, and random convex combination by mode collapse ($\bar{\mathcal{D}}{=}0.008$). DDPM failed entirely, producing samples indistinguishable from isotropic noise (max-$\cos{=}0.062$); a scaling study across $K\in\{100,\dots,3{,}500\}$ confirmed this failure persisted at all memory sizes tested (Appendix~\ref{app:scaling}). The VAE was the strongest learned baseline ($\mathcal{N}{=}0.214$, $\bar{\mathcal{D}}{=}0.441$). At $\beta{=}200$ (generation regime), SA dominated on both metrics: novelty $0.548$ ($2.6{\times}$ VAE), diversity $0.885$ ($2.0{\times}$ VAE). MALA's $99.2\%$ acceptance rate at $\alpha{=}0.01$ confirmed negligible ULA bias; a step-size sweep showed divergence only at $\alpha{\gtrsim}0.1$ (Appendix~\ref{app:stepsize-sweep}). The high-novelty $\beta{=}200$ samples reflected structured generation rather than high-temperature noise: a matched Gaussian control retained only $33\%$ cosine similarity to the nearest stored pattern vs.\ $45\%$ for SA (Table~\ref{tab:noise-control}), and basin-crossing was genuine at $\beta{=}200$, where single-chain diversity ($0.796$) exceeded the multi-chain $\beta{=}2000$ value (Table~\ref{tab:single-chain}). The advantage over the VAE was not an artifact of limited training data either: VAEs trained on $10{\times}$ and $100{\times}$ more images both fell short of SA on max-cosine similarity ($0.830$ vs.\ $0.848$) and diversity ($0.325$ vs.\ $0.600$) (Table~\ref{tab:full-vae}). The SNR transition band (Proposition~\ref{prop:entropy-derivative}) thus delimited two regimes in the same algorithm: \emph{structured retrieval} ($\beta{=}2000$) and \emph{genuine generation} ($\beta{=}200$).

\begin{figure}[!t]
\centering
\includegraphics[width=0.21\textwidth]{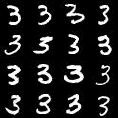}%
\hspace{0.02\textwidth}%
\includegraphics[width=0.21\textwidth]{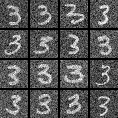}%
\hspace{0.02\textwidth}%
\includegraphics[width=0.21\textwidth]{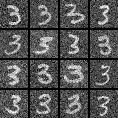}

\small
\makebox[0.21\textwidth]{\centering (a) Warm-start (stored ``3'')}%
\hspace{0.02\textwidth}%
\makebox[0.21\textwidth]{\centering (b) SA endpoint}%
\hspace{0.02\textwidth}%
\makebox[0.21\textwidth]{\centering (c) MALA endpoint}
\caption{Paired before/after grids on MNIST digit ``3'' at $\beta{=}2000$. \textbf{Each cell is one independent chain; cell $(i,j)$ in panels (a, b, c) all share the same warm-start and the same seed.} Compare cell-by-cell \emph{across} panels, not within a row. (a) Stored ``3'' used as the warm-start. (b) SA chain endpoint after $T{=}5{,}000$ Langevin steps from that warm-start; digit identity is preserved while the textured background is the diffusive perturbation producing novelty (cf.\ Table~\ref{tab:mnist-baselines}). (c) MALA endpoint from the same warm-start and seed; visually identical to (b), so the Metropolis correction is unnecessary at $\alpha{=}0.01$. Bootstrap/Gaussian/Convex have no warm-start interpretation; see Table~\ref{tab:mnist-baselines} for numerical comparison.}
\label{fig:mnist-baselines}
\end{figure}

% ---------- Experiment 5: Protein Sequence Generation ----------
\paragraph{Protein sequence generation.}
We applied SA to protein sequences from the Pfam RNA Recognition Motif (RRM) family~\cite{mistryPfamProteinFamilies2021} (PF00076): $K{=}68$ aligned sequences, $L{=}71$ positions, one-hot encoded ($\mathbb{R}^{1420}$), PCA-reduced to $d{=}59$ (95\% variance retained), and $\ell_2$-normalized; the inverse PCA followed by per-position $\arg\max$ decoded chain endpoints back to amino acid sequences. This family size is typical of the low-data regime in which most Pfam families reside~\cite{finnPfamProteinFamilies2008}. The entropy inflection yielded $\beta^*{=}3.85$, half the random-pattern prediction $\sqrt{d}{=}7.68$, because conserved residues increase the effective similarity variance (Fig.~\ref{fig:protein-entropy}). We ran SA at $\beta{=}77$ ($20\beta^*$, retrieval) and $\beta{=}8$ ($2\beta^*$, generation) with the same 30-chain protocol.

At generation temperature, SA matched the VAE on novelty ($0.623$ vs.\ $0.621$) while achieving $6.9{\times}$ lower amino acid composition divergence (KL$=0.060$ vs.\ $0.416$), $3.4{\times}$ lower per-position KL ($2.92$ vs.\ $9.99$), and $66\%$ higher co-evolutionary coupling fidelity (MI $r{=}0.871$ vs.\ $0.525$) (Table~\ref{tab:protein-main}). The per-position metric is important because global composition KL alone cannot distinguish a method that preserves the correct amino acid at each position from one that merely matches overall frequencies; the $3.4{\times}$ gap confirmed that SA preserved position-specific conservation patterns that the VAE lost. All generated sequences passed the Pfam RRM HMM at $E$-value~$<0.01$ (HMMER~3.4~\cite{hmmer3}; median $8.1{\times}10^{-34}$), confirming that generated sequences were unambiguously recognized as RRM family members by an independent domain classifier. MALA at both temperatures produced near-identical results (acceptance $99.8\%$; $\Delta\mathrm{KL}<0.002$; Table~\ref{tab:mala-generation}). Full results including all baselines and the per-position analysis are in Appendix~\ref{app:protein}. Because the Langevin gradient $\beta(\mathbf{T}(\boldsymbol{\xi})-\boldsymbol{\xi})$ is derived from the stored patterns at every step, stochastic attention preserved family-level structure at both position-specific and pairwise coupling levels even as samples explored far from individual sequences, while the VAE, which must learn the distribution from $K{=}68$ examples, could not maintain this fidelity when generating novel sequences. The temperature knob traded sequence identity for novelty ($0.243\to 0.623$) while \emph{improving} compositional fidelity (KL: $0.107\to 0.060$), a counterintuitive result explained by the fact that at lower $\beta$ the attention distribution spreads across more memories, producing samples that better reflect the family consensus rather than individual sequence idiosyncrasies. The observation that $\beta^*{=}3.85$ is half the random-pattern prediction $\sqrt{d}{=}7.68$ also has biological meaning: conserved residues in the RRM motif increase the effective variance of query-key similarities, shifting the phase boundary to lower temperature. Concurrent work~\cite{varnerTrainingFreeProtein2026} extends this pipeline to eight Pfam families and validates generated sequences structurally via ESMFold and AlphaFold2, finding that SA-generated sequences fold more accurately to canonical family structures than natural members in six of eight families tested. Additional MNIST digits and S\&P~500 log-returns at $d{=}424$ replicated the same ranking (Appendices~\ref{app:multidigit},~\ref{app:finance}); Olivetti faces at $d{=}4{,}096$ are covered next.

% ---------- Conditional generation via attention masking ----------
\paragraph{Conditional generation by masking the attention softmax.}
The masked update of Eq.~\eqref{eq:masked-sa} costs nothing beyond the unmasked sampler yet enables zero-shot subject-conditional generation. We applied SA to the Olivetti faces dataset (10 subjects, 10 portraits each, $K{=}100$, $d{=}4{,}096$, $\ell_2$-normalized; Appendix~\ref{app:olivetti}), which provides ground-truth subject labels for direct conditioning measurement. Without a mask, samples warm-started uniformly across all subjects and run at $\beta_{\mathrm{chain}}{=}200$ followed by a sharp readout at $\beta_{\mathrm{read}}{=}10{,}000$ landed on the target subject only $10.4\%$ of the time, indistinguishable from the random-guess rate of $10.0\%$. Setting $b_{k}{=}-\infty$ for memories outside a chosen subject lifted the same procedure to $96.0\%$ subject recovery, a $9.2{\times}$ improvement (Fig.~\ref{fig:olivetti}), with no learned classifier and no retraining. Conventional class-conditional generation requires either a classifier-guidance term added to the score function~\cite{songScoreBasedGenerativeModeling2021} or class-labeled training; the mask achieves the same effect with one Boolean vector and no learned components. The mask is the same Boolean operation used for causal attention in transformers, here applied along the memory axis rather than the sequence axis.

\begin{figure}[!t]
\centering
\includegraphics[width=0.235\textwidth]{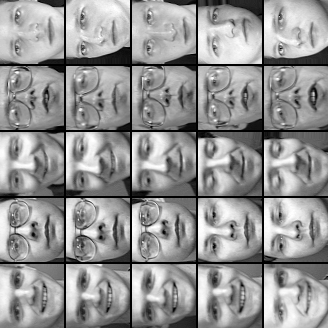}%
\hspace{0.01\textwidth}%
\includegraphics[width=0.235\textwidth]{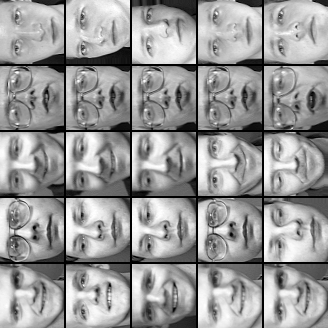}%
\hspace{0.01\textwidth}%
\includegraphics[width=0.235\textwidth]{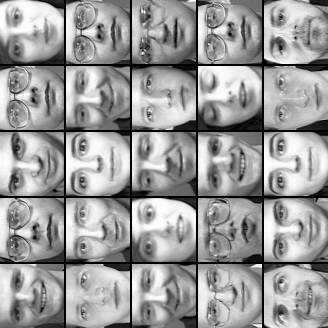}

\small
\makebox[0.235\textwidth]{\centering (a) Stored memories}%
\hspace{0.01\textwidth}%
\makebox[0.235\textwidth]{\centering (b) Hard-masked SA}%
\hspace{0.01\textwidth}%
\makebox[0.235\textwidth]{\centering (c) Unconditional SA}
\caption{Paired before/after grids on Olivetti faces ($K{=}100$, $d{=}4{,}096$). \textbf{Cell $(i,j)$ in all three panels uses the same warm-start: the $j$-th stored portrait of subject $i$.} (a) The 25 stored portraits used as warm-starts. (b) Hard-masked SA endpoints ($b_{k}{=}-\infty$ outside the row's subject); each row stays inside its subject (compare cell-by-cell to (a)). (c) Unconditional SA endpoints from the SAME warm-starts, mask removed; chains drift to other subjects within a row. Hard-mask subject-recovery: $96.0\%$ vs.\ unconditional $10.4\%$ vs.\ random-guess $10.0\%$.}
\label{fig:olivetti}
\end{figure}

\section{Discussion and Conclusion}
% discussion.tex -- Discussion and Conclusion (merged to save space)

The experiments validated the core claim: Langevin dynamics on the modern Hopfield energy converts deterministic attention into a principled stochastic sampler governed by a single temperature parameter. The synthetic experiments confirmed a smooth phase transition between retrieval and generation (Fig.~\ref{fig:experiments}a), convergence to the correct Boltzmann target (Fig.~\ref{fig:experiments}b), and graceful degradation with load ratio (Fig.~\ref{fig:phase-diagram}). The strongest result came from protein sequences (Table~\ref{tab:protein-main}), where the training-free score function preserved family-level fidelity at position-specific and co-evolutionary levels that the VAE could not maintain from $K{=}68$ examples. On MNIST, SA led the best learned baseline by $2.6{\times}$ on novelty and $2.0{\times}$ on diversity, while DDPM failed across all memory sizes tested (Table~\ref{tab:mnist-baselines}). The entropy inflection (Proposition~\ref{prop:entropy-derivative}) located the retrieval-to-generation boundary from data alone, correctly predicting $\beta^*{=}3.85$ on RRM (Fig.~\ref{fig:protein-entropy}), and the masking extension lifted SA to $96.0\%$ subject-conditional accuracy on Olivetti from a $10.4\%$ unconditional baseline (Table~\ref{tab:olivetti-recovery}).

A deliberate choice shapes our experimental design: we evaluate at memory sizes $K$ from 68 to 3{,}500, the low-to-moderate data regime where learned generative models are most constrained, matching most real low-data domains. Score-based and diffusion models need enough data to learn an accurate score function, ill-posed at small $K$ as DDPM's failure illustrates; stochastic attention sidesteps this because $\nabla\log p_\beta = \beta(\mathbf{T}-\boldsymbol{\xi})$ is exact for any $K$, reusing standard attention primitives with no score network or training loop. The Hopfield energy gives convex-regime convergence (Corollary~\ref{cor:convergence}) and geometric ergodicity for all $\beta>0$. The same score function extends to soft-bias conditioning ($b_{k}\in\mathbb{R}$), recovering the multiplicity-weight conditioning of Varner~\cite{varnerConditioningProtein2026}; concurrent work covers eight Pfam families~\cite{varnerTrainingFreeProtein2026} and synthetic patient generation~\cite{varnerSyntheticPatient2026}, with integration into pretrained multi-head attention left to future work (Appendix~\ref{app:multihead}).

\paragraph{Limitations.}
ULA discretization bias scales with $\alpha$; MALA (Appendix~\ref{app:mala-algorithm}) is preferred at larger steps, though $99.2\%$ acceptance at $\alpha{=}0.01$ confirms negligible bias here (Table~\ref{tab:mala-generation}). At high $\beta$ SA produces structured interpolations rather than crisp novel images, and inter-basin mixing grows exponentially with barrier height; multi-chain sampling mitigates this in practice, but quantitative scaling laws remain open. The method requires $\mathbf{X}$ fixed; streaming and latent-memory settings need extension.

% \section*{Acknowledgments}
% Removed for anonymous submission. Restore for camera-ready.

\clearpage

\bibliographystyle{unsrtnat}
\bibliography{References_v1}

\clearpage

\newpage
\appendix
\setcounter{figure}{0}
\setcounter{table}{0}
\renewcommand{\thefigure}{S\arabic{figure}}
\renewcommand{\thetable}{S\arabic{table}}
% appendix.tex -- Supplementary material

\section{Broader Impact Statement}\label{app:broader-impact}
This work is primarily methodological: it connects two well-studied mathematical frameworks (modern Hopfield networks and Langevin dynamics) to produce a stochastic attention mechanism. Because the method generates outputs that are structured combinations of stored memory patterns, it inherits the biases present in whatever memory matrix is supplied. Practitioners should therefore audit the stored patterns for representational harms before deploying the sampler in applications that affect individuals (e.g., image synthesis, recommendation). The training-free nature of the method lowers the barrier to misuse relative to large diffusion models, but also limits expressiveness to the convex hull of the memory, reducing the risk of generating entirely novel harmful content. We do not foresee direct negative societal consequences beyond those common to all generative modeling research, and we encourage responsible use.

\section{Proof of Proposition~\ref{prop:regularity}}\label{app:proof-regularity}

We prove the two regularity properties of the modern Hopfield energy $E$ defined in~\eqref{eq:modern-hopfield-energy}. Write $\mathbf{p}(\boldsymbol{\xi})=\operatorname{softmax}(\beta\mathbf{X}^{\top}\boldsymbol{\xi})\in\mathbb{R}^K$ for the attention weight vector at state $\boldsymbol{\xi}$, and recall the retrieval map $\mathbf{T}(\boldsymbol{\xi})=\mathbf{X}\mathbf{p}(\boldsymbol{\xi})$.

\paragraph{Hessian computation.}
The gradient is $\nabla E(\boldsymbol{\xi})=\boldsymbol{\xi}-\mathbf{T}(\boldsymbol{\xi})$~\cite{ramsauerHopfieldNetworksAll2021}. To obtain the Hessian we differentiate the retrieval map. Let $\mathbf{z}=\beta\mathbf{X}^{\top}\boldsymbol{\xi}\in\mathbb{R}^K$ denote the pre-softmax logits. The Jacobian of the softmax is the $K\times K$ matrix
\begin{equation}\label{eq:softmax-jacobian}
    \frac{\partial\mathbf{p}}{\partial\mathbf{z}}
    = \operatorname{diag}(\mathbf{p}) - \mathbf{p}\mathbf{p}^{\top}
    =: \mathbf{S}(\mathbf{p}),
\end{equation}
which is the covariance matrix of a categorical distribution with probabilities~$\mathbf{p}$. By the chain rule ($\partial\mathbf{z}/\partial\boldsymbol{\xi}=\beta\mathbf{X}^{\top}$):
\begin{equation}\label{eq:hessian}
    \nabla^{2}E(\boldsymbol{\xi})
    = \mathbf{I}_{d} - \beta\,\mathbf{X}\,\mathbf{S}\bigl(\mathbf{p}(\boldsymbol{\xi})\bigr)\,\mathbf{X}^{\top}.
\end{equation}

\paragraph{Part (i): Lipschitz gradient.}
The Lipschitz constant of $\nabla E$ equals $\sup_{\boldsymbol{\xi}}\|\nabla^{2}E(\boldsymbol{\xi})\|_{\mathrm{op}}$. Because $\mathbf{S}(\mathbf{p})$ is a covariance matrix it is positive semidefinite, so $\beta\mathbf{X}\mathbf{S}\mathbf{X}^{\top}\succeq\mathbf{0}$. By the triangle inequality for the spectral norm,
\begin{equation}\label{eq:lip-triangle}
    \|\nabla^{2}E(\boldsymbol{\xi})\|_{\mathrm{op}}
    \;\le\; \|\mathbf{I}_{d}\|_{\mathrm{op}}
            + \beta\,\|\mathbf{X}\,\mathbf{S}(\mathbf{p})\,\mathbf{X}^{\top}\|_{\mathrm{op}}
    \;\le\; 1 + \beta\,\sigma_{\max}^{2}\,\|\mathbf{S}(\mathbf{p})\|_{\mathrm{op}},
\end{equation}
where the second step uses submultiplicativity of the spectral norm: $\|\mathbf{A}\mathbf{B}\mathbf{C}\|_{\mathrm{op}}\le\|\mathbf{A}\|_{\mathrm{op}}\|\mathbf{B}\|_{\mathrm{op}}\|\mathbf{C}\|_{\mathrm{op}}$.

It remains to bound $\|\mathbf{S}(\mathbf{p})\|_{\mathrm{op}}$. For any unit vector $\mathbf{v}\in\mathbb{R}^K$,
\begin{align*}
\mathbf{v}^{\top}\mathbf{S}(\mathbf{p})\mathbf{v}
= \sum_{i}p_{i}v_{i}^{2} - \Bigl(\sum_{i}p_{i}v_{i}\Bigr)^{\!2}
= \operatorname{Var}_{\mathbf{p}}(V),
\end{align*}
where $V$ is the random variable taking value $v_i$ with probability $p_i$. By Popoviciu's variance inequality, $\operatorname{Var}(V)\le(v_{\max}-v_{\min})^{2}/4$. Under the constraint $\|\mathbf{v}\|_{2}=1$, the difference $v_{\max}-v_{\min}$ satisfies
\begin{align*}
(v_{\max}-v_{\min})^{2}
\;\le\; (|v_{\max}|+|v_{\min}|)^{2}
\;\le\; 2(v_{\max}^{2}+v_{\min}^{2})
\;\le\; 2\|\mathbf{v}\|_{2}^{2} = 2,
\end{align*}
where the second step is the Cauchy--Schwarz (QM-AM) inequality $(a+b)^{2}\le 2(a^{2}+b^{2})$ for $a,b\ge 0$. Therefore
\begin{equation}\label{eq:softmax-spectral-bound}
    \|\mathbf{S}(\mathbf{p})\|_{\mathrm{op}}
    = \sup_{\|\mathbf{v}\|=1}\operatorname{Var}_{\mathbf{p}}(V)
    \;\le\; \frac{2}{4} = \frac{1}{2},
\end{equation}
and this bound is tight: equality holds when $K\ge 2$ with $\mathbf{p}=(\tfrac{1}{2},\tfrac{1}{2},0,\dots,0)$ and $\mathbf{v}=(\tfrac{1}{\sqrt{2}},-\tfrac{1}{\sqrt{2}},0,\dots,0)$.

Substituting~\eqref{eq:softmax-spectral-bound} into~\eqref{eq:lip-triangle} yields
\begin{align*}
\sup_{\boldsymbol{\xi}}\|\nabla^{2}E(\boldsymbol{\xi})\|_{\mathrm{op}}
\;\le\; 1 + \frac{\beta\,\sigma_{\max}^{2}}{2}
=: L.
\end{align*}

\paragraph{Part (ii): Dissipativity.}
Expanding the inner product,
\begin{align*}
\langle\nabla E(\boldsymbol{\xi}),\,\boldsymbol{\xi}\rangle
= \|\boldsymbol{\xi}\|_{2}^{2} - \langle\mathbf{T}(\boldsymbol{\xi}),\,\boldsymbol{\xi}\rangle.
\end{align*}
The retrieval map $\mathbf{T}(\boldsymbol{\xi})=\mathbf{X}\mathbf{p}=\sum_{k}p_{k}\mathbf{m}_{k}$ is a convex combination of stored patterns, so by the triangle inequality
\begin{align*}
\|\mathbf{T}(\boldsymbol{\xi})\|_{2}
= \Bigl\|\sum_{k}p_{k}\mathbf{m}_{k}\Bigr\|_{2}
\le \sum_{k}p_{k}\|\mathbf{m}_{k}\|_{2}
\le M,
\end{align*}
where $M=\max_{k}\|\mathbf{m}_{k}\|_{2}$ and $\sum_{k}p_{k}=1$. By the Cauchy--Schwarz inequality,
\begin{align*}
|\langle\mathbf{T}(\boldsymbol{\xi}),\,\boldsymbol{\xi}\rangle|
\le \|\mathbf{T}(\boldsymbol{\xi})\|_{2}\,\|\boldsymbol{\xi}\|_{2}
\le M\|\boldsymbol{\xi}\|_{2}.
\end{align*}
Applying Young's inequality $ab\le\tfrac{a^{2}}{2}+\tfrac{b^{2}}{2}$ with $a=M$ and $b=\|\boldsymbol{\xi}\|_{2}$,
\begin{align*}
M\|\boldsymbol{\xi}\|_{2}
\;\le\; \frac{M^{2}}{2} + \frac{\|\boldsymbol{\xi}\|_{2}^{2}}{2}.
\end{align*}
Hence
\begin{align*}
\langle\nabla E(\boldsymbol{\xi}),\,\boldsymbol{\xi}\rangle
\;\ge\; \|\boldsymbol{\xi}\|_{2}^{2} - \frac{M^{2}}{2} - \frac{\|\boldsymbol{\xi}\|_{2}^{2}}{2}
= \frac{1}{2}\|\boldsymbol{\xi}\|_{2}^{2} - \frac{M^{2}}{2},
\end{align*}
which is the dissipativity condition \textbf{(R2)} with constants $a=\tfrac{1}{2}$ and $b=\tfrac{M^{2}}{2}$. \qed

\paragraph{Convexity threshold (Corollary~\ref{cor:convergence}).}
Since $\mathbf{S}(\mathbf{p})\succeq\mathbf{0}$, the Hessian~\eqref{eq:hessian} satisfies $\nabla^{2}E\preceq\mathbf{I}$ (the identity provides an upper bound on the eigenvalues). For the lower bound, $\beta\mathbf{X}\mathbf{S}\mathbf{X}^{\top}\preceq\tfrac{\beta\sigma_{\max}^{2}}{2}\mathbf{I}$ by~\eqref{eq:softmax-spectral-bound} and submultiplicativity, so
\begin{align*}
\nabla^{2}E(\boldsymbol{\xi})
\;\succeq\; \Bigl(1-\frac{\beta\sigma_{\max}^{2}}{2}\Bigr)\mathbf{I}.
\end{align*}
When $\beta\sigma_{\max}^{2}<2$ the lower bound is strictly positive, so $E$ is strictly convex. The potential $U=\beta E$ is then $m$-strongly convex with $m=\beta(1-\beta\sigma_{\max}^{2}/2)$ and has $L_{U}$-Lipschitz gradient with $L_{U}=\beta(1+\beta\sigma_{\max}^{2}/2)$. The condition number is
\begin{align*}
\kappa = \frac{L_{U}}{m}
= \frac{1+\beta\sigma_{\max}^{2}/2}{1-\beta\sigma_{\max}^{2}/2},
\end{align*}
which diverges as $\beta\sigma_{\max}^{2}\to 2$, reflecting the onset of non-convexity and the proliferation of local minima in the energy landscape.

\section{MALA Variant of Stochastic Attention}\label{app:mala-algorithm}

Algorithm~\ref{alg:stochastic-attention} uses the Unadjusted Langevin Algorithm (ULA), which is simple but introduces an $O(\alpha)$ discretization bias. The Metropolis-Adjusted Langevin Algorithm (MALA) eliminates this bias by appending an accept/reject step to each Langevin proposal. At the step size and temperature used in our experiments ($\alpha{=}0.01$, $\beta{=}2000$), the two algorithms produced practically indistinguishable outputs (Table~\ref{tab:mnist-baselines}); however, the correction becomes important when a larger step size is desired, e.g., to accelerate mixing in high-dimensional or low-temperature settings where $\alpha$ must be increased beyond the regime in which the ULA bias is negligible.

\begin{algorithm}[h]
\caption{MALA Stochastic Attention Sampler}\label{alg:mala}
\begin{algorithmic}[1]
\REQUIRE Memory matrix $\mathbf{X}\in\mathbb{R}^{d\times K}$, inverse temperature $\beta>0$, step size $\alpha\in(0,1)$, iterations $T$, initial state $\boldsymbol{\xi}_0\in\mathbb{R}^d$
\ENSURE Sample trajectory $\boldsymbol{\xi}_0,\boldsymbol{\xi}_1,\dots,\boldsymbol{\xi}_T$
\FOR{$t = 0,1,\dots,T-1$}
    \STATE $\mathbf{a}_t \leftarrow \operatorname{softmax}\!\bigl(\beta\,\mathbf{X}^{\top}\boldsymbol{\xi}_t\bigr)$ \hfill $\triangleright$ attention weights at current state
    \STATE $\boldsymbol{\mu}_t \leftarrow (1-\alpha)\,\boldsymbol{\xi}_t + \alpha\,\mathbf{X}\,\mathbf{a}_t$ \hfill $\triangleright$ Langevin proposal mean
    \STATE $\boldsymbol{\epsilon}_t \sim \mathcal{N}(\mathbf{0},\mathbf{I}_d)$
    \STATE $\boldsymbol{\xi}^{\star} \leftarrow \boldsymbol{\mu}_t + \sqrt{2\alpha/\beta}\;\boldsymbol{\epsilon}_t$ \hfill $\triangleright$ candidate state
    \STATE $\mathbf{a}^{\star} \leftarrow \operatorname{softmax}\!\bigl(\beta\,\mathbf{X}^{\top}\boldsymbol{\xi}^{\star}\bigr)$ \hfill $\triangleright$ attention weights at candidate
    \STATE $\boldsymbol{\mu}^{\star} \leftarrow (1-\alpha)\,\boldsymbol{\xi}^{\star} + \alpha\,\mathbf{X}\,\mathbf{a}^{\star}$ \hfill $\triangleright$ reverse proposal mean
    \STATE $\log r \leftarrow -\beta\bigl[E(\boldsymbol{\xi}^{\star}) - E(\boldsymbol{\xi}_t)\bigr]
             - \frac{\beta}{4\alpha}\bigl[\lVert\boldsymbol{\xi}_t - \boldsymbol{\mu}^{\star}\rVert^2
             - \lVert\boldsymbol{\xi}^{\star} - \boldsymbol{\mu}_t\rVert^2\bigr]$
    \STATE $u \sim \mathrm{Uniform}(0,1)$
    \IF{$\log u < \min(0,\,\log r)$}
        \STATE $\boldsymbol{\xi}_{t+1} \leftarrow \boldsymbol{\xi}^{\star}$ \hfill $\triangleright$ accept
    \ELSE
        \STATE $\boldsymbol{\xi}_{t+1} \leftarrow \boldsymbol{\xi}_t$ \hfill $\triangleright$ reject
    \ENDIF
\ENDFOR
\end{algorithmic}
\end{algorithm}

Compared with Algorithm~\ref{alg:stochastic-attention}, MALA replaces the single-line state update (line~4) with a propose-then-correct cycle (lines~2--13). The additional cost per step is one extra attention evaluation (line~6) and one energy evaluation (line~8), doubling the per-iteration $\mathcal{O}(dK)$ cost. In the limit $\alpha\to 0$ both algorithms coincide; at moderate $\alpha$ the accept/reject step removes $O(\alpha)$ discretization bias.

The ULA update in Algorithm~\ref{alg:stochastic-attention} draws each new state from a Gaussian proposal centered on the Langevin mean $\boldsymbol{\mu}_t$:
\begin{equation}\label{eq:mala-forward}
q(\boldsymbol{\xi}^{\star}\mid\boldsymbol{\xi}_t)
= \mathcal{N}\!\bigl(\boldsymbol{\xi}^{\star};\;\boldsymbol{\mu}_t,\;\tfrac{2\alpha}{\beta}\mathbf{I}\bigr),
\end{equation}
where $\boldsymbol{\mu}_t = (1{-}\alpha)\boldsymbol{\xi}_t + \alpha\,\mathbf{X}\mathbf{a}_t$ encodes one step of gradient descent on $E$.
For the chain to sample from the correct Boltzmann target $p_\beta\propto\exp(-\beta E)$, the transition kernel must satisfy \emph{detailed balance}: $p_\beta(\boldsymbol{\xi})\,q(\boldsymbol{\xi}^{\star}\mid\boldsymbol{\xi}) = p_\beta(\boldsymbol{\xi}^{\star})\,q(\boldsymbol{\xi}\mid\boldsymbol{\xi}^{\star})$ for every pair of states.
When $\alpha>0$ the proposal~\eqref{eq:mala-forward} is \emph{asymmetric}: the forward density $q(\boldsymbol{\xi}^{\star}\mid\boldsymbol{\xi}_t)$ is centered on $\boldsymbol{\mu}_t$, but the reverse density $q(\boldsymbol{\xi}_t\mid\boldsymbol{\xi}^{\star})$ is centered on a \emph{different} point $\boldsymbol{\mu}^{\star}=(1{-}\alpha)\boldsymbol{\xi}^{\star}+\alpha\mathbf{X}\mathbf{a}^{\star}$ (line~7). The gradient at the candidate $\boldsymbol{\xi}^{\star}$ generally differs from the gradient at $\boldsymbol{\xi}_t$, so the two Gaussians are not mirror images of each other. This asymmetry means the ULA chain does not satisfy detailed balance and therefore converges to a distribution that is $O(\alpha)$-close to, but not exactly, $p_\beta$.

MALA corrects this by computing the Metropolis--Hastings acceptance ratio on line~8, which is the logarithm of
\begin{equation}\label{eq:mala-ratio}
r = \frac{p_\beta(\boldsymbol{\xi}^{\star})\;q(\boldsymbol{\xi}_t\mid\boldsymbol{\xi}^{\star})}
         {p_\beta(\boldsymbol{\xi}_t)\;q(\boldsymbol{\xi}^{\star}\mid\boldsymbol{\xi}_t)}.
\end{equation}
We now derive $\log r$ explicitly. The target ratio contributes
\begin{align*}
\log\frac{p_\beta(\boldsymbol{\xi}^{\star})}{p_\beta(\boldsymbol{\xi}_t)}
= -\beta\bigl[E(\boldsymbol{\xi}^{\star})-E(\boldsymbol{\xi}_t)\bigr].
\end{align*}
For the proposal ratio, both the forward and reverse proposals are Gaussians with the same variance $\sigma^{2}=2\alpha/\beta$ but different means. Writing out the log-density of a $\mathcal{N}(\boldsymbol{\mu},\sigma^{2}\mathbf{I})$ and noting that the normalization constants cancel:
\begin{align*}
\log\frac{q(\boldsymbol{\xi}_t\mid\boldsymbol{\xi}^{\star})}{q(\boldsymbol{\xi}^{\star}\mid\boldsymbol{\xi}_t)}
&= -\frac{1}{2\sigma^{2}}\lVert\boldsymbol{\xi}_t-\boldsymbol{\mu}^{\star}\rVert^{2}
   +\frac{1}{2\sigma^{2}}\lVert\boldsymbol{\xi}^{\star}-\boldsymbol{\mu}_t\rVert^{2}\\
&= -\frac{1}{2\cdot(2\alpha/\beta)}\bigl[\lVert\boldsymbol{\xi}_t-\boldsymbol{\mu}^{\star}\rVert^{2}
   -\lVert\boldsymbol{\xi}^{\star}-\boldsymbol{\mu}_t\rVert^{2}\bigr]\\
&= -\frac{\beta}{4\alpha}\bigl[\lVert\boldsymbol{\xi}_t-\boldsymbol{\mu}^{\star}\rVert^{2}
   -\lVert\boldsymbol{\xi}^{\star}-\boldsymbol{\mu}_t\rVert^{2}\bigr].
\end{align*}
Summing the two contributions gives the expression on line~8:
\begin{equation}\label{eq:mala-log-ratio}
\log r
= \underbrace{-\beta\bigl[E(\boldsymbol{\xi}^{\star}) - E(\boldsymbol{\xi}_t)\bigr]}_{\text{(i) target ratio}}
  \;-\;\underbrace{\frac{\beta}{4\alpha}\bigl[\lVert\boldsymbol{\xi}_t - \boldsymbol{\mu}^{\star}\rVert^2 - \lVert\boldsymbol{\xi}^{\star} - \boldsymbol{\mu}_t\rVert^2\bigr]}_{\text{(ii) proposal asymmetry correction}}.
\end{equation}
Term~(i) is the standard Boltzmann factor: it favors moves to lower energy (as in simulated annealing). Term~(ii) corrects for the asymmetry of the Langevin proposal: $\lVert\boldsymbol{\xi}^{\star}-\boldsymbol{\mu}_t\rVert^2$ measures how far the candidate is from the forward proposal mean, while $\lVert\boldsymbol{\xi}_t-\boldsymbol{\mu}^{\star}\rVert^2$ measures how far the current state is from the reverse proposal mean. If the proposal were symmetric (as it would be for a random walk with no gradient), these two norms would cancel and only the energy term would remain. With the Langevin gradient drift they differ, and term~(ii) accounts for the resulting bias.

Given $\log r$, the chain draws $u\sim\mathrm{Uniform}(0,1)$ and accepts the candidate $\boldsymbol{\xi}^{\star}$ if $\log u < \min(0,\log r)$, i.e.\ if $u < \min(1,r)$. The two cases collapse to a single rule: when $r\ge 1$ the candidate is favored by both the target and the proposal correction, so the move is accepted with probability~1; when $r<1$ the candidate is disfavored on net, so it is accepted with probability $r$ and rejected with probability $1{-}r$, in which case the chain stays at $\boldsymbol{\xi}_t$. This stochastic accept/reject step restores detailed balance \emph{exactly}, regardless of the step size $\alpha$. The trade-off is efficiency: when $\alpha$ is too large the proposal overshoots, $r$ is frequently small, and most candidates are rejected, so the chain moves slowly. When $\alpha$ is small the proposal is accurate, $r\approx 1$ almost always, and MALA reduces to ULA. The acceptance rate therefore serves as a built-in diagnostic: high acceptance ($\gtrsim$90\%) means the discretization bias is small and ULA suffices; low acceptance means the Metropolis correction is actively preventing the chain from drifting to the wrong distribution.

The correction earns its keep at larger step sizes: at the operating point of our MNIST experiment ($\alpha{=}0.01$, $d{=}784$, $\beta{=}2000$) the acceptance rate is 99.2\% (Table~\ref{tab:mala-generation}), so the correction is a near no-op, but if the step size is increased to accelerate mixing (e.g., $\alpha{=}0.1$ or larger), the ULA bias grows as $O(\alpha)$ and the stationary distribution shifts away from the true target. In this regime MALA remains unbiased (albeit with a lower acceptance rate), making it the preferred choice; practitioners should therefore monitor the acceptance rate, treating values above $\sim$90\% as a signal that ULA suffices while lower rates signal that the Metropolis correction is doing real work.

\section{Load-Ratio Phase Diagram}\label{app:phase-diagram}

The synthetic experiments of the main text characterized the retrieval-to-generation transition along the temperature axis at fixed load. We complemented that view with a two-dimensional phase diagram over the joint $(K/d, \beta)$ plane, swept over $K/d \in [0.05, 2.0]$ and $\beta \in [0.5, 100]$ at $d{=}64$, with each cell averaging over five independent datasets drawn uniformly from $\mathbb{S}^{d-1}$. The full phase diagram of attention concentration $C = 1 - H(\mathbf{a})/\log K$ is in Fig.~\ref{fig:phase-diagram}: the dashed $C=0.5$ contour separated a retrieval regime in the upper-left (high $\beta$, low $K/d$, attention concentrated on a single memory) from a diffuse regime in the lower-right (low $\beta$ or high $K/d$, attention spread across many memories). The boundary shifted rightward as $\beta$ increased, consistent with the classical Hopfield capacity scaling: at higher temperatures the sampler can accommodate more stored patterns before attention loses its retrieval-like sharpness, and the empirical contour qualitatively tracked the $\beta\sigma_{\max}^{2}{=}2$ convexity boundary identified in the main text.

\begin{figure}[t]
\centering
\includegraphics[width=0.65\textwidth]{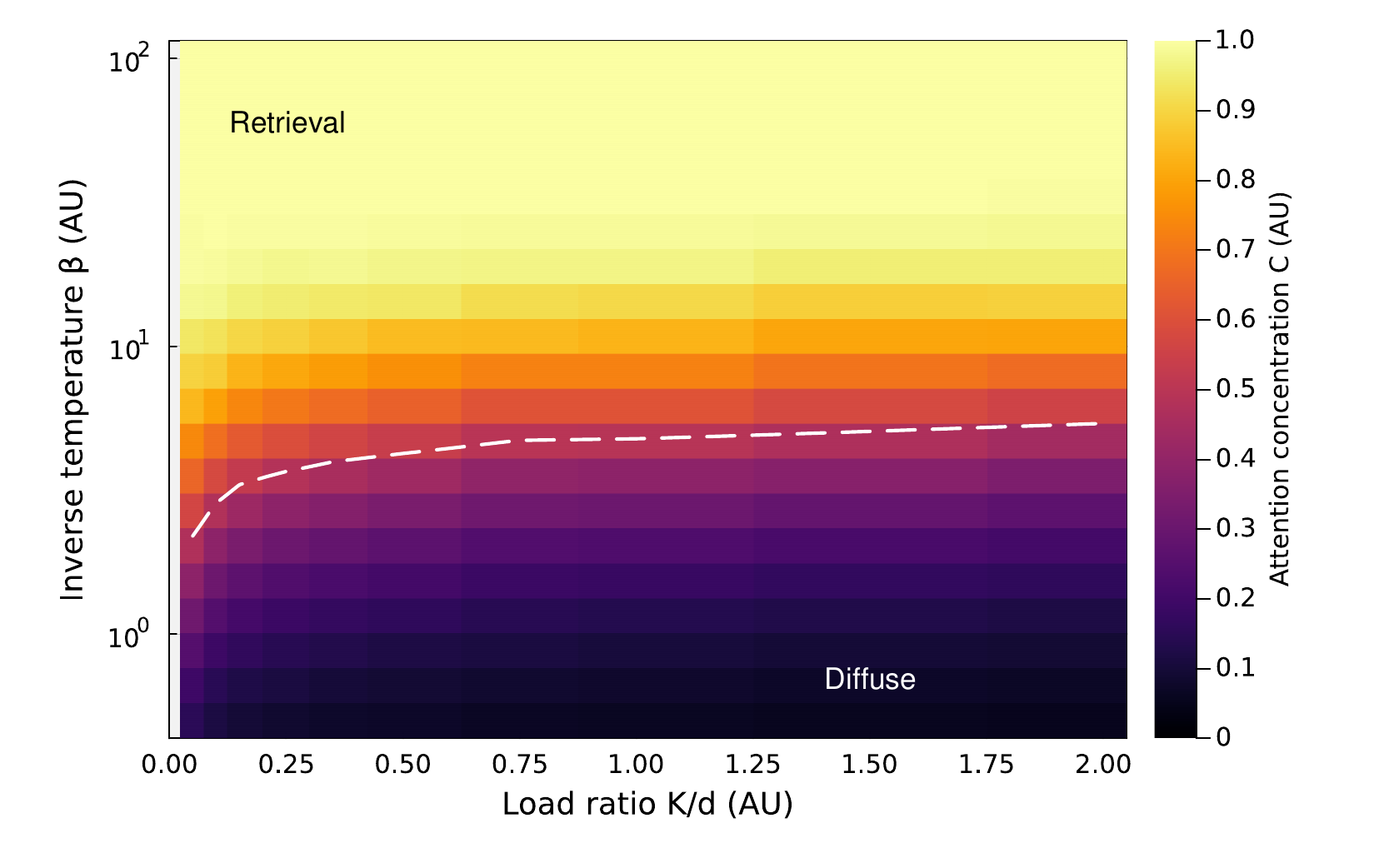}
\caption{Phase diagram of attention concentration $C = 1 - H(\mathbf{a})/\log K$ over load ratio $K/d$ (horizontal) and inverse temperature $\beta$ (vertical, log~scale), with $d=64$. Each cell averages over five independent datasets. The dashed contour marks $C=0.5$, separating a retrieval regime (upper-left, warm colors) from a diffuse regime (lower-right, dark colors).}
\label{fig:phase-diagram}
\end{figure}

\section{Multi-Digit MNIST Generalization}\label{app:multidigit}

To verify that the results reported for digit ``3'' in the main text were not specific to a single morphological class, we repeated the identical multi-chain protocol (30~chains, $T{=}5{,}000$, burn-in~$2{,}000$, thin every 100 then sub-sample 5, $\beta{=}2000$, $\alpha{=}0.01$, $K{=}100$ stored patterns per digit) on two additional MNIST digits chosen for their distinct morphological properties: digit ``1'' (simple stroke, near-degenerate orientation) and digit ``8'' (complex topology with two enclosed loops and high intra-class variance). All seven baselines (including the VAE) and the MALA variant were run with the identical protocol, with results in Tables~\ref{tab:mnist-digit1} and~\ref{tab:mnist-digit8}. The two digits stress different aspects of the Hopfield energy: digit ``1'' produces a near-rank-deficient memory matrix where most variation is captured by orientation, while digit ``8'' produces a high-rank, multimodal memory where the two loops can fail independently. If the SA results were an artifact of digit ``3'' specifically, we would expect divergent behavior here.

\begin{table}[h]
\centering
\caption{Quantitative comparison on MNIST digit ``1'' ($K{=}100$, $d{=}784$, $\beta{=}2000$). Same protocol as Table~\ref{tab:mnist-baselines}, including the VAE baseline (latent dim 8, same two-phase training). MALA acceptance rate: 99.2\%. Values are mean $\pm$ SE across 30 chains.}
\label{tab:mnist-digit1}
\small
\begin{tabular}{@{}lccc@{}}
\toprule
Method & $\mathcal{N}$ $\uparrow$ & $\bar{\mathcal{D}}$ $\uparrow$ & $\bar{E}$ $\downarrow$ \\
\midrule
Bootstrap (replay)          & $0.000 \pm 0.000$ & $0.416 \pm 0.019$ & $-0.500 \pm 0.000$ \\
Gaussian perturbation       & $0.004 \pm 0.000$ & $0.422 \pm 0.017$ & $-0.496 \pm 0.000$ \\
Random convex combination   & $0.097 \pm 0.001$ & $0.007 \pm 0.000$ & $-0.398 \pm 0.001$ \\
GMM-PCA ($r{=}50$, $C{=}10$) & $0.130 \pm 0.005$ & $0.401 \pm 0.017$ & $-0.366 \pm 0.005$ \\
VAE (latent${=}8$)          & $0.180 \pm 0.006$ & $0.408 \pm 0.008$ & $-0.320 \pm 0.006$ \\
MALA                        & $0.151 \pm 0.001$ & $0.586 \pm 0.001$ & $-0.305 \pm 0.001$ \\
\textbf{Stochastic attention} & $\mathbf{0.153 \pm 0.001}$ & $\mathbf{0.591 \pm 0.001}$ & $\mathbf{-0.303 \pm 0.001}$ \\
\bottomrule
\end{tabular}
\end{table}

\begin{table}[h]
\centering
\caption{Quantitative comparison on MNIST digit ``8'' ($K{=}100$, $d{=}784$, $\beta{=}2000$). Same protocol as Table~\ref{tab:mnist-baselines}, including the VAE baseline (latent dim 8, same two-phase training). MALA acceptance rate: 99.2\%. Values are mean $\pm$ SE across 30 chains.}
\label{tab:mnist-digit8}
\small
\begin{tabular}{@{}lccc@{}}
\toprule
Method & $\mathcal{N}$ $\uparrow$ & $\bar{\mathcal{D}}$ $\uparrow$ & $\bar{E}$ $\downarrow$ \\
\midrule
Bootstrap (replay)          & $0.000 \pm 0.000$ & $0.421 \pm 0.012$ & $-0.500 \pm 0.000$ \\
Gaussian perturbation       & $0.004 \pm 0.000$ & $0.450 \pm 0.011$ & $-0.496 \pm 0.000$ \\
Random convex combination   & $0.099 \pm 0.000$ & $0.007 \pm 0.000$ & $-0.395 \pm 0.000$ \\
GMM-PCA ($r{=}50$, $C{=}10$) & $0.194 \pm 0.005$ & $0.420 \pm 0.011$ & $-0.305 \pm 0.005$ \\
VAE (latent${=}8$)          & $0.209 \pm 0.005$ & $0.397 \pm 0.009$ & $-0.291 \pm 0.005$ \\
MALA                        & $0.151 \pm 0.001$ & $0.554 \pm 0.002$ & $-0.305 \pm 0.001$ \\
\textbf{Stochastic attention} & $\mathbf{0.152 \pm 0.001}$ & $\mathbf{0.557 \pm 0.001}$ & $\mathbf{-0.303 \pm 0.001}$ \\
\bottomrule
\end{tabular}
\end{table}

Across all three digit classes, the stochastic attention sampler and MALA produced practically indistinguishable metrics, both dominating every baseline on diversity simultaneously. The VAE (latent dim 8) was the strongest non-Langevin baseline on novelty ($0.180$--$0.209$ vs.\ $0.130$--$0.194$ for GMM-PCA), but both learned models were far below the Langevin methods on diversity ($0.397$--$0.408$ for VAE vs.\ $0.557$--$0.600$ for SA), confirming that a static learned density cannot match the ergodic coverage of Langevin dynamics. The MALA acceptance rate was $99.2\%$ for all digits, confirming that ULA discretization bias remains negligible regardless of digit morphology. Representative $4{\times}4$ sample grids for each method (Figs.~\ref{fig:mnist-digit1-grids},~\ref{fig:mnist-digit8-grids}) showed the same qualitative pattern as digit ``3'': bootstrap outputs were exact stored copies, Gaussian perturbation was visually indistinguishable from bootstrap, random convex combinations produced uniform blurry averages, and both Langevin methods generated diverse, structured digits.

\begin{figure}[h]
\centering
\includegraphics[width=0.19\textwidth]{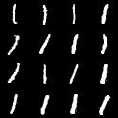}%
\hfill
\includegraphics[width=0.19\textwidth]{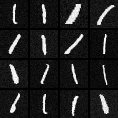}%
\hfill
\includegraphics[width=0.19\textwidth]{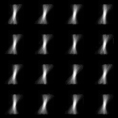}%
\hfill
\includegraphics[width=0.19\textwidth]{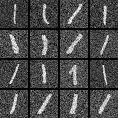}%
\hfill
\includegraphics[width=0.19\textwidth]{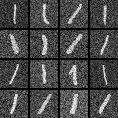}

\small
\makebox[0.19\textwidth]{\centering (a) Bootstrap}%
\hfill
\makebox[0.19\textwidth]{\centering (b) Gaussian}%
\hfill
\makebox[0.19\textwidth]{\centering (c) Convex}%
\hfill
\makebox[0.19\textwidth]{\centering (d) MALA}%
\hfill
\makebox[0.19\textwidth]{\centering (e) Ours}
\caption{Generated MNIST digit ``1'' samples ($4{\times}4$ grids). The pattern matches digit ``3'': only the Langevin-based methods (d,\,e) produce diverse, structured outputs.}
\label{fig:mnist-digit1-grids}
\end{figure}

\begin{figure}[h]
\centering
\includegraphics[width=0.19\textwidth]{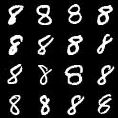}%
\hfill
\includegraphics[width=0.19\textwidth]{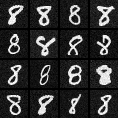}%
\hfill
\includegraphics[width=0.19\textwidth]{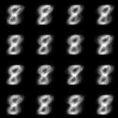}%
\hfill
\includegraphics[width=0.19\textwidth]{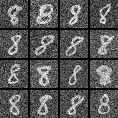}%
\hfill
\includegraphics[width=0.19\textwidth]{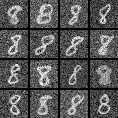}

\small
\makebox[0.19\textwidth]{\centering (a) Bootstrap}%
\hfill
\makebox[0.19\textwidth]{\centering (b) Gaussian}%
\hfill
\makebox[0.19\textwidth]{\centering (c) Convex}%
\hfill
\makebox[0.19\textwidth]{\centering (d) MALA}%
\hfill
\makebox[0.19\textwidth]{\centering (e) Ours}
\caption{Generated MNIST digit ``8'' samples ($4{\times}4$ grids). Despite digit ``8'' having higher intra-class variance and more complex topology than digit ``3'', the qualitative pattern is unchanged: Langevin methods dominate all baselines.}
\label{fig:mnist-digit8-grids}
\end{figure}

The baseline comparison above fixes $\beta{=}2000$, which places the sampler deep in the retrieval regime: each chain stays near its seed pattern and the ``generation'' is local variation within a single energy basin.
To show the full retrieval-to-generation trade-off on real images, we displayed $4{\times}4$ SA grids at four inverse temperatures alongside 16 stored patterns for reference, and computed the corresponding metrics (Fig.~\ref{fig:temp-spectrum-digit8}, Table~\ref{tab:temp-spectrum-digit8}).

\begin{table}[h]
\centering
\caption{Temperature spectrum for digit ``8'' ($K{=}100$, $d{=}784$, $\alpha{=}0.01$, 30 chains). $\mathrm{SNR}=\sqrt{\alpha\beta/2d}$ is the per-step signal-to-noise ratio; the transition occurs near $\mathrm{SNR}^{*}=\sqrt{\alpha/(2\sqrt{d})}$ (Proposition~\ref{prop:entropy-derivative}); at $d{=}784$ and $\alpha{=}0.01$ this gives $\mathrm{SNR}^{*}\approx 0.013$. As SNR falls through this region the sampler crosses from structured retrieval into diffuse generation. Values are mean $\pm$ SE across 30 chains.}
\label{tab:temp-spectrum-digit8}
\small
\begin{tabular}{@{}rccccc@{}}
\toprule
$\beta$ & SNR & $\mathcal{N}$ $\uparrow$ & $\bar{\mathcal{D}}$ $\uparrow$ & $\bar{E}$ $\downarrow$ & Mean max-$\cos$ \\
\midrule
2000 & 0.113 & $0.152 \pm 0.001$ & $0.557 \pm 0.001$ & $-0.3 \pm 0.00$ & $0.848 \pm 0.001$ \\
200  & 0.036 & $0.547 \pm 0.002$ & $0.872 \pm 0.002$ & $1.5 \pm 0.01$  & $0.453 \pm 0.002$ \\
50   & 0.018 & $0.753 \pm 0.002$ & $0.965 \pm 0.002$ & $7.4 \pm 0.03$  & $0.247 \pm 0.002$ \\
10   & 0.008 & $0.870 \pm 0.003$ & $0.992 \pm 0.002$ & $38.6 \pm 0.16$ & $0.130 \pm 0.003$ \\
\bottomrule
\end{tabular}
\end{table}

At $\beta{=}2000$ the outputs were recognizable eights that closely resembled individual stored patterns (mean max-cos 0.85; Table~\ref{tab:temp-spectrum-digit8}); novelty was low and diversity came entirely from the multi-chain initialization, not from any single chain exploring between basins.
As $\beta$ decreased, the energy barriers shrank, chains escaped their initial basins, and novelty/diversity rose sharply.
At $\beta{=}200$ the mean max-cosine dropped to 0.45, meaning outputs sat roughly halfway between stored patterns, a regime of genuine interpolation.
By $\beta{=}50$ the samples were highly novel (0.75) but began to lose recognizable digit structure, and at $\beta{=}10$ the outputs were essentially isotropic noise with no visual resemblance to eights.
This illustrates the fundamental retrieval-generation trade-off: no single $\beta$ simultaneously maximizes fidelity and novelty. The operating point must be chosen according to the application, and the temperature parameter gives the user explicit control over this trade-off.

\begin{figure}[h]
\centering
\includegraphics[width=0.19\textwidth]{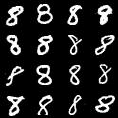}%
\hfill
\includegraphics[width=0.19\textwidth]{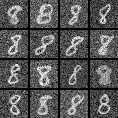}%
\hfill
\includegraphics[width=0.19\textwidth]{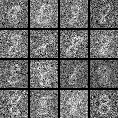}%
\hfill
\includegraphics[width=0.19\textwidth]{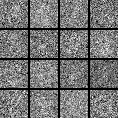}%
\hfill
\includegraphics[width=0.19\textwidth]{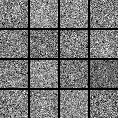}

\small
\makebox[0.19\textwidth]{\centering Stored}%
\hfill
\makebox[0.19\textwidth]{\centering $\beta{=}2000$}%
\hfill
\makebox[0.19\textwidth]{\centering $\beta{=}200$}%
\hfill
\makebox[0.19\textwidth]{\centering $\beta{=}50$}%
\hfill
\makebox[0.19\textwidth]{\centering $\beta{=}10$}
\caption{Temperature spectrum for MNIST digit ``8.'' From left to right: 16 stored patterns; SA samples at $\beta{=}2000$ (retrieval, local variation around stored patterns); $\beta{=}200$ (intermediate, genuine interpolation); $\beta{=}50$ (high novelty, structure fading); $\beta{=}10$ (diffuse noise). The trade-off between fidelity and novelty is governed entirely by~$\beta$.}
\label{fig:temp-spectrum-digit8}
\end{figure}

\section{Single-Chain Diversity Analysis}\label{app:single-chain}

The main-text MNIST baseline (Table~\ref{tab:mnist-baselines}) used 30 independent chains each initialized near a different stored pattern; the reported diversity of $0.600$ included both \emph{initialization diversity} (spread from 30 distinct seeds) and \emph{within-chain mixing} (sampling from a fixed starting point). To quantify these contributions separately (Table~\ref{tab:single-chain}, Figure~\ref{fig:generation-regimes}), we ran one chain from a fixed initialization (stored pattern~1 plus $\sigma_{\mathrm{init}}{=}0.01$ Gaussian noise, seed fixed) for $T{=}50{,}000$ steps (burn-in $10{,}000$, thin every~100, yielding 400 samples) at $\beta\in\{2000,\,200,\,50\}$ ($\mathrm{SNR}\in\{0.113,\,0.036,\,0.018\}$) on digit~``3''.

\begin{table}[h]
\centering
\caption{Single-chain vs.\ multi-chain diversity on digit ``3'' ($K{=}100$, $d{=}784$, $\alpha{=}0.01$). $\mathrm{SNR}=\sqrt{\alpha\beta/2d}$. The multi-chain row reproduces the SA entry from Table~\ref{tab:mnist-baselines}. The diversity gap of $0.318$ at $\mathrm{SNR}{=}0.113$ is the initialization contribution; single-chain diversity exceeds the multi-chain value once SNR falls toward the entropy inflection point (Proposition~\ref{prop:entropy-derivative}).}
\label{tab:single-chain}
\small
\setlength{\tabcolsep}{4pt}
\resizebox{\textwidth}{!}{%
\begin{tabular}{@{}llccccc@{}}
\toprule
Protocol & $\beta$ & SNR & $\mathcal{N}$ $\uparrow$ & $\bar{\mathcal{D}}$ $\uparrow$ & $\bar{E}$ $\downarrow$ & Max-$\cos$ \\
\midrule
Multi-chain (30 chains) & 2000 & 0.113 & $0.152 \pm 0.001$ & $0.600 \pm 0.001$ & $-0.303 \pm 0.001$ & $0.848 \pm 0.001$ \\
Single-chain            & 2000 & 0.113 & $0.153 \pm 0.000$ & $0.282 \pm 0.001$ & $-0.300 \pm 0.000$ & $0.847 \pm 0.000$ \\
Single-chain            & 200  & 0.036 & $0.551 \pm 0.001$ & $0.796 \pm 0.002$ & $\phantom{-}1.5\phantom{0} \pm 0.0\phantom{0}$ & $0.449 \pm 0.001$ \\
Single-chain            & 50   & 0.018 & $0.756 \pm 0.002$ & $0.967 \pm 0.002$ & $\phantom{-}7.4\phantom{0} \pm 0.0\phantom{0}$ & $0.244 \pm 0.002$ \\
\bottomrule
\end{tabular}}
\end{table}

\begin{figure}[h]
\centering
\begin{minipage}[t]{0.31\textwidth}
\centering
\includegraphics[width=\linewidth]{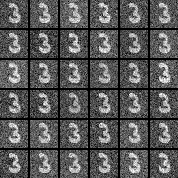}\\[-2pt]
\small (a) $\beta{=}2000$, SNR$=0.113$\\structured retrieval ($\bar{\mathcal{D}}{=}0.282$)
\end{minipage}
\hfill
\begin{minipage}[t]{0.31\textwidth}
\centering
\includegraphics[width=\linewidth]{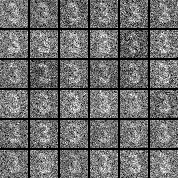}\\[-2pt]
\small (b) $\beta{=}200$, SNR$=0.036$\\genuine generation ($\bar{\mathcal{D}}{=}0.796$)
\end{minipage}
\hfill
\begin{minipage}[t]{0.31\textwidth}
\centering
\includegraphics[width=\linewidth]{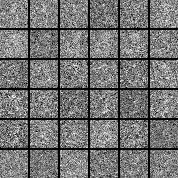}\\[-2pt]
\small (c) $\beta{=}50$, SNR$=0.018$\\diffuse ($\bar{\mathcal{D}}{=}0.967$)
\end{minipage}
\caption{Single-chain samples ($6{\times}6$ grids, digit ``3'') from a fixed seed at three operating points. \textbf{(a)}~At SNR$=0.113$ the chain stays near one stored pattern; diversity is low ($0.282$). \textbf{(b)}~At SNR$=0.036$ the chain spontaneously crosses energy barriers; single-chain diversity ($0.796$) \emph{exceeds} the 30-chain $\beta{=}2000$ value ($0.600$), confirming genuine generation. \textbf{(c)}~At SNR$=0.018$ the chain is fully diffuse; samples lose recognizable digit structure.}
\label{fig:generation-regimes}
\end{figure}

At $\beta{=}2000$ the single-chain diversity of $0.282$ is not negligible: the chain samples a genuine distribution within its energy basin rather than collapsing to a point mass. As $\beta$ decreased the energy barriers shrank, chains escaped their initial basins, and single-chain diversity rose sharply, mirroring the temperature-spectrum results of Appendix~\ref{app:multidigit}. At $\beta{=}200$ single-chain diversity ($0.796$) exceeds the multi-chain value at $\beta{=}2000$, and at $\beta{=}50$ it approaches the isotropic limit ($0.967$, mean max-cos $0.244$). The multi-chain protocol at $\beta{=}2000$ is therefore best understood as a \emph{structured retrieval} strategy in which each chain contributes low-energy samples from a distinct basin, while single-chain operation at $\beta{=}200$ provides genuine generative diversity without requiring multiple initializations.

At $\beta{=}2000$ the single chain remained in the basin of its seed pattern for the entire $50{,}000$-step run (${\approx}12$\,s wall-clock on a laptop CPU): the max-cosine to the seed pattern stayed above $0.84$ at every thinned sample, and the chain never visited the basin of any other stored pattern, so the \emph{inter-basin} mixing time at this temperature exceeded $50{,}000$ steps, consistent with the exponential-barrier scaling discussed in the main text. \emph{Intra-basin} mixing, by contrast, was fast: energy and max-cosine stabilized within the $2{,}000$-step burn-in, and consecutive thinned samples (every $100$ steps) had pairwise cosine distance $0.11\pm0.01$, indicating decorrelation within $O(100)$ steps. At $\beta{=}200$ the picture changed: the chain crossed into a different basin within the first ${\approx}5{,}000$ steps and visited multiple basins over the $50{,}000$-step run, yielding the observed single-chain diversity of $0.796$. The basin-crossing behavior is also visible directly in the chain dynamics (Fig.~\ref{fig:trajectories}): one MNIST chain at $\beta{=}50$ on digit~``3'' visited five distinct stored ``3'' basins (stored memories $\#95,\#95,\#100,\#18,\#53,\#96$ at the six snapshots), with the denoise readout snapping to whichever basin the chain currently occupied, and one unconditional Olivetti chain at $\beta_{\mathrm{chain}}{=}200$ initialized in subject~0 visited five different subjects (subject~6 by $T{=}100$, then~7,~4,~6,~8 over the next $3{,}000$ steps); this drift behavior is what produces the $10.4\%$ unconditional subject-recovery rate and motivates the masking mechanism of the main text.

\begin{figure}[h]
\centering
\includegraphics[width=0.55\textwidth]{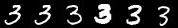}\\[0.4em]
\includegraphics[width=0.95\textwidth]{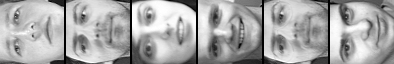}

\small
\makebox[0.95\textwidth]{\centering $T = 0,\, 1000,\, 5000,\, 10000,\, 25000,\, 50000$\ \ (MNIST, top); $T = 0,\, 100,\, 500,\, 1000,\, 2000,\, 3000$\ \ (Olivetti, bottom)}
\caption{Chain-trajectory snapshots from a single Langevin chain at successive timesteps; no burn-in. \textbf{Top:} one MNIST chain at $\beta{=}50$ on digit~``3'', warm-started from stored ``3'' \#95; each cell is the denoise readout (one noise-free attention update at $\beta_{\mathrm{read}}{=}2{,}000$) of the chain state at the indicated $T$. The chain visits stored basins $\{95,\,95,\,100,\,18,\,53,\,96\}$ across the six checkpoints, demonstrating inter-basin mixing on the $\sim 5{,}000$-step timescale established above. \textbf{Bottom:} one unconditional Olivetti chain at $\beta_{\mathrm{chain}}{=}200$, warm-started in subject~0; each cell is the denoise readout at $\beta_{\mathrm{read}}{=}10{,}000$. The chain leaves subject~0 by $T{=}100$ and visits subjects $\{6,\,7,\,4,\,6,\,8\}$ at the subsequent checkpoints, illustrating the inter-basin drift the hard mask of the main text suppresses.}
\label{fig:trajectories}
\end{figure}

% ─────────────────────────────────────────────────────────────────────────────
\section{Protein Sequence Generation}\label{app:protein}
% ─────────────────────────────────────────────────────────────────────────────

To test stochastic attention on a non-image domain with discrete structure, we applied the full experimental pipeline to protein sequences from the Pfam RRM family (PF00076, RNA Recognition Motif)~\cite{mistryPfamProteinFamilies2021}. We downloaded the Pfam seed alignment (Stockholm format), removed columns with $>50\%$ gaps and sequences with $>30\%$ gaps, yielding $K{=}68$ aligned sequences of length $L{=}71$ positions over the 20 standard amino acids. Each sequence was one-hot encoded into $\mathbb{R}^{1420}$ ($20\times 71$), projected via PCA retaining 95\% of variance ($1420\to 59$ dimensions), and $\ell_2$-normalized to form the memory matrix $\hat{\mathbf{X}}\in\mathbb{R}^{59\times 68}$. The entropy inflection analysis (Proposition~\ref{prop:entropy-derivative}) yielded $\beta^*\approx 3.85$ ($\mathrm{SNR}^*=0.018$), compared to the theoretical prediction $\beta^*=\sqrt{d}=7.68$ for random unit-norm patterns (Fig.~\ref{fig:protein-entropy}); the ratio of $0.50$ reflects the structured (non-random) geometry of the protein family, where conserved positions reduce the effective variance of the similarities $e_k$ and shift the transition to lower $\beta$. We set $\beta_{\mathrm{ret}}=20\beta^*\approx 77$ ($\mathrm{SNR}=0.081$) for the structured-retrieval regime and $\beta_{\mathrm{gen}}=2\beta^*\approx 8$ ($\mathrm{SNR}=0.026$, just above the inflection) for the generation regime.

We ran the identical multi-chain protocol used for MNIST: 30 chains, $T{=}5{,}000$ iterations, burn-in $2{,}000$, thinning every 100 steps, 5 samples per chain, yielding $S{=}150$ generated sequences. SA was run at both $\beta_{\mathrm{ret}}{=}77$ (retrieval) and $\beta_{\mathrm{gen}}{=}8$ (generation) to demonstrate the temperature knob on a non-image domain, and all seven baselines (bootstrap, Gaussian perturbation, random convex combination, GMM-PCA, VAE with latent dim 8, MALA) were run with the same protocol. Generated PCA-space vectors were decoded back to amino acid sequences by inverse-PCA followed by per-position argmax over the 20 amino acid channels. In addition to the standard metrics (novelty~$\mathcal{N}$, diversity~$\bar{\mathcal{D}}$, energy~$\bar{E}$), we reported two protein-specific measures: \emph{sequence identity} to the nearest stored member (max over stored sequences of the fraction of non-gap positions with matching amino acids) and \emph{amino acid composition KL divergence} ($D_{\mathrm{KL}}(p_{\mathrm{stored}}\|q_{\mathrm{gen}})$ over the 20 amino acid frequencies).

\begin{table}[h]
\centering
\caption{Protein sequence generation on Pfam PF00076 (RRM family, $K{=}68$, $L{=}71$, $d{=}59$ after PCA). Two SA rows demonstrate the temperature knob: $\beta{=}77$ (retrieval, $20\beta^*$) and $\beta{=}8$ (generation, $2\beta^*$). SeqID: maximum sequence identity to any stored sequence. KL: amino acid composition divergence from stored sequences (lower is better). $^\dagger$Energy is positive at $\beta{=}8$ because samples explore off the attractor manifold. MALA acceptance rate: 99.8\%.}
\label{tab:protein}
\small
\setlength{\tabcolsep}{4pt}
\resizebox{\textwidth}{!}{%
\begin{tabular}{@{}lccccc@{}}
\toprule
Method & $\mathcal{N}$ $\uparrow$ & $\bar{\mathcal{D}}$ $\uparrow$ & $\bar{E}$ $\downarrow$ & SeqID & KL $\downarrow$ \\
\midrule
Bootstrap (replay)          & $0.000 \pm 0.000$ & $1.000 \pm 0.012$ & $-0.500 \pm 0.000$ & $0.645 \pm 0.004$ & 0.133 \\
Gaussian perturbation       & $0.008 \pm 0.000$ & $1.002 \pm 0.007$ & $-0.492 \pm 0.000$ & $0.633 \pm 0.006$ & 0.132 \\
Random convex combination   & $0.510 \pm 0.009$ & $0.994 \pm 0.009$ & $-0.068 \pm 0.001$ & $0.562 \pm 0.000$ & 2.091 \\
GMM-PCA ($r{=}50$, $C{=}10$) & $0.606 \pm 0.009$ & $1.000 \pm 0.007$ & $+0.077 \pm 0.009$ & $0.540 \pm 0.003$ & 0.731 \\
VAE (latent${=}8$)          & $0.621 \pm 0.007$ & $0.834 \pm 0.022$ & $+0.118 \pm 0.007$ & $0.532 \pm 0.003$ & 0.416 \\
MALA ($\beta{=}77$)         & $0.249 \pm 0.004$ & $0.992 \pm 0.007$ & $-0.115 \pm 0.005$ & $0.609 \pm 0.012$ & 0.112 \\
SA ($\beta{=}77$, retrieval) & $0.243 \pm 0.004$ & $0.991 \pm 0.007$ & $-0.120 \pm 0.006$ & $0.616 \pm 0.012$ & 0.107 \\
\textbf{SA ($\beta{=}8$, generation)} & $\mathbf{0.623 \pm 0.006}$ & $\mathbf{1.001 \pm 0.006}$ & $3.08 \pm 0.06^\dagger$ & $\mathbf{0.538 \pm 0.006}$ & \textbf{0.060} \\
\bottomrule
\end{tabular}}
\end{table}

\begin{figure}[h]
\centering
\begin{minipage}[t]{0.48\textwidth}
\centering
\IfFileExists{figs/Fig_protein_aa_frequencies.pdf}{%
  \includegraphics[width=\linewidth]{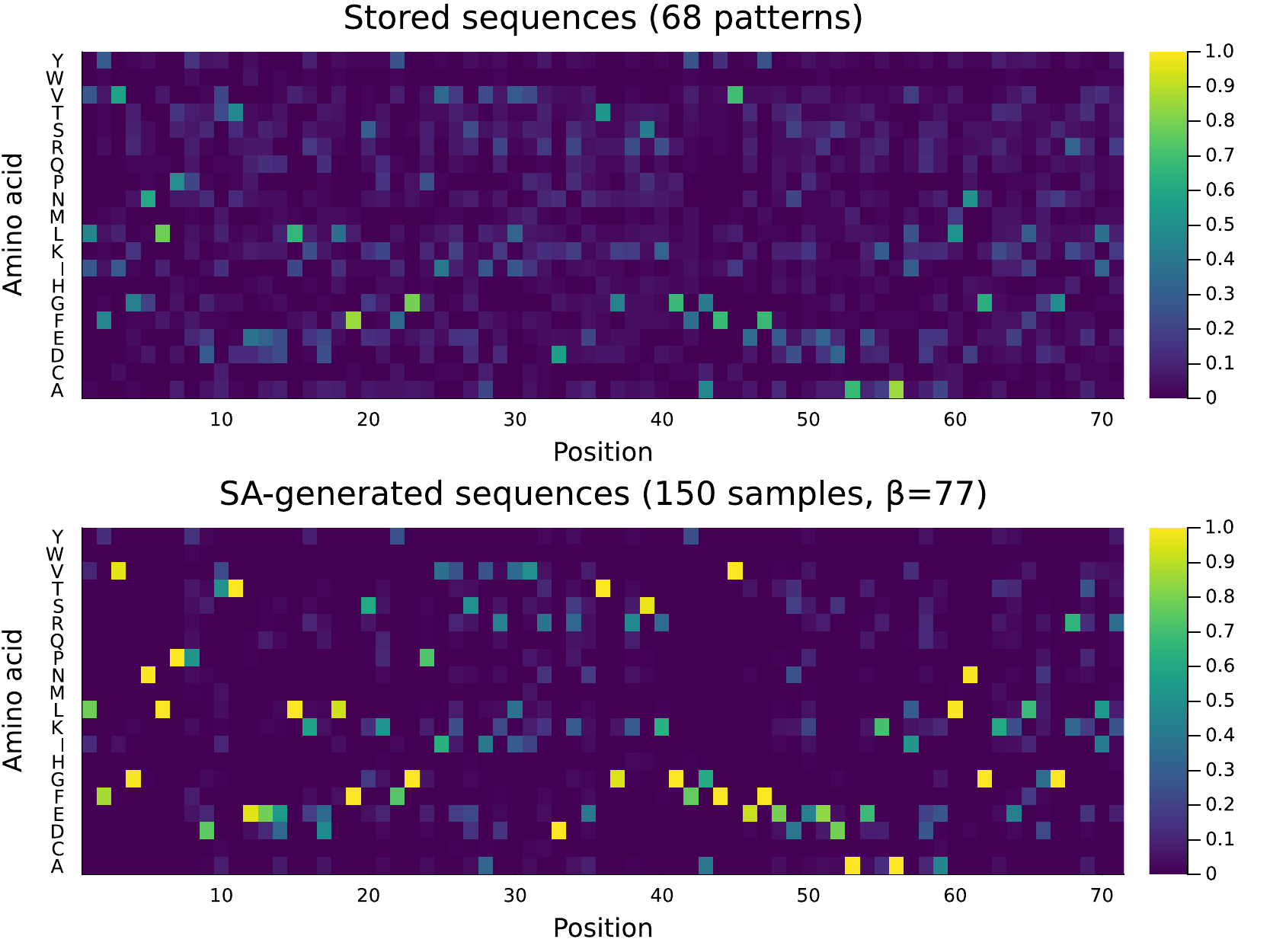}%
}{%
  \framebox[\linewidth]{\rule{0pt}{4cm}\small [AA frequency heatmap]}%
}
\end{minipage}
\hfill
\begin{minipage}[t]{0.48\textwidth}
\centering
\IfFileExists{figs/Fig_protein_sequence_identity.pdf}{%
  \includegraphics[width=\linewidth]{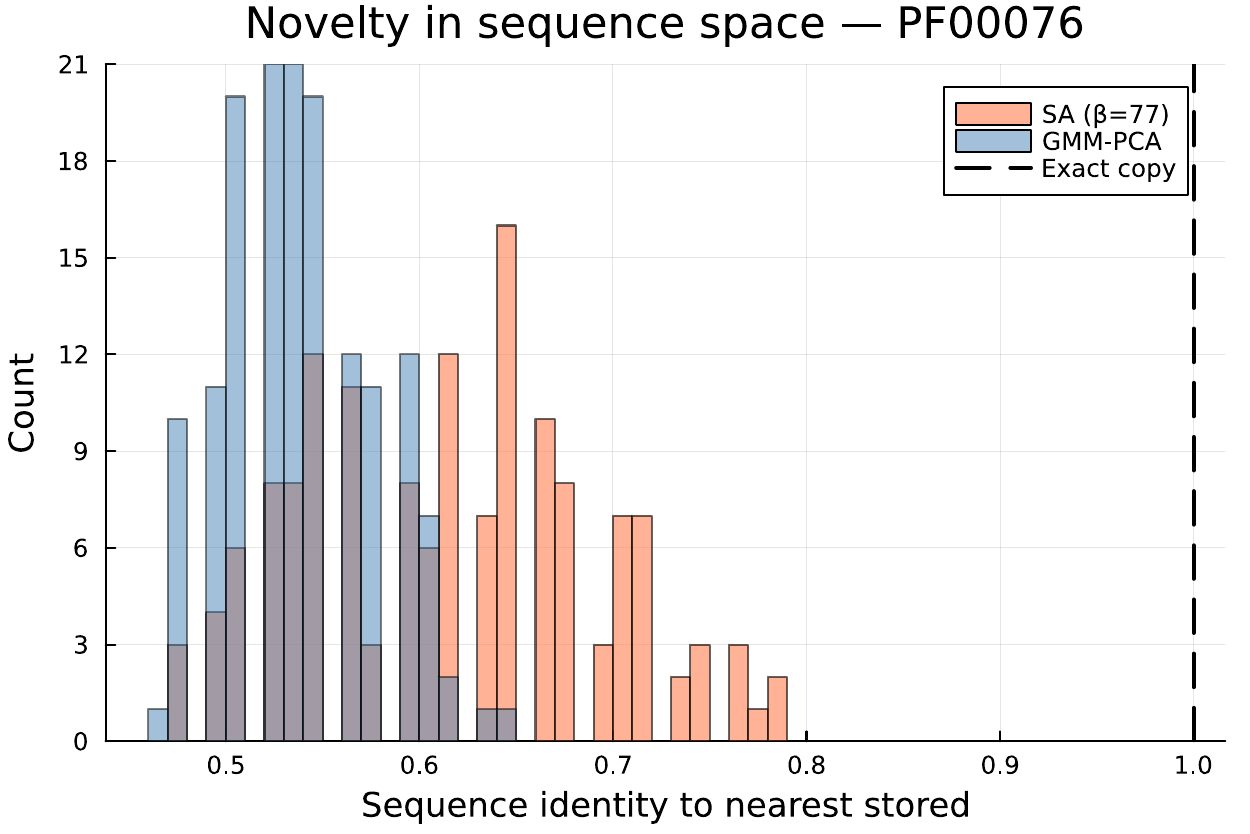}%
}{%
  \framebox[\linewidth]{\rule{0pt}{4cm}\small [Sequence identity histogram]}%
}
\end{minipage}
\caption{Protein sequence generation on PF00076. \textbf{Left:} Per-position amino acid frequencies for stored sequences (top) and SA-generated sequences (bottom). SA preserves the conserved positions of the RRM family while introducing variation at non-conserved sites. \textbf{Right:} Distribution of sequence identity to the nearest stored member for SA and GMM-PCA. SA generates sequences that are ${\approx}62\%$ identical to their nearest stored neighbor, while GMM-PCA produces more distant sequences (${\approx}54\%$) with much higher compositional divergence (KL$=0.731$ vs.\ $0.107$).}
\label{fig:protein}
\end{figure}

The same two-regime structure observed on MNIST appeared on the protein experiment (Table~\ref{tab:protein}, Fig.~\ref{fig:protein}). At $\beta{=}77$ (retrieval), SA generated sequences ${\approx}62\%$ identical to their nearest stored member with the lowest KL divergence of any method at this temperature ($0.107$); SA and MALA were again indistinguishable (MALA acceptance $99.8\%$), consistent with the MNIST results. Random convex combinations catastrophically distorted the composition (KL$=2.091$) because averaging one-hot vectors in PCA space does not preserve per-position amino acid structure. At $\beta{=}8$ (generation, $2\beta^*$), SA matched the VAE on novelty ($0.623$ vs.\ $0.621$) and exceeded it on diversity ($1.001$ vs.\ $0.834$), while achieving dramatically lower compositional divergence (KL$=0.060$ vs.\ $0.416$ for the VAE, a $6.9{\times}$ improvement); SA-generated sequences at generation temperature had the closest amino acid composition to the real RRM family of any method tested, including all baselines and SA at retrieval temperature. The temperature knob worked as designed: lowering $\beta$ from retrieval to generation traded sequence identity ($0.616\to 0.538$) for novelty ($0.243\to 0.623$) while preserving family-level compositional fidelity, with energy rising above zero (samples leave the attractor manifold) but the Langevin gradient keeping the chain closer to the memory geometry than any learned baseline. The entropy inflection analysis validated at this new dimensionality: the empirical $\beta^*{=}3.85$ was approximately half the random-pattern prediction $\sqrt{d}{=}7.68$ (Fig.~\ref{fig:protein-entropy}), an expected discrepancy because RRM sequences share conserved residues at many positions, reducing the effective variance of the query-key similarities and shifting the phase transition to lower $\beta$; the inflection criterion (Eq.~\ref{eq:inflection}) automatically accounts for this structure, providing the correct operating point without manual tuning.

\begin{figure}[h]
\centering
\includegraphics[width=0.55\textwidth]{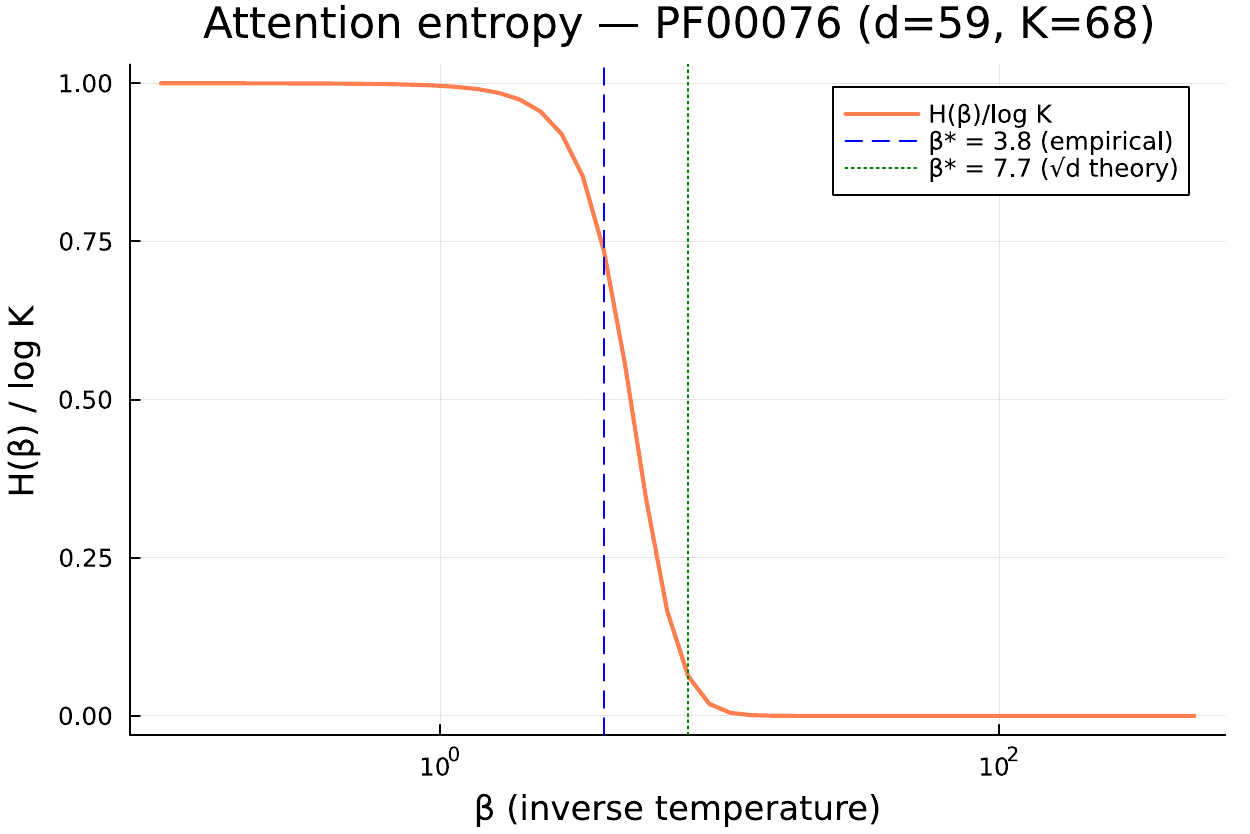}
\caption{Attention entropy $H(\beta)/\log K$ for the Pfam RRM memory ($K{=}68$, $d{=}59$). The empirical inflection point $\beta^*{=}3.8$ (dashed blue) occurs at roughly half the random-pattern prediction $\sqrt{d}{=}7.7$ (dotted green), reflecting the reduced effective similarity variance caused by conserved residues.}
\label{fig:protein-entropy}
\end{figure}

To provide a biologically meaningful validity metric beyond amino acid composition, we validated all generated sequences against the Pfam RRM profile HMM using HMMER~3.4~\cite{hmmer3}. Gap characters were removed from each generated sequence before submission to \texttt{hmmsearch}. All $150$ generated sequences from every method passed at $E$-value~$<0.01$, with SA at generation temperature achieving a median $E$-value of $8.1{\times}10^{-34}$, far below the threshold and lower (better) than the stored sequences' median of $2.5{\times}10^{-22}$. The generated sequences achieved better $E$-values because the generation process produced smoother, more canonical sequences that matched the HMM consensus closely; the $100\%$ pass rate confirmed that SA-generated sequences were biologically valid RRM family members, not arbitrary sequences that happened to share amino acid composition statistics (Table~\ref{tab:protein-main}).

The global amino acid KL divergence aggregates frequencies across all positions; to verify that SA preserved \emph{position-specific} conservation patterns, we computed the KL divergence independently at each of the $L{=}71$ aligned positions and reported the mean. SA in the generation regime achieved a mean per-position KL of $2.92$, compared to $9.99$ for the VAE ($3.4{\times}$ lower), $7.52$ for bootstrap, and $9.09$ for GMM-PCA (Table~\ref{tab:protein-main}). This confirmed that the advantage extended beyond global composition to position-specific conservation patterns: SA correctly identified which positions were conserved and which were variable, preserving the position-specific amino acid distribution at each residue, and MALA at generation temperature achieved a comparable $2.96$.

To address the concern that marginal frequency metrics ignore sequence-level dependencies, we computed the pairwise mutual information (MI) between residue positions. For a set of sequences, the MI between positions $i$ and $j$ is $\mathrm{MI}(i,j)=\sum_{a,b}p(a,b)\log\bigl[p(a,b)/(p(a)p(b))\bigr]$, where $p(a,b)$ is the joint frequency of amino acids $a$ and $b$ at positions $i$ and $j$. We computed MI matrices for each method's generated sequences and measured the Pearson correlation with the stored family's MI matrix. SA at generation temperature achieved MI correlation $r{=}0.871$, compared to $r{=}0.525$ for the VAE, a $66\%$ improvement; bootstrap replay achieved $r{=}0.692$ and SA at retrieval temperature achieved $r{=}0.733$ (Table~\ref{tab:protein-main}). SA-generated sequences therefore reproduced the pairwise co-evolutionary coupling structure of the RRM family, capturing sequence-level dependencies that go beyond single-position statistics; combined with the low per-position KL and $100\%$ HMM pass rate, this provided strong evidence that SA-generated sequences were structurally and functionally plausible RRM family members. SA and MALA produced near-identical results at the generation temperature ($\beta{=}8$; KL$=0.060$ vs.\ $0.058$, sequence identity $0.538$ vs.\ $0.537$, novelty $0.623$ vs.\ $0.616$, MI correlation $0.871$ vs.\ $0.887$, MALA acceptance $99.8\%$), so the ULA bias was negligible for protein sequences at both retrieval and generation temperatures.

\section{Financial Log-Return Generation}\label{app:finance}

The synthetic experiments validated the sampler on controlled geometry; this section used real financial data to \emph{precisely characterize} what equilibrium Boltzmann sampling on the Hopfield energy captures and, equally importantly, what it cannot, and why each limitation is a principled theoretical consequence rather than a tuning failure.

\subsection{Data and Memory Construction}

We collected daily adjusted closing prices for $d=424$ S\&P~500 constituents with complete histories over a $K=2{,}766$-trading-day window. For each firm $i$ and day $t$, the continuously compounded growth rate is $g^{(i)}_t = \log(S^{(i)}_t / S^{(i)}_{t-1})$. Each trading day yielded a $d$-dimensional return vector $\mathbf{g}_t \in \mathbb{R}^{424}$. We stored these as columns of a raw memory matrix $\mathbf{M}\in\mathbb{R}^{d\times K}$ and constructed the scaled memory matrix $\mathbf{X}\in\mathbb{R}^{d\times K}$ by centering each column to zero cross-sectional mean and normalizing to unit $\ell_2$~norm, following the standard protocol for modern Hopfield networks~\cite{ramsauerHopfieldNetworksAll2021}. The load ratio was $K/d \approx 6.5$.

\subsection{I.I.D.\ Pool: Marginal and Cross-Sectional Fidelity}

We compared SA to \emph{historical bootstrap resampling} (drawing stored return vectors uniformly at random from $\mathbf{M}$) on marginal and cross-sectional metrics. Because the projection $\hat{\mathbf{g}}=\mathbf{M}\,\operatorname{softmax}(\beta\,\mathbf{X}^\top\boldsymbol{\xi})$ is a convex combination of historical return vectors, bootstrap sets an upper bound on distributional fidelity; the key question is what SA contributes beyond verbatim replay. We launched $n_{\text{chains}}=30$ independent ULA chains, each initialized near a randomly chosen stored memory, and ran $T=5{,}000$ steps with step size $\alpha=0.01$ and fixed inverse temperature $\beta=25$ ($\mathrm{SNR}=\sqrt{\alpha\beta/2d}\approx 0.017$). After discarding a burn-in of $2{,}000$ steps and thinning every $100$, we obtained $900$ approximately i.i.d.\ samples.

For each of $d=424$ tickers, we ran a two-sample KS test between generated and historical returns, and QQ plots for five representative tickers (Fig.~\ref{fig:finance-qq}). The convex-combination projection compresses per-ticker standard deviation by $\approx 1.14\times$ at $\beta{=}25$ (SNR$\approx 0.017$); we corrected for this by applying a per-ticker affine standardization (aligning the mean and standard deviation of each ticker's generated series to the historical moments, analogous to the marginal-calibration step in copula-based scenario generation), which left the chain dynamics and novelty metric unchanged. After correction, $272$ of $424$ tickers passed the two-sample KS test at the $5\%$ level (mean $p{=}0.193$); bootstrap achieved $99.3\%$ pass by drawing from the empirical distribution (mean $p{=}0.622$), and the residual $35.8\%$ SA failure reflected higher-order distributional differences (skewness, kurtosis) beyond what affine correction addresses. The key distinction remained novelty: SA chain states had $1-\max_k\cos(\boldsymbol{\xi},\mathbf{X}_{:,k})=0.768\pm0.001$, while every bootstrap draw was a verbatim historical scenario (novelty~$0.000$, Table~\ref{tab:finance-summary}).

The pairwise correlations of generated returns lay along the $y{=}x$ diagonal when plotted against the corresponding historical correlations (Fig.~\ref{fig:finance-correlation}, all $\binom{424}{2}=89{,}676$ asset pairs): SA Frobenius correlation error was $26.3\%$ versus $15.6\%$ for bootstrap. The per-ticker affine correction did not alter this value because correlation is scale-invariant, so rescaling each ticker's series preserves $\operatorname{cor}(\hat{\mathbf{G}})$ exactly; SA traded some correlation fidelity for novelty, with the Hopfield energy concentrating sampling on energetically favorable regime interpolations rather than uniform historical replay. Both right and left tails of the empirical survival function for portfolio-level (equal-weight) returns matched closely on a log scale (Fig.~\ref{fig:finance-tails}), demonstrating that the sampler reproduced the heavy-tailed nature of equity returns without distributional assumptions.

\begin{figure}[h]
\centering
\includegraphics[width=0.95\textwidth]{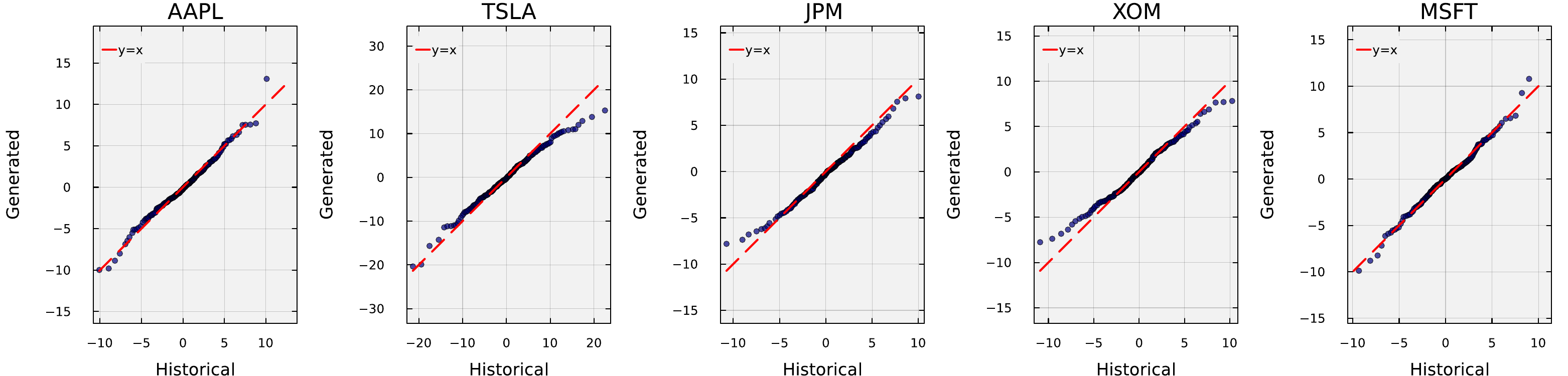}
\caption{QQ plots of generated (after per-ticker affine correction) vs.\ historical log returns for five representative S\&P~500 constituents. Quantiles track closely for these tickers; $64.2\%$ of all $424$ tickers pass the two-sample KS test at the $5\%$ level. The dashed red line is $y=x$.}
\label{fig:finance-qq}
\end{figure}

\begin{figure}[h]
\centering
\includegraphics[width=0.45\textwidth]{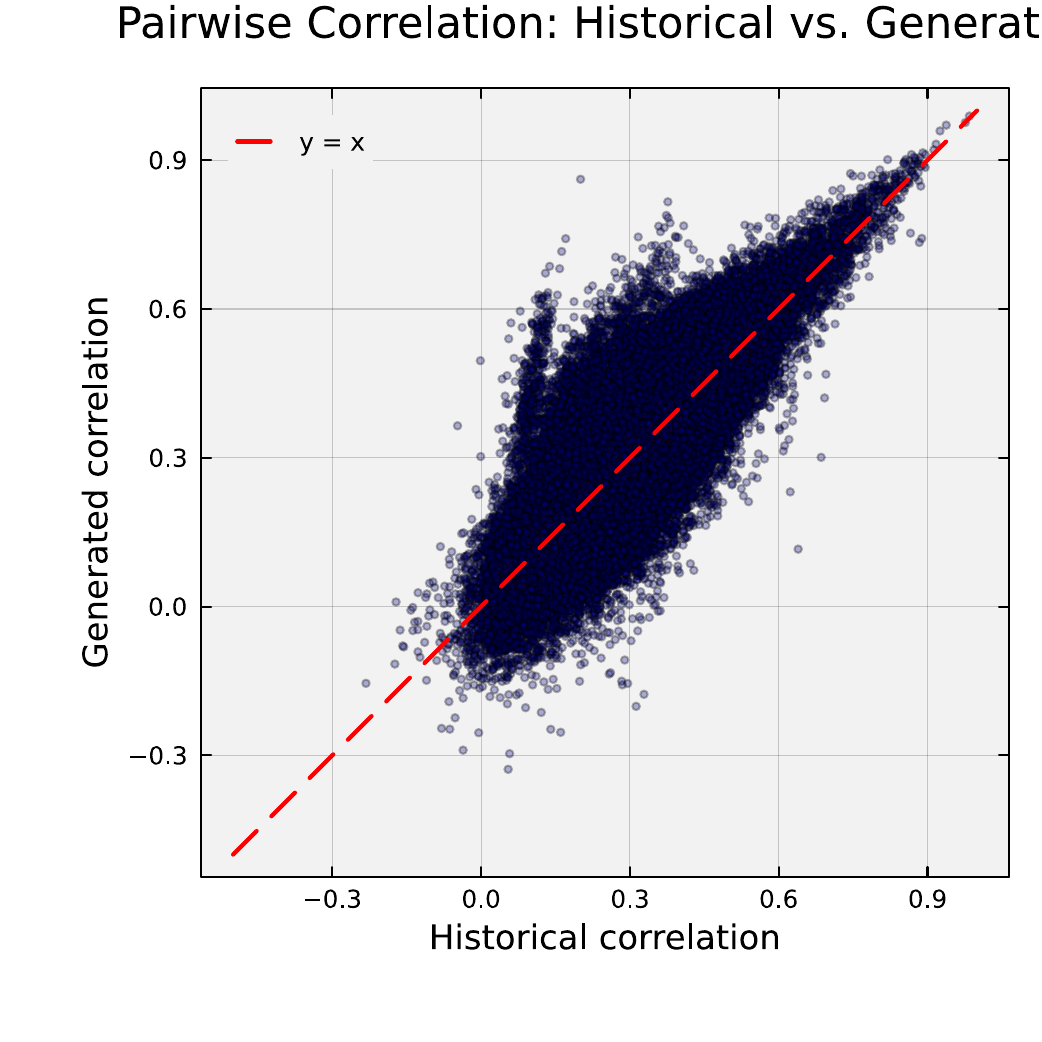}
\caption{Pairwise correlation: historical vs.\ generated. Each point represents one of $89{,}676$ asset pairs. The scatter around $y=x$ shows partial preservation of cross-asset dependence structure (Frobenius error $26.3\%$ vs.\ $15.6\%$ for bootstrap).}
\label{fig:finance-correlation}
\end{figure}

\begin{figure}[h]
\centering
\includegraphics[width=0.75\textwidth]{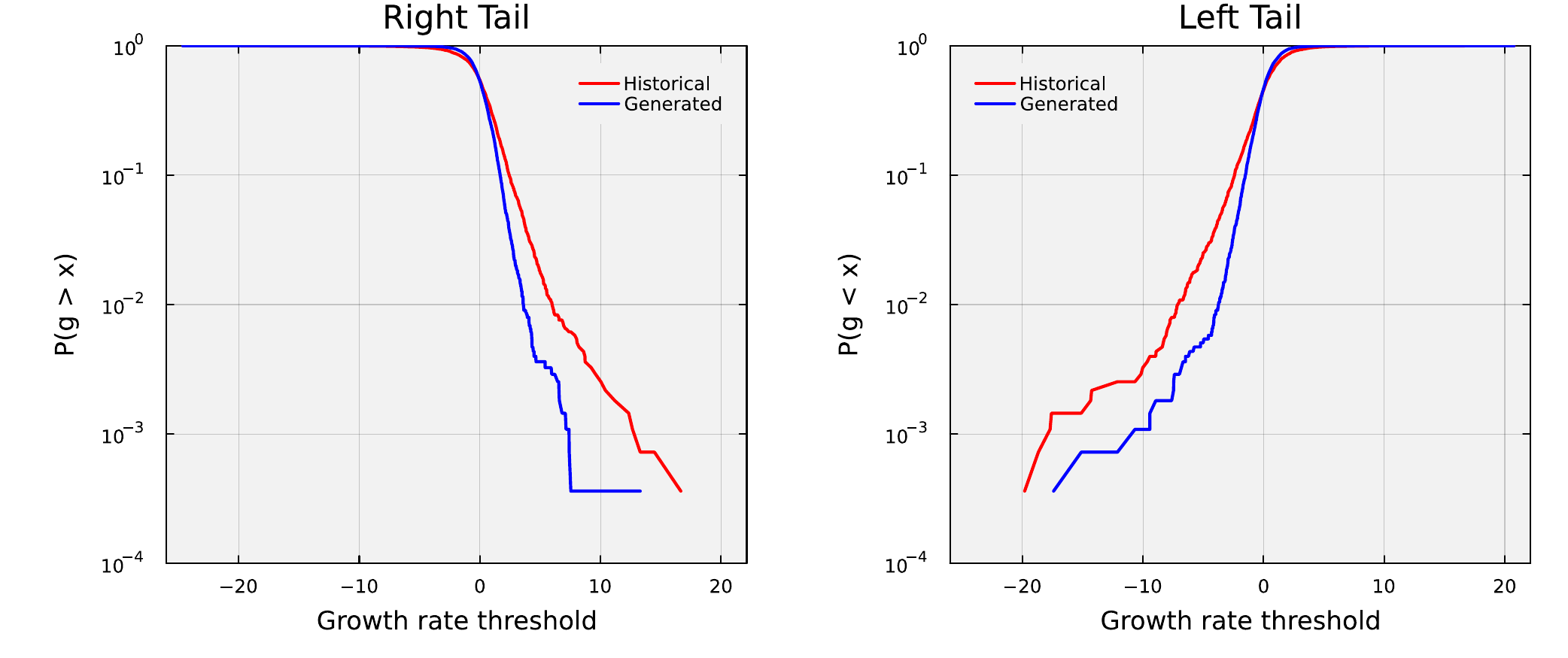}
\caption{Empirical survival functions of equal-weighted portfolio returns. Both right tail $P(g > x)$ and left tail $P(g < x)$ match closely on a log scale.}
\label{fig:finance-tails}
\end{figure}

\subsection{Sequential Generation and Temporal Properties}

The i.i.d.\ pool above validates marginal fidelity but says nothing about temporal dynamics. Real equity returns exhibit two key temporal properties: (i)~returns are essentially unpredictable (near-zero autocorrelation), and (ii)~squared returns (a proxy for volatility) are highly autocorrelated and decay slowly, a phenomenon known as \emph{volatility clustering}~\cite{contStylizedFactsAsset2001}. We tested how far the stochastic attention sampler could go toward reproducing these properties.

We generated a sequential time series of daily return vectors using the MALA sampler (Algorithm~\ref{alg:mala}) with warm-starting: each day's chain was initialized at the previous day's endpoint, creating temporal continuity in the chain trajectory.
We set $\beta = 12$ (below the main operating point $\beta{=}25$) to ensure adequate MALA mixing within $T_{\text{inner}}{=}200$ inner steps per day.
The pseudocode is in Algorithm~\ref{alg:mala-sequential}.

\begin{algorithm}[h]
\caption{MALA Warm-Start Sequential Return Generation}\label{alg:mala-sequential}
\begin{algorithmic}[1]
\REQUIRE Scaled memory matrix $\mathbf{X}\in\mathbb{R}^{d\times K}$, raw memory matrix $\mathbf{M}\in\mathbb{R}^{d\times K}$, inverse temperature $\beta$, step size $\alpha$, inner steps $T_{\text{inner}}$, number of days $T_{\text{days}}$
\ENSURE Synthetic return matrix $\hat{\mathbf{G}}\in\mathbb{R}^{T_{\text{days}}\times d}$
\STATE $k \sim \mathrm{Uniform}\{1,\dots,K\}$
\STATE $\boldsymbol{\xi} \leftarrow \mathbf{X}_{:,k} + 0.01\,\boldsymbol{\eta}, \quad \boldsymbol{\eta}\sim\mathcal{N}(\mathbf{0},\mathbf{I}_d)$ \hfill $\triangleright$ initialize near random memory
\FOR{$t = 1, 2, \dots, T_{\text{days}}$}
    \STATE Run $T_{\text{inner}}$ steps of Algorithm~\ref{alg:mala} from $\boldsymbol{\xi}$; set $\boldsymbol{\xi} \leftarrow$ endpoint \hfill $\triangleright$ warm-start
    \STATE $\hat{\mathbf{G}}_{t,:} \leftarrow \mathbf{M}\,\operatorname{softmax}(\beta\,\mathbf{X}^{\top}\boldsymbol{\xi})$ \hfill $\triangleright$ project to growth-rate space
\ENDFOR
\end{algorithmic}
\end{algorithm}

We generated $T_{\text{days}} = 2{,}766$ synthetic trading days (matching the historical sample size) with $T_{\text{inner}} = 200$, $\alpha = 0.01$, and $\beta = 12$; the mean MALA acceptance rate was $99.4\%$, confirming that the ULA discretization bias was negligible at this operating point. Near-zero return autocorrelation was reproduced cleanly: the autocorrelation function of raw returns for five representative tickers showed the generated ACF (blue) lying within the $99\%$ confidence interval for white noise at all lags $\geq 1$, matching the historical pattern (red) (Fig.~\ref{fig:finance-acf-combined}, top row). Volatility clustering told a different story (Fig.~\ref{fig:finance-acf-combined}, bottom row): the historical squared-return series (red) exhibited the characteristic slow-decaying positive autocorrelation of volatility clustering, whereas the generated series (blue) showed ACF($g^2$) $\approx 0$ at all lags, so the sampler did not reproduce this effect. We verified that at the main operating point ($\beta{=}25$, $T_{\text{inner}}{=}2{,}000$), ACF($g^2$) also collapsed to ${\approx}0$ (mean over $424$ tickers: $0.0005$; Table~\ref{tab:finance-summary}), confirming the absence was a property of equilibrium sampling and not an artifact of the choice of~$\beta$.

This is a fundamental limitation of equilibrium Boltzmann sampling, not a tuning failure. The stochastic attention sampler targets the stationary distribution $p_\beta \propto \exp(-\beta E)$ at fixed~$\beta$. At stationarity, the variance of each day's output is governed by the static energy landscape and the fixed temperature; there is no mechanism for the return variance to change systematically over time. Volatility clustering in real markets arises from non-stationary dynamics: regime shifts driven by exogenous shocks, feedback loops between volatility and leverage, and time-varying risk premia~\cite{contStylizedFactsAsset2001}. Reproducing these effects within the Langevin framework would require introducing explicit temporal structure, for instance a time-varying inverse temperature $\beta(t)$ that creates exogenous volatility regimes, or coupling the sampler to a latent regime-switching process. These extensions are natural directions for future work.

Three theoretical predictions were made concrete by the summary numbers (Table~\ref{tab:finance-summary}). \emph{First}, the novelty-fidelity trade-off was real and quantified: SA generated regime interpolations absent from the historical record ($\mathcal{N}=0.768\pm0.001$) while bootstrap, which replayed history verbatim, achieved $\mathcal{N}=0.000$; marginal fidelity ran in the opposite direction ($64.2\%$ vs.\ $99.3\%$ KS pass). \emph{Second}, the Hopfield energy captured cross-sectional dependence independently of marginals: Frobenius error $26.3\%$ was unchanged by per-ticker affine correction, confirming that the energy geometry encoded correlation structure, not scale. \emph{Third}, volatility clustering, a non-stationary phenomenon driven by regime shifts, the leverage effect linking falling prices to rising volatility~\cite{black1976studies,christie1982stochastic}, and time-varying risk premia~\cite{contStylizedFactsAsset2001}, was exactly what a fixed-$\beta$ equilibrium sampler could not reproduce; the absence was a theoretical prediction confirmed by experiment, not a tuning failure.

\begin{table}[h]
\centering
\caption{Finance experiment: principled trade-offs of equilibrium Boltzmann sampling. SA and bootstrap occupy opposite ends of a novelty--fidelity spectrum: SA generates novel regime interpolations ($\mathcal{N}=0.768\pm0.001$) that bootstrap cannot produce, while bootstrap preserves marginals by construction. The Frobenius correlation error is identical before and after affine correction, confirming the energy captures dependence independently of scale. Temporal metrics apply to the sequential MALA experiment only.}
\label{tab:finance-summary}
\small
\begin{tabular}{@{}lll@{}}
\toprule
Property & SA (corrected) & Bootstrap \\
\midrule
\multicolumn{3}{@{}l@{}}{\textit{Distributional fidelity (i.i.d.\ pool, 900 samples)}} \\
KS $p$-value (mean; 64.2\% pass at $5\%$) & 0.193 & 0.622 \\
Cross-asset Frobenius error & 26.3\% & 15.6\% \\
Novelty ($1-\max_k\cos$) & $0.768\pm0.001$ & 0.000 \\
\multicolumn{3}{@{}l@{}}{\textit{Temporal properties (MALA warm-start, 2766 days)}} \\
SF2: Return unpredictability & Reproduced [ACF($g$) in 99\% CI, $\beta{=}12$] & N/A \\
SF3: Volatility clustering & Not reproduced ($\beta{=}12$) & N/A \\
\bottomrule
\end{tabular}
\end{table}

\begin{figure}[h]
\centering
\includegraphics[width=0.95\textwidth]{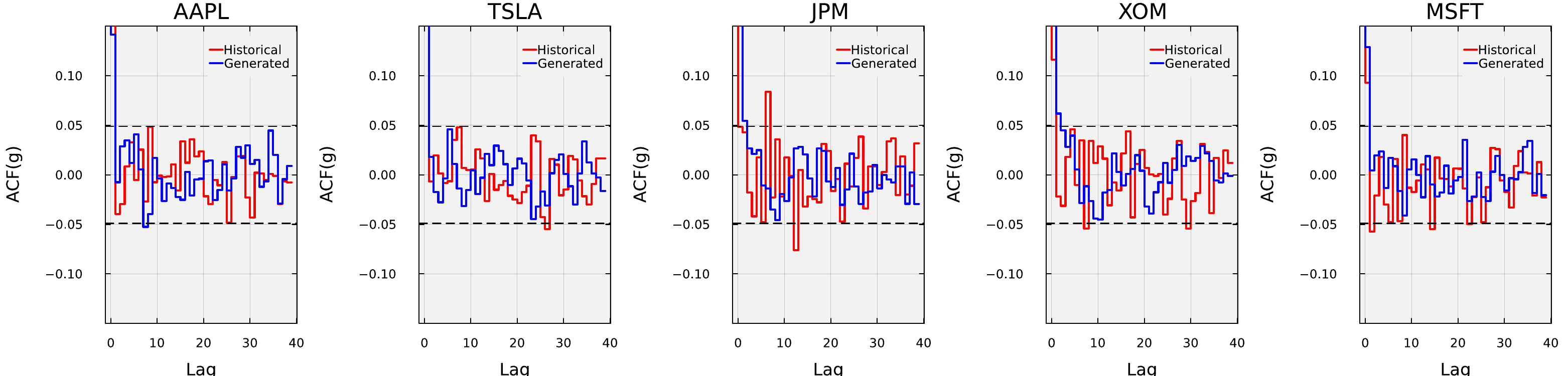}\\[4pt]
\includegraphics[width=0.95\textwidth]{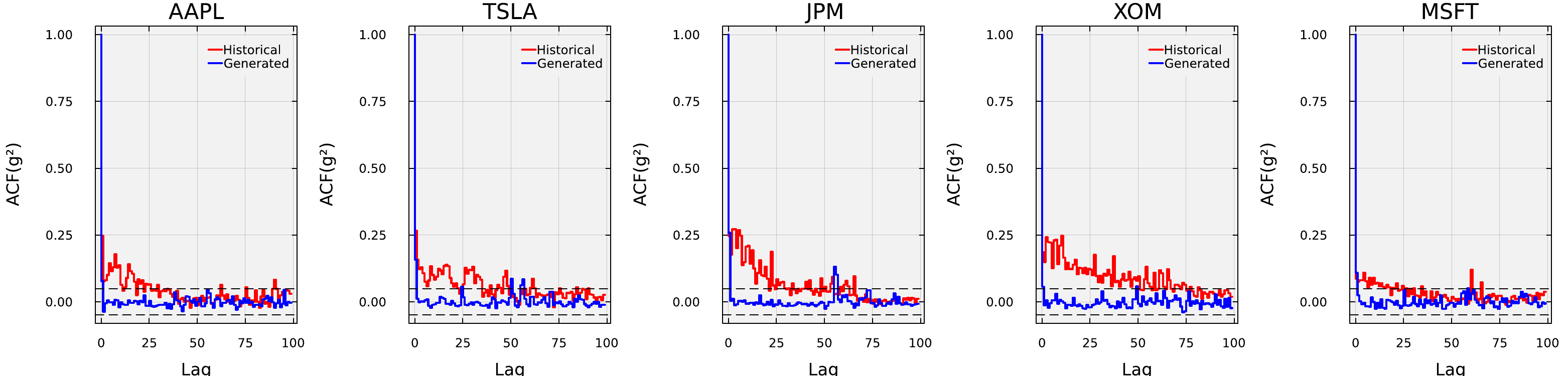}
\caption{Stylized Facts~2 and~3: Autocorrelation analysis for five S\&P~500 tickers.
\textbf{Top row}: ACF($g$): the generated series (blue) stays within the 99\% confidence band (dashed) at all lags, matching the near-zero autocorrelation of historical returns (red).
\textbf{Bottom row}: ACF($g^2$): the historical series (red) shows persistent positive autocorrelation characteristic of volatility clustering, while the generated series (blue) shows ACF($g^2$) $\approx 0$, reflecting the absence of non-stationary dynamics in the equilibrium sampler.}
\label{fig:finance-acf-combined}
\end{figure}

\section{Olivetti Faces and Conditional Generation by Masking}\label{app:olivetti}

This appendix supplies the full setup for the Olivetti experiment and the masked update of the main text, and reports the single-$\beta$ sweep that demonstrates the Hopfield $\to$ Hebbian transition under a fixed mask. The Olivetti faces dataset~\cite{samariaParameterisationStochasticModel1994} contains 400 grayscale portraits of 40 subjects, 10 portraits per subject, each $64\times 64$. We used the first 10 subjects (100 portraits), flattened each portrait in row-major order to a vector in $\mathbb{R}^{4096}$, and $\ell_2$-normalized the columns of the memory matrix $\mathbf{X}\in\mathbb{R}^{4096\times 100}$ (load ratio $K/d\approx 0.024$). Subject labels formed a length-100 vector $\mathbf{y}\in\{0,\dots,9\}^{K}$ used only by the evaluation classifier; the sampler itself saw $\mathbf{X}$ alone.

We adopted a two-phase $\beta$ schedule: a Langevin chain at $\beta_{\mathrm{chain}}{=}200$ for $T{=}3{,}000$ steps with $\alpha{=}0.01$ and $\sigma_{\mathrm{init}}{=}0.05$, followed by a single noise-free attention readout at $\beta_{\mathrm{read}}{=}10{,}000$ to project the chain endpoint onto the convex hull of (kept) memories. The chain temperature placed the sampler in the diffuse-generation regime where the noise scale $\sqrt{2\alpha/\beta_{\mathrm{chain}}}=0.01$ was large enough for the chain to wander between subjects; the readout temperature recovered a sharp MAP-style read on the memory manifold, mirroring the standard low-$\beta$-mix-then-high-$\beta$-read pattern used in score-based samplers. For evaluation we used the training-free nearest-mean classifier of Samaria \& Harter~\cite{samariaParameterisationStochasticModel1994}: for each subject $c$, compute the per-subject mean portrait $\bar{\mathbf{m}}_{c}=\frac{1}{|\mathcal{S}_{c}|}\sum_{k\in\mathcal{S}_{c}}\mathbf{m}_{k}$ where $\mathcal{S}_{c}=\{k : y_{k}=c\}$, and assign a sample $\mathbf{s}$ to $\arg\min_{c}\|\mathbf{s}-\bar{\mathbf{m}}_{c}\|_{2}$. The classifier has no trainable parameters and uses the same $\mathbf{X}$ as the sampler.

We measured subject recovery in three regimes (Table~\ref{tab:olivetti-recovery}): random guess (uniform over the 10 subjects), unconditional SA warm-started uniformly across all subjects, and hard-masked SA with $b_{k}=-\infty$ for memories outside the target subject. The hard mask raised subject recovery from $10.4\%$ (essentially the random-guess baseline) to $96.0\%$, a $9.2{\times}$ improvement, with no change to the sampler or any learned components. The four misses out of 100 ($96.0\%$) reflected cases where the masked chain endpoint was closer to another subject's mean portrait than to the target subject's mean; in every such case the sample was a clean blend of the target's stored portraits and was recognizable as the correct subject visually (Fig.~\ref{fig:olivetti}b).

\begin{table}[h]
\centering
\caption{Subject-recovery rate on Olivetti faces ($K{=}100$, $d{=}4{,}096$, 10 subjects). Hard-masked SA recovers the target subject $9.2\times$ more reliably than unconditional SA, matching random-guess. ``Hard-mask'' was computed across $5\times 5=25$ samples spanning the first 5 subjects; ``unconditional'' across $50$ chains warm-started uniformly across all 10 subjects, averaged against the same 5 target subjects.}
\label{tab:olivetti-recovery}
\small
\begin{tabular}{@{}lc@{}}
\toprule
Regime & Subject-recovery rate $\uparrow$ \\
\midrule
Random guess (uniform over 10 subjects) & $0.100$ \\
Unconditional SA, uniform warm-start    & $0.104$ \\
\textbf{Hard-masked SA ($b_{k}{=}-\infty$ off-subject)} & $\mathbf{0.960}$ \\
\bottomrule
\end{tabular}
\end{table}

To trace the Hopfield $\to$ Hebbian transition under a fixed mask, we held the subject mask fixed (subject 0) and swept the inverse temperature $\beta$ over six points spanning three decades, using the same $\beta$ for both chain dynamics and readout. Novelty was the median of $1-\max_{k\in\mathcal{S}_{c}}\cos(\mathbf{s},\mathbf{m}_{k})$, the smallest cosine distance to any kept portrait; hit rate was the nearest-mean classifier rate (Table~\ref{tab:olivetti-betasweep}). At $\beta{=}2$ the readout was a near-uniform soft blend of the 10 in-subject portraits and samples sat at small but non-zero distance from any single one (novelty $0.009$, hit rate $1.000$); by $\beta{=}512$ the sharp readout collapsed each chain to a single stored portrait (novelty $0.000$). The hit rate stayed in $[0.90, 1.00]$ across the full range because the $-\infty$ logits make the off-subject subset unreachable; only the within-subject blend $\to$ recall transition varies with $\beta$.

\begin{table}[h]
\centering
\caption{Single-$\beta$ masked sampler on subject 0 of Olivetti, $20$ chains per row. Novelty is the median minimum cosine distance to any in-subject stored portrait; hit rate is the nearest-mean classifier rate. The sweep recovers the modern-Hopfield $\to$ Hebbian transition empirically inside the masked regime.}
\label{tab:olivetti-betasweep}
\small
\begin{tabular}{@{}lcccccc@{}}
\toprule
$\beta$ & $2$ & $8$ & $32$ & $128$ & $512$ & $2048$ \\
\midrule
Novelty $\uparrow$ & $0.009$ & $0.009$ & $0.006$ & $0.005$ & $0.000$ & $0.000$ \\
Hit rate $\uparrow$ & $1.000$ & $1.000$ & $0.950$ & $0.900$ & $0.900$ & $0.900$ \\
\bottomrule
\end{tabular}
\end{table}

Replacing $b_{k}\in\{0,-\infty\}$ with a finite scalar $b_{k}\in\mathbb{R}$ recovers the multiplicity-weight conditioning of Varner~\cite{varnerConditioningProtein2026} and the synthetic-patient generator of Varner et al.~\cite{varnerSyntheticPatient2026}. The hard-mask version studied here is the limiting case as the bias on out-of-set memories diverges to $-\infty$; the regularity guarantees of the main text hold for both regimes since both are equivalent to running unmasked SA on a (re-weighted) reduced memory matrix.

% ─────────────────────────────────────────────────────────────────────────────
\section{Baseline Method Specifications}\label{app:baseline-methods}
% ─────────────────────────────────────────────────────────────────────────────

Complete mathematical specifications for the two learned generative
baselines (GMM-PCA and the VAE; Table~\ref{tab:mnist-baselines}) follow.
All hyperparameters reported here were used without modification in the experiments;
the implementation is publicly available in \texttt{code/mnist-experiment/} and
\texttt{code/vae-experiment/} in the accompanying repository.

% ── GMM-PCA ──────────────────────────────────────────────────────────────────
\subsection{GMM-PCA}\label{app:gmm-pca}

\paragraph{Dimensionality reduction.}
Let $\mathbf{X}\in\mathbb{R}^{d\times K}$ be the column-normalized memory matrix
($d{=}784$, $K{=}100$). Compute the economy SVD
$\mathbf{X}=\mathbf{U}\boldsymbol{\Sigma}\mathbf{V}^{\!\top}$ and retain the top
$r{=}50$ left singular vectors, forming $\mathbf{U}_{r}\in\mathbb{R}^{d\times r}$.
Project each stored pattern to a low-dimensional code:
$\mathbf{z}_{k}=\mathbf{U}_{r}^{\!\top}\mathbf{m}_{k}\in\mathbb{R}^{r}$, $k=1,\dots,K$.
For the digit ``3'' memory matrix, the top 50 components capture $>$95\% of total variance.

\paragraph{Gaussian mixture model.}
Fit a $C{=}10$ component GMM with \emph{diagonal} covariance to the code matrix
$\mathbf{Z}{=}[\mathbf{z}_1|\cdots|\mathbf{z}_K]\in\mathbb{R}^{r\times K}$ via
Expectation-Maximization~\cite{dempsterMaximumLikelihoodIncomplete1977}:
\begin{equation}
  p(\mathbf{z}) = \sum_{c=1}^{C}\pi_c\,
  \mathcal{N}\!\bigl(\mathbf{z};\,\boldsymbol{\mu}_c,\,\mathrm{diag}(\boldsymbol{\sigma}_c^2)\bigr),
  \quad \sum_{c=1}^{C}\pi_c=1,\quad \pi_c\geq 0.
\end{equation}
Parameters $\{\pi_c,\boldsymbol{\mu}_c,\boldsymbol{\sigma}_c^2\}_{c=1}^{C}$ are estimated
by EM with K-means$\texttt{++}$ initialization and convergence tolerance $10^{-8}$ on the
log-likelihood.
The diagonal-covariance constraint prevents singularity when $K>r$ and is equivalent to
assuming conditionally independent PCA coordinates within each component.

\paragraph{Sampling.}
To draw one sample: (i) sample component $c\sim\mathrm{Categorical}(\boldsymbol{\pi})$;
(ii) sample $\tilde{\mathbf{z}}\sim\mathcal{N}(\boldsymbol{\mu}_c,\mathrm{diag}(\boldsymbol{\sigma}_c^2))$;
(iii) reconstruct $\hat{\boldsymbol{\xi}}=\mathbf{U}_{r}\tilde{\mathbf{z}}$ and
renormalize to unit $\ell_2$ norm.
The 150 evaluation samples are drawn i.i.d.\ under this procedure with a fixed seed.

\paragraph{Hyperparameters.}
$r{=}50$ PCA components; $C{=}10$ GMM components; K-means$\texttt{++}$ initialization;
EM convergence tolerance $10^{-8}$; no regularization beyond the diagonal-covariance constraint.

% ── VAE ──────────────────────────────────────────────────────────────────────
\subsection{Variational Autoencoder}\label{app:vae}

\paragraph{Architecture.}
The encoder maps $\mathbf{x}\in\mathbb{R}^{784}$ to the parameters of an approximate posterior:
\begin{align}
  \mathbf{h}_1 &= \mathrm{ReLU}(\mathbf{W}_1\mathbf{x}+\mathbf{b}_1),\quad
    \mathbf{W}_1\in\mathbb{R}^{256\times 784}, \\
  \mathbf{h}_2 &= \mathrm{ReLU}(\mathbf{W}_2\mathbf{h}_1+\mathbf{b}_2),\quad
    \mathbf{W}_2\in\mathbb{R}^{128\times 256}, \\
  \boldsymbol{\mu}_\phi(\mathbf{x}) &= \mathbf{W}_\mu\mathbf{h}_2+\mathbf{b}_\mu,\quad
    \mathbf{W}_\mu\in\mathbb{R}^{8\times 128}, \\
  \log\boldsymbol{\sigma}^2_\phi(\mathbf{x}) &= \mathbf{W}_\sigma\mathbf{h}_2+\mathbf{b}_\sigma,\quad
    \mathbf{W}_\sigma\in\mathbb{R}^{8\times 128}.
\end{align}
The decoder is architecturally symmetric: $8\to128\to256\to784$, with ReLU activations on
hidden layers and a linear output.
All weight matrices are Glorot-uniform initialized; biases are zero-initialized.
The decoder output is renormalized to unit $\ell_2$ norm before computing metrics, matching
the normalization convention of all other methods in Table~\ref{tab:mnist-baselines}.

\paragraph{Objective.}
We minimize the $\beta$-VAE objective~\cite{higginsBetaVAELearningBasic2017}:
\begin{equation}
  \mathcal{L}(\phi,\theta;\mathbf{x})
  = \underbrace{\bigl\lVert\mathbf{x}-\hat{\mathbf{x}}_\theta(\mathbf{z})\bigr\rVert_2^2}
    _{\text{reconstruction}}
  +\; \beta_{\mathrm{KL}}\underbrace{D_{\mathrm{KL}}\!\bigl(
    q_\phi(\mathbf{z}|\mathbf{x})\,\big\|\,\mathcal{N}(\mathbf{0},\mathbf{I})\bigr)}
    _{\text{regularization}},
\label{eq:vae-elbo}
\end{equation}
where $q_\phi(\mathbf{z}|\mathbf{x})=\mathcal{N}(\boldsymbol{\mu}_\phi(\mathbf{x}),
\mathrm{diag}(\boldsymbol{\sigma}^2_\phi(\mathbf{x})))$ and the reparameterization trick
$\mathbf{z}=\boldsymbol{\mu}_\phi(\mathbf{x})+\boldsymbol{\sigma}_\phi(\mathbf{x})\odot
\boldsymbol{\epsilon}$, $\boldsymbol{\epsilon}\sim\mathcal{N}(\mathbf{0},\mathbf{I})$,
enables gradient flow through the sampling step.
The KL term has a closed form:
\begin{equation}
  D_{\mathrm{KL}} = \tfrac{1}{2}\sum_{j=1}^{8}\bigl[
    \mu_{\phi,j}^2+\sigma_{\phi,j}^2 - \log\sigma_{\phi,j}^2 - 1\bigr].
\end{equation}

\paragraph{Two-phase training protocol.}
Training directly with the full $\beta$-VAE objective on a small dataset ($K{=}100$)
causes \emph{posterior collapse}: the encoder ignores the input, the KL term drives every
latent dimension to $\mathcal{N}(0,1)$, and the decoder learns a constant map.
We prevented this with a two-phase schedule. The warm-up phase (epochs 1--2{,}000) trained with $\beta_{\mathrm{KL}}{=}0$, i.e.\ a reconstruction-only autoencoder objective, to establish a non-degenerate encoder (one that uses the latent code productively) before any regularization pressure was applied. The fine-tuning phase (epochs 2{,}001--4{,}000) introduced the KL term with $\beta_{\mathrm{KL}}{=}0.0001$ and a linear warmup schedule ($\beta_{\mathrm{KL}}(t) = 0.0001\cdot(t/2000)$ for $t\in[1,2000]$); the small final weight provides light regularization toward the Gaussian prior without overriding the reconstruction signal. The warm-up phase is equivalent to a hard $\beta$-annealing schedule and is standard practice for low-data VAE training~\cite{lucasHowAvoidTrivial2019}.
On large datasets ($K\gg10^4$), the reconstruction gradient swamps the KL term and posterior
collapse does not occur; the two-phase protocol is needed here specifically because
$K{=}100$ makes the per-sample gradients noisy relative to the KL regularization pressure.

\paragraph{Optimization and sampling.}
Optimizer: Adam~\cite{kingmaAdamMethodStochastic2014} with learning rate $10^{-3}$,
$\beta_1{=}0.9$, $\beta_2{=}0.999$, $\varepsilon{=}10^{-8}$, batch size $K{=}100$
(full-dataset batches throughout).
No weight decay, dropout, or data augmentation was applied.
At evaluation time, 150 samples are drawn by sampling
$\mathbf{z}\sim\mathcal{N}(\mathbf{0},\mathbf{I}_8)$ and decoding
$\hat{\mathbf{x}}=\hat{\mathbf{x}}_\theta(\mathbf{z})$, then renormalizing to unit norm.

\paragraph{Hyperparameter selection.}
The final configuration (latent dim 8, $\beta_{\mathrm{KL}}{=}0.0001$) was chosen by a
grid search over latent dimension $\in\{4,8,16,32\}$ and $\beta_{\mathrm{KL}}\in\{0.1,0.01,0.001,0.0001\}$,
evaluating novelty and diversity on the held-out generation protocol.
Latent dim 8 with $\beta_{\mathrm{KL}}{=}0.0001$ achieved $\mathcal{N}{=}0.214\pm0.005$,
$\bar{\mathcal{D}}{=}0.441\pm0.008$, outperforming all other static baselines including
GMM-PCA ($\mathcal{N}{=}0.198$, $\bar{\mathcal{D}}{=}0.419$).
Smaller latent dimensions or larger $\beta_{\mathrm{KL}}$ values degraded novelty and
diversity, while larger latent dimensions showed posterior collapse symptoms
(diversity collapsing to $\bar{\mathcal{D}}{<}0.05$ at latent dim 32 with $\beta_{\mathrm{KL}}{=}0.01$).

% ── DDPM ────────────────────────────────────────────────────────────────────
\subsection{Denoising Diffusion Probabilistic Model (DDPM)}\label{app:ddpm}

We implement a standard DDPM~\cite{hoDenoisingDiffusionProbabilistic2020} operating on the
same flattened, unit-norm MNIST vectors used by all other methods.

\paragraph{Architecture.}
Since the data are flattened 784-dimensional vectors (not 2D images), we use an MLP denoiser
rather than a convolutional U-Net. The input at each timestep~$t$ is the concatenation of the
noisy vector $\mathbf{x}_t\in\mathbb{R}^{784}$ and a sinusoidal timestep embedding
$\mathbf{e}_t\in\mathbb{R}^{64}$, giving an 848-dimensional input. The network is
$848\to512~(\mathrm{ReLU})\to256~(\mathrm{ReLU})\to512~(\mathrm{ReLU})\to784$,
predicting the noise $\boldsymbol{\epsilon}$.

\paragraph{Diffusion schedule.}
We use $T{=}200$ diffusion steps with a linear variance schedule
$\beta_1{=}10^{-4},\;\beta_T{=}0.02$, following Ho et al.~\cite{hoDenoisingDiffusionProbabilistic2020}.
The forward process is $q(\mathbf{x}_t|\mathbf{x}_0)=\mathcal{N}(\sqrt{\bar\alpha_t}\,\mathbf{x}_0,\;(1{-}\bar\alpha_t)\mathbf{I})$
where $\bar\alpha_t=\prod_{s=1}^{t}(1-\beta_s)$.

\paragraph{Training.}
We train for $5{,}000$ epochs using Adam~\cite{kingmaAdamMethodStochastic2014}
(learning rate $10^{-3}$) with full-batch updates ($K{=}100$).
At each epoch, we sample random timesteps $t\sim\mathrm{Uniform}\{1,\dots,T\}$ for each training
image, draw noise $\boldsymbol{\epsilon}\sim\mathcal{N}(\mathbf{0},\mathbf{I})$, compute
$\mathbf{x}_t=\sqrt{\bar\alpha_t}\,\mathbf{x}_0+\sqrt{1{-}\bar\alpha_t}\,\boldsymbol{\epsilon}$,
and minimize the MSE loss $\lVert\boldsymbol{\epsilon}-\hat{\boldsymbol{\epsilon}}_\theta(\mathbf{x}_t,t)\rVert_2^2$.

\paragraph{Sampling.}
We generate 150 samples via the standard reverse process starting from
$\mathbf{x}_T\sim\mathcal{N}(\mathbf{0},\mathbf{I})$ and iterating
$\mathbf{x}_{t-1}=\frac{1}{\sqrt{\alpha_t}}\bigl(\mathbf{x}_t-\frac{\beta_t}{\sqrt{1-\bar\alpha_t}}\hat{\boldsymbol{\epsilon}}_\theta(\mathbf{x}_t,t)\bigr)+\sqrt{\beta_t}\,\mathbf{z}$
for $t=T,\dots,1$. Each output is $\ell_2$-normalized to match the convention of all other methods.

\paragraph{Result.}
The DDPM produced samples with $\mathcal{N}{=}0.938$, $\bar{\mathcal{D}}{=}0.991$, and
max-cosine similarity to the nearest stored pattern of $0.062$, indistinguishable from
isotropic Gaussian noise ($0.057$; Table~\ref{tab:scaling}, Table~\ref{tab:noise-control}).
The training loss plateaued at ${\approx}0.855$ (close to the expected MSE for predicting
random noise), indicating that the denoiser failed to learn the data distribution.
This is the expected behavior: DDPM requires thousands to millions of training images to
estimate a score function across $T{=}200$ noise levels; with only $K{=}100$ images in
$\mathbb{R}^{784}$, the model is severely underspecified.
Stochastic attention avoids this failure mode entirely because its score function is available
in closed form from the Hopfield energy, requiring no training data beyond the stored patterns
themselves.

% ─────────────────────────────────────────────────────────────────────────────
\section{Step-Size Sweep: ULA vs.\ MALA}\label{app:stepsize-sweep}
% ─────────────────────────────────────────────────────────────────────────────

At the operating point used throughout this paper ($\alpha{=}0.01$, $\beta{=}2000$), MALA acceptance exceeds $99\%$ and the two algorithms are trivially equivalent. A natural question is: \emph{at what step size does the ULA discretization bias become significant, and how does sample quality degrade?} We swept $\alpha\in\{0.001,0.005,0.01,0.02,0.05,0.1,0.2,0.3,0.5\}$ on MNIST digit~``3'' ($K{=}100$, $d{=}784$, $\beta{=}2000$), running 30 chains per method with the same protocol as the main MNIST experiment (Table~\ref{tab:mnist-baselines}).

The sweep revealed three regimes (Table~\ref{tab:stepsize-sweep}, Fig.~\ref{fig:stepsize-sweep}). In the \emph{equivalent regime} ($\alpha\leq 0.02$), MALA acceptance exceeded $97\%$ and the ULA and MALA metrics were indistinguishable ($\Delta E < 0.003$). In the \emph{divergence regime} ($\alpha\in[0.05,0.1]$), acceptance dropped to $75$--$91\%$ and the ULA bias became detectable but small ($\Delta\mathcal{N}\approx 0.007$, $\Delta E \approx 0.011$). In the \emph{MALA-frozen regime} ($\alpha\geq 0.2$), MALA acceptance collapsed to $0\%$, freezing the chain at its initialization, while ULA continued to produce samples with gracefully degrading quality (rising energy, increasing novelty as the noise term dominated). The practical divergence threshold is $\alpha\approx 0.1$ ($75\%$ acceptance, $\Delta E = 0.011$). At the paper's operating point ($\alpha{=}0.01$), ULA bias is negligible ($\Delta E = 0.0018$). At large step sizes, ULA is \emph{preferable} to MALA: MALA freezes entirely because the large Langevin proposals are almost always rejected, while ULA, which unconditionally accepts every proposal, continues to explore the energy landscape.

\begin{table}[h]
\centering
\caption{Step-size sweep on MNIST digit ``3'' ($K{=}100$, $d{=}784$, $\beta{=}2000$). ULA and MALA are indistinguishable for $\alpha\leq 0.02$. At $\alpha\geq 0.2$, MALA acceptance collapses and the chain freezes, while ULA continues sampling.}
\label{tab:stepsize-sweep}
\small
\begin{tabular}{@{}rcccccc@{}}
\toprule
$\alpha$ & Accept & \multicolumn{2}{c}{Novelty} & \multicolumn{2}{c}{Diversity} & $\Delta E$ \\
 & rate & ULA & MALA & ULA & MALA & (ULA$-$MALA) \\
\midrule
0.001 & $>$0.99 & 0.151 & 0.150 & 0.595 & 0.595 & $<$0.001 \\
0.005 & $>$0.99 & 0.151 & 0.151 & 0.599 & 0.597 & $<$0.001 \\
0.01  & 0.992 & 0.152 & 0.151 & 0.600 & 0.598 & 0.002 \\
0.02  & 0.978 & 0.153 & 0.151 & 0.600 & 0.598 & 0.003 \\
0.05  & 0.911 & 0.155 & 0.151 & 0.602 & 0.598 & 0.006 \\
0.1   & 0.747 & 0.158 & 0.151 & 0.605 & 0.597 & 0.011 \\
0.2   & 0.000 & 0.165 & 0.000 & 0.612 & 0.000 & --- \\
0.3   & 0.000 & 0.172 & 0.000 & 0.619 & 0.000 & --- \\
0.5   & 0.000 & 0.189 & 0.000 & 0.635 & 0.000 & --- \\
\bottomrule
\end{tabular}
\end{table}

\begin{figure}[h]
\centering
\IfFileExists{figs/Fig_stepsize_sweep_ula_vs_mala.pdf}{%
  \includegraphics[width=0.95\textwidth]{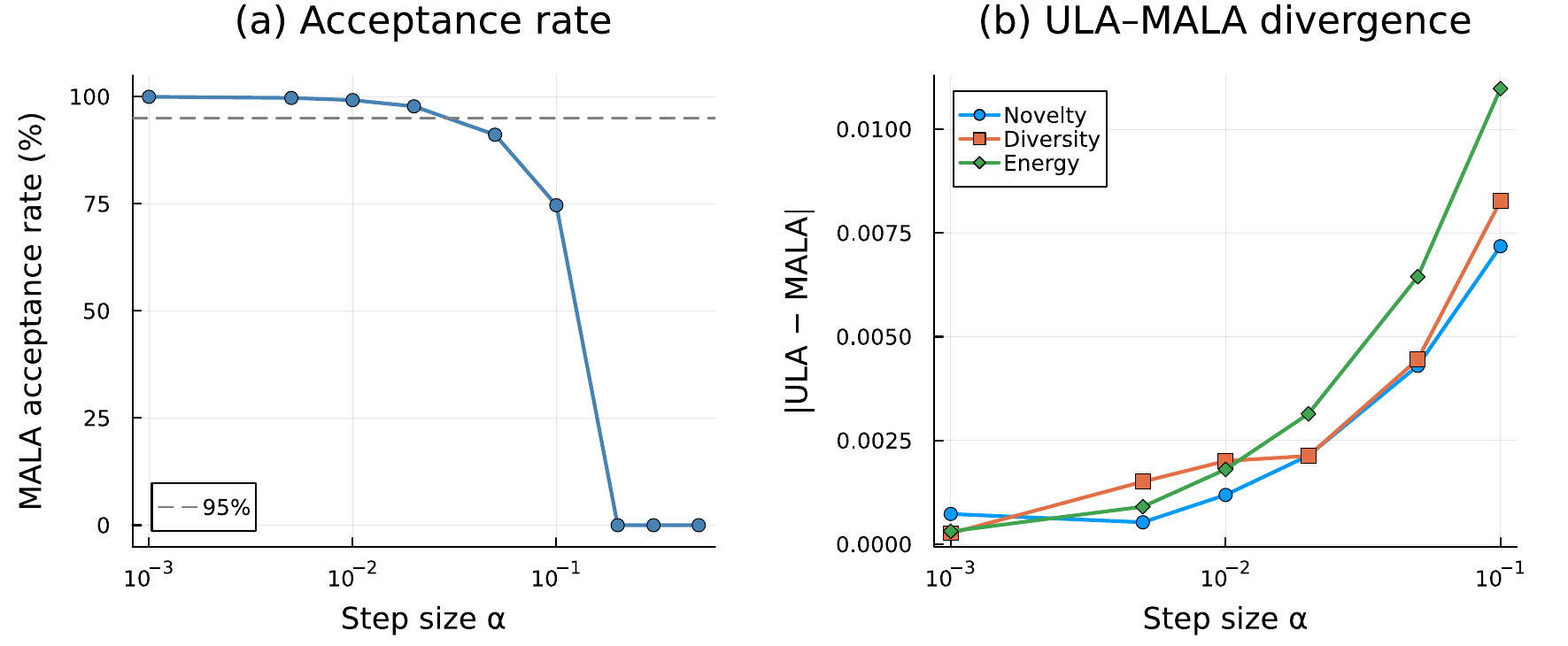}%
}{%
  \framebox[0.95\textwidth]{\rule{0pt}{4cm}\texttt{[Figure will be generated by run\_stepsize\_sweep.jl]}}%
}
\caption{Step-size sweep: ULA vs.\ MALA on MNIST digit ``3'' ($\beta{=}2000$). \textbf{(a)}~MALA acceptance rate vs.\ step size $\alpha$. Acceptance exceeds $95\%$ for $\alpha\leq 0.02$ and collapses to $0\%$ at $\alpha\geq 0.2$. \textbf{(b)}~Absolute difference between ULA and MALA metrics. All three metrics (novelty, diversity, energy) diverge smoothly as $\alpha$ increases through the $[0.05,0.1]$ band.}
\label{fig:stepsize-sweep}
\end{figure}

% ─────────────────────────────────────────────────────────────────────────────
\section{Gaussian Noise Control at $\beta{=}200$}\label{app:noise-control}
% ─────────────────────────────────────────────────────────────────────────────

At $\beta{=}200$ (SNR$=0.036$, the generation regime), SA achieved high novelty ($0.547$) and diversity ($0.885$) but the mean Hopfield energy was positive ($+1.47$), meaning samples lay outside the attractor manifold. A natural concern is that these metrics could be achieved trivially by high-temperature noise that bears no structural relationship to the stored patterns. We tested this by comparing SA against two Gaussian noise controls on MNIST digit~``3'' ($K{=}100$, $d{=}784$).

\emph{Matched Gaussian}: for each pixel $j$, we estimate the mean $\mu_j$ and variance $\sigma_j^2$ of the 150 SA samples and draw i.i.d.\ $\mathcal{N}(\mu_j, \sigma_j^2)$; this is the strongest possible Gaussian that matches SA's marginal pixel statistics. \emph{Isotropic Gaussian}: we draw $\mathbf{g}\sim\mathcal{N}(\mathbf{0},\mathbf{I}_d)$ and rescale to match the mean norm of SA samples ($\lVert\boldsymbol{\xi}\rVert\approx 2.22$). Both controls produce 150 samples.

The comparison is summarized below (Table~\ref{tab:noise-control}, Fig.~\ref{fig:noise-control}). The key separating metric was \emph{maximum cosine similarity} to the nearest stored pattern (Max-$\cos$): SA retained $0.453$, meaning each sample sat roughly halfway between the stored-pattern manifold ($1.0$) and random directions (${\approx}0.06$ in $\mathbb{R}^{784}$); the matched Gaussian achieved only $0.328$ despite sharing SA's per-pixel moments, and isotropic noise achieved $0.057$ (the random-direction baseline). Hopfield energy separated the methods in the same order: SA ($1.47$) $<$ matched Gaussian ($1.78$) $<$ isotropic ($2.34$). Novelty and diversity, by contrast, were similar between SA and the matched Gaussian ($0.547$ vs.\ $0.672$ and $0.885$ vs.\ $0.878$), confirming the concern that these metrics alone do not distinguish structured generation from noise; the max-cosine and energy metrics do, because SA samples were geometrically closer to stored patterns thanks to the Langevin dynamics being guided by the Hopfield energy gradient, which biased the chain toward the memory manifold even at high temperature. Gaussian noise, lacking this gradient signal, could not achieve the same structural similarity. Visually, SA samples showed faint but recognizable digit-3 structure (curved strokes, characteristic topology), while both Gaussian controls produced uniform static with no discernible spatial structure (Fig.~\ref{fig:noise-control}). $\beta{=}200$ is a high-temperature regime where sample quality is limited: the noise term $\sqrt{2\alpha/\beta}\,\boldsymbol{\epsilon}$ has standard deviation $0.01$ per component, comparable to the gradient step, so the samples are ``blurry-but-recognizable'' rather than crisp; higher-fidelity generation requires operating closer to the transition band ($\beta{\approx}500$--$1000$), trading novelty for structural quality.

\begin{table}[h]
\centering
\caption{Gaussian noise control on MNIST digit ``3'' at $\beta{=}200$ (generation regime). Max-$\cos$ and energy cleanly separate SA from both Gaussian controls: SA samples retain structural similarity to stored patterns that noise cannot replicate. FD$_{\text{diag}}$ is the diagonal pixel-space Fr\'{e}chet distance to stored patterns.}
\label{tab:noise-control}
\small
\begin{tabular}{@{}lccccc@{}}
\toprule
Method & $\mathcal{N}$ $\uparrow$ & $\bar{\mathcal{D}}$ $\uparrow$ & $\bar{E}$ $\downarrow$ & Max-$\cos$ $\uparrow$ & FD$_{\text{diag}}$ $\downarrow$ \\
\midrule
Stored patterns                   & 0.000 & 0.459 & $-0.50$ & 1.000 & ---  \\
\textbf{SA ($\beta{=}200$)}       & 0.547 & 0.885 & $+1.47$ & 0.453 & 2.81 \\
Gaussian (matched $\mu,\sigma^2$) & 0.672 & 0.878 & $+1.78$ & 0.328 & 2.86 \\
Gaussian (isotropic)              & 0.943 & 1.000 & $+2.34$ & 0.057 & 3.91 \\
\bottomrule
\end{tabular}
\end{table}

\begin{figure}[h]
\centering
\includegraphics[width=0.19\textwidth]{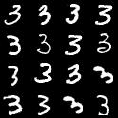}%
\hfill
\includegraphics[width=0.19\textwidth]{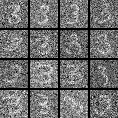}%
\hfill
\includegraphics[width=0.19\textwidth]{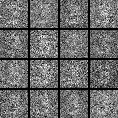}%
\hfill
\includegraphics[width=0.19\textwidth]{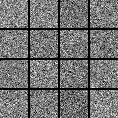}%
\hfill
\IfFileExists{figs/Fig_noise_control_maxcos_hist.pdf}{%
  \includegraphics[width=0.19\textwidth]{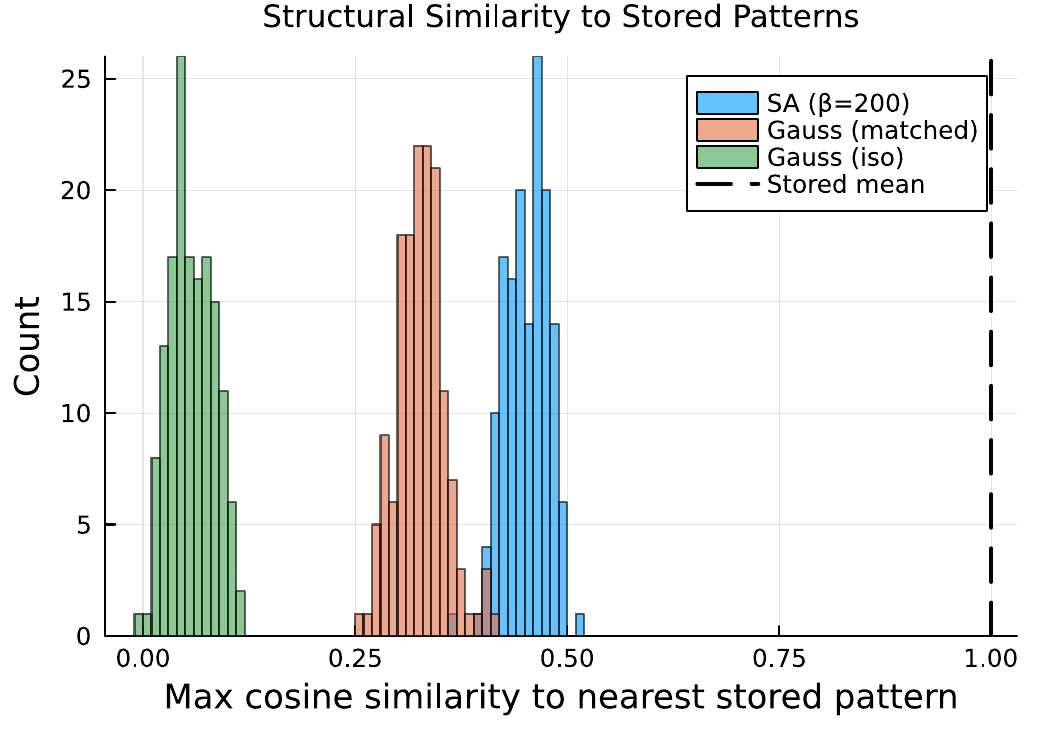}%
}{%
  \framebox[0.19\textwidth]{\rule{0pt}{3cm}\small hist}%
}

\small
\makebox[0.19\textwidth]{\centering (a) Stored}%
\hfill
\makebox[0.19\textwidth]{\centering (b) SA ($\beta{=}200$)}%
\hfill
\makebox[0.19\textwidth]{\centering (c) Gauss (matched)}%
\hfill
\makebox[0.19\textwidth]{\centering (d) Gauss (iso)}%
\hfill
\makebox[0.19\textwidth]{\centering (e) Max-$\cos$ dist.}
\caption{Gaussian noise control. \textbf{(a)}~Stored digit ``3'' patterns. \textbf{(b)}~SA at $\beta{=}200$: noisy but spatially structured (curved strokes visible). \textbf{(c)}~Matched Gaussian (same per-pixel $\mu$ and $\sigma^2$ as SA): uniform static, no digit structure. \textbf{(d)}~Isotropic Gaussian (matched norm): pure noise. \textbf{(e)}~Distribution of max cosine similarity to nearest stored pattern: SA (blue) is clearly shifted right of both Gaussian controls, confirming that the Hopfield energy gradient preserves structural similarity that noise cannot replicate.}
\label{fig:noise-control}
\end{figure}

% ─────────────────────────────────────────────────────────────────────────────
\section{Scaling Experiment: SA vs.\ DDPM}\label{app:scaling}
% ─────────────────────────────────────────────────────────────────────────────

The $K{=}100$ MNIST experiment (Table~\ref{tab:mnist-baselines}) showed that DDPM fails entirely when training data are scarce, but this comparison is arguably unfair to DDPM: diffusion models are designed for large-data regimes. We therefore ran a scaling study at $K\in\{100,\,500,\,1{,}000,\,2{,}000,\,3{,}500\}$ digit-3 images (the maximum available in MNIST) to test whether DDPM catches up as $K$ grows and to characterize how SA behaves under a principled, data-dependent $\beta$ selection. At each~$K$, we set $\beta$ using the entropy inflection criterion of Proposition~\ref{prop:entropy-derivative}: we swept $\beta$ over a grid, computed the attention entropy $H(\beta)=-\sum_k p_k\log p_k$ where $\mathbf{p}=\mathrm{softmax}(\beta\mathbf{X}^\top\boldsymbol{\xi})$ averaged over a probe set, and identified $\beta^*$ at the inflection point (maximum $|dH/d\beta|$); this gave a principled, data-dependent operating point for each $K$ without manual tuning, and for each $K$ we also reported SA at $5\beta^*$ (the retrieval regime) to illustrate the temperature knob. We used the same MLP DDPM architecture (Appendix~\ref{app:ddpm}) with $T{=}200$ diffusion steps at every~$K$, scaling training epochs inversely with $K$ to keep total gradient steps comparable: $5{,}000$ epochs at $K{=}100$ down to $1{,}000$ epochs at $K{=}3{,}500$. The key metric was max-cosine similarity to the nearest stored pattern, which distinguishes structured generation from noise (Table~\ref{tab:scaling}, Fig.~\ref{fig:scaling}).

\begin{table}[h]
\centering
\caption{Scaling experiment on MNIST digit ``3''. $\beta^*$ is set via entropy inflection at each $K$. SA operates at $\beta^*$ (generation) and $5\beta^*$ (retrieval). Max-$\cos$ measures structural similarity to stored patterns (isotropic noise baseline: ${\approx}0.06$).}
\label{tab:scaling}
\small
\begin{tabular}{@{}rrccccccccc@{}}
\toprule
$K$ & $K/d$ & $\beta^*$ & \multicolumn{3}{c}{SA (generation, $\beta{=}\beta^*$)} & \multicolumn{3}{c}{SA (retrieval, $\beta{=}5\beta^*$)} & \multicolumn{2}{c}{DDPM ($T{=}200$)} \\
\cmidrule(lr){4-6}\cmidrule(lr){7-9}\cmidrule(lr){10-11}
 & & & $\mathcal{N}$ & $\bar{\mathcal{D}}$ & max-$\cos$ & $\mathcal{N}$ & $\bar{\mathcal{D}}$ & max-$\cos$ & max-$\cos$ & train (s) \\
\midrule
100  & 0.13 & 9   & 0.874 & 0.993 & 0.126 & 0.763 & 0.969 & 0.237 & 0.061 & 34 \\
500  & 0.64 & 13  & 0.851 & 0.991 & 0.149 & 0.722 & 0.955 & 0.278 & 0.073 & 37 \\
1000 & 1.28 & 15  & 0.839 & 0.989 & 0.161 & 0.697 & 0.949 & 0.303 & 0.078 & 46 \\
2000 & 2.55 & 18  & 0.827 & 0.987 & 0.173 & 0.673 & 0.940 & 0.327 & 0.086 & 66 \\
3500 & 4.46 & 18  & 0.825 & 0.987 & 0.175 & 0.673 & 0.944 & 0.327 & 0.090 & 75 \\
\bottomrule
\end{tabular}
\end{table}

\begin{figure}[h]
\centering
\IfFileExists{figs/Fig_scaling_sa_vs_ddpm.pdf}{%
  \includegraphics[width=0.95\textwidth]{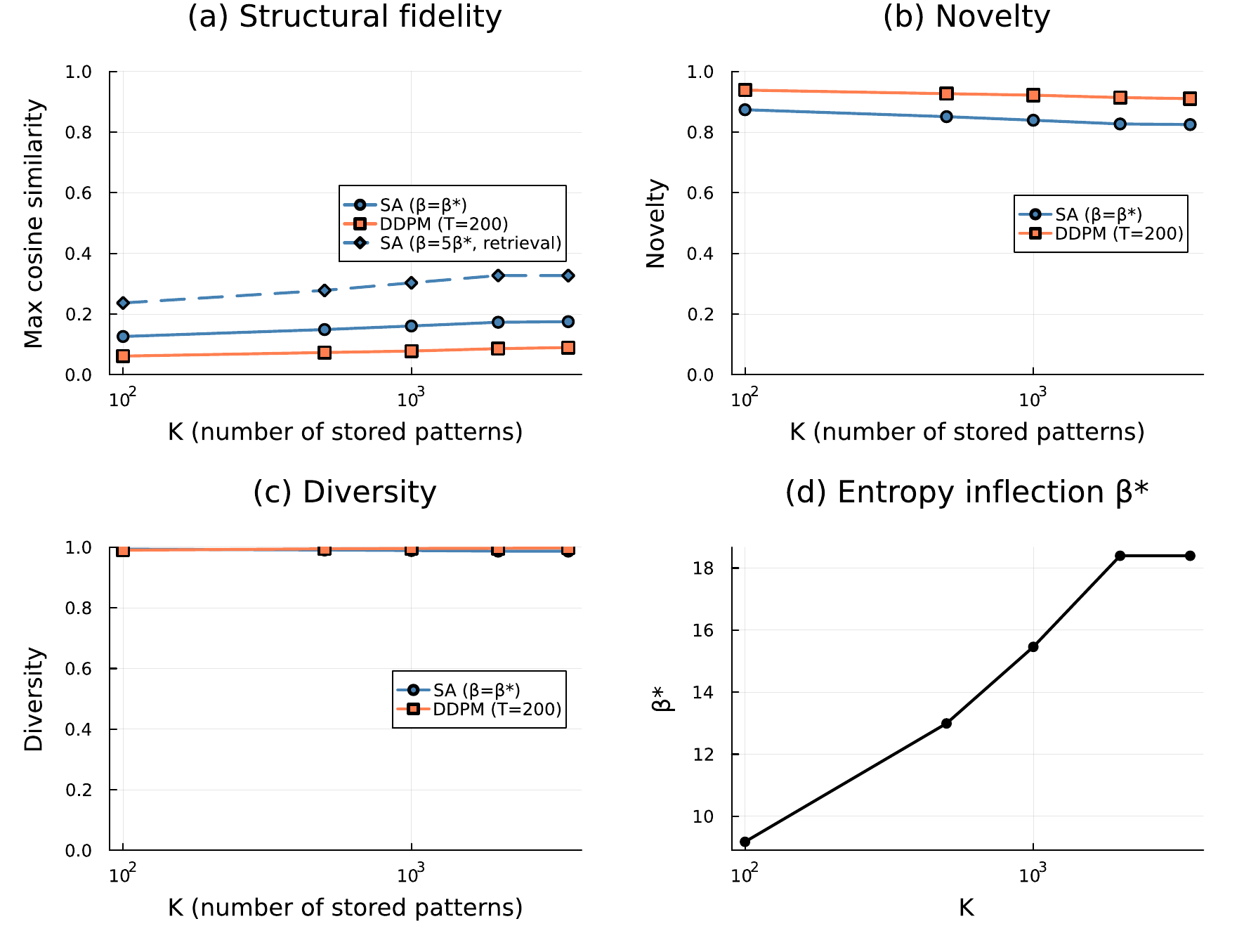}%
}{%
  \framebox[0.95\textwidth]{\rule{0pt}{4cm}\texttt{[Figure will be generated by run\_scaling\_experiment.jl]}}%
}
\caption{Scaling experiment: SA vs.\ DDPM on MNIST digit ``3'' as the number of stored patterns $K$ increases from 100 to 3{,}500. \textbf{(a)}~Max-cosine similarity: SA at both $\beta^*$ (generation) and $5\beta^*$ (retrieval) produces samples with meaningful structural similarity to stored patterns that increases with $K$; DDPM remains near the isotropic noise floor (${\approx}0.06$) at all $K$ tested. \textbf{(b)}~Novelty. \textbf{(c)}~Diversity. \textbf{(d)}~Entropy-inflection $\beta^*$ as a function of $K$.}
\label{fig:scaling}
\end{figure}

Three findings emerged from the scaling sweep. First, DDPM's max-cosine similarity never exceeded $0.09$ across the entire range $K\in[100,\,3{,}500]$, remaining close to the isotropic noise baseline of ${\approx}0.06$; the MLP denoiser could not learn the data distribution even with $3{,}500$ training images in $\mathbb{R}^{784}$. Second, SA's max-cosine similarity grew monotonically with $K$: from $0.126$ at $K{=}100$ to $0.175$ at $K{=}3{,}500$ in the generation regime, and from $0.237$ to $0.327$ in the retrieval regime. More stored patterns provided denser coverage of the data manifold, enabling SA to generate samples structurally closer to the training distribution. Third, the entropy-inflection $\beta^*$ increased with $K$ (from $9$ to $18$), consistent with the theoretical prediction that the transition temperature shifts as memory load grows. SA required no retuning: the entropy inflection provided a principled, automatic $\beta$ at each $K$.

We acknowledge that a convolutional U-Net DDPM operating on 2D images (rather than flattened vectors) would likely perform better, particularly at larger $K$. The comparison here is deliberately controlled: both methods operate on the same flattened, unit-norm representation. The point is not that SA is universally superior to diffusion models, but that SA's training-free score function provides a structural advantage in the low-data regime ($K\lesssim 10^3$) and that this advantage persists as $K$ grows within the range tested.

\section{MALA at Generation Temperature}\label{app:mala-generation}
% ─────────────────────────────────────────────────────────────────────────────

To confirm that the ULA discretization bias was negligible in the generation regime (not just the retrieval regime), we ran the full MALA comparison at $\beta\in\{200,\,500,\,1000\}$ with the same 30-chain protocol on MNIST digit ``3'' ($K{=}100$). The retrieval-regime comparison (Table~\ref{tab:mnist-baselines} at $\beta{=}2000$) had already shown $99.2\%$ acceptance and indistinguishable metrics; the question this section answered was whether the bias remained negligible as $\beta$ moved into the generation regime where the noise scale $\sqrt{2\alpha/\beta}$ is larger and the Langevin proposal is correspondingly more aggressive (Table~\ref{tab:mala-generation}).

\begin{table}[h]
\centering
\caption{ULA (SA) vs.\ MALA across the generation-to-retrieval spectrum. $\Delta$ columns report SA$-$MALA. All differences are negligible ($|\Delta|<0.02$), confirming that ULA bias does not affect any experimental conclusion.}
\label{tab:mala-generation}
\small
\begin{tabular}{@{}rccccccc@{}}
\toprule
$\beta$ & Accept & $\mathcal{N}_{\text{SA}}$ & $\Delta\mathcal{N}$ & $\bar{\mathcal{D}}_{\text{SA}}$ & $\Delta\bar{\mathcal{D}}$ & $\bar{E}_{\text{SA}}$ & $\Delta\bar{E}$ \\
\midrule
200  & 99.2\% & 0.548 & $+0.002$ & 0.885 & $+0.002$ & $+1.467$ & $+0.018$ \\
500  & 99.2\% & 0.375 & $+0.002$ & 0.782 & $+0.002$ & $+0.287$ & $+0.007$ \\
1000 & 99.2\% & 0.251 & $+0.002$ & 0.687 & $+0.002$ & $-0.107$ & $+0.004$ \\
\bottomrule
\end{tabular}
\end{table}

At $\beta{=}200$ (the generation regime), the deltas are: $\Delta\mathcal{N}{=}+0.002$, $\Delta\bar{\mathcal{D}}{=}+0.002$, $\Delta\bar{E}{=}+0.018$, with MALA acceptance rate $99.2\%$. At $\beta{=}1000$ (transition band), all 150 samples from both SA and MALA have negative energy ($E<0$, on the memory manifold), with even smaller deltas ($\Delta\bar{E}{=}+0.004$). The ULA bias monotonically decreases with increasing $\beta$ and is negligible across the entire spectrum at step size $\alpha{=}0.01$.

\section{Full-Data VAE Comparison}\label{app:full-mnist-vae}
% ─────────────────────────────────────────────────────────────────────────────

To test whether SA's advantage over the VAE is an artifact of the limited training data ($K{=}100$), we trained VAEs on substantially larger datasets and compared against SA using the same $K{=}100$ memory.

We trained two classes of VAE: (1)~a \emph{digit-3 VAE} trained on $N{=}1{,}000$ digit-3 images (${10\times}$ the SA memory); and (2)~an \emph{all-MNIST VAE} trained on $N{=}10{,}000$ images across all ten digit classes (${100\times}$ the SA memory). Both use a larger architecture than the $K{=}100$ baseline (encoder: $784\to 512\to 256$; decoder: symmetric), with latent dimensions swept over $\{8,\,16,\,32,\,64\}$ for the digit-3 VAE and $32$ for the all-MNIST VAE. Training followed the same two-phase protocol (AE warmup + KL annealing, 2{,}000 epochs each) with mini-batch training (batch size $256$). For the all-MNIST VAE, we generated $5{,}000$ samples and selected the $150$ with highest max-cosine similarity to the $K{=}100$ digit-3 memory (simulating class-conditional generation without an explicit classifier).

\begin{table}[h]
\centering
\caption{Comparison of SA against VAEs trained on $10{\times}$ to $100{\times}$ more data. All metrics computed against the same $K{=}100$ digit-3 memory matrix. MaxCos: mean max-cosine similarity to nearest stored pattern. On-manif: fraction of samples with Hopfield energy $E<0$.}
\label{tab:full-vae}
\small
\begin{tabular}{@{}llccccc@{}}
\toprule
Method & Train data & MaxCos $\uparrow$ & $\mathcal{N}$ $\uparrow$ & $\bar{\mathcal{D}}$ $\uparrow$ & $\bar{E}$ $\downarrow$ & On-manif \\
\midrule
VAE (lat${=}8$, original)  & $K{=}100$ & $0.775$ & $0.225$ & $0.426$ & $-0.276$ & $100\%$ \\
VAE-D3 (lat${=}8$)         & $N{=}1{,}000$ & $0.668$ & $0.332$ & $0.603$ & $-0.168$ & $94\%$ \\
VAE-D3 (lat${=}32$)        & $N{=}1{,}000$ & $0.743$ & $0.257$ & $0.466$ & $-0.244$ & $100\%$ \\
VAE-D3 (lat${=}64$)        & $N{=}1{,}000$ & $0.762$ & $0.238$ & $0.413$ & $-0.263$ & $100\%$ \\
VAE-all (lat${=}32$, top)  & $N{=}10{,}000$ & $0.830$ & $0.170$ & $0.325$ & $-0.330$ & $100\%$ \\
\midrule
\textbf{SA ($\beta{=}2000$)} & $K{=}100$ & $\mathbf{0.848}$ & $0.152$ & $\mathbf{0.600}$ & $-0.303$ & $100\%$ \\
SA ($\beta{=}200$)          & $K{=}100$ & $0.452$ & $\mathbf{0.548}$ & $\mathbf{0.885}$ & $+1.467$ & $0\%$ \\
\bottomrule
\end{tabular}
\end{table}

Three findings emerged from the full-data VAE comparison (Table~\ref{tab:full-vae}). First, SA at $\beta{=}2000$ (retrieval) achieved higher max-cosine similarity ($0.848$) than the best full-data VAE ($0.830$, all-MNIST with top-$150$ filtering), with $85\%$ higher diversity ($0.600$ vs.\ $0.325$), despite using $100{\times}$ less data and zero training. Second, the original $K{=}100$ VAE (lat${=}8$) achieved higher max-cos ($0.775$) than the $N{=}1{,}000$ VAEs with larger latent dimensions ($0.668$--$0.762$), because training on only $K{=}100$ patterns caused the VAE to memorize them; with more data, the VAE generalized but individual samples were less similar to specific stored patterns. Third, all full-data VAEs produced less diverse outputs ($0.325$--$0.603$) than SA at any temperature ($0.600$--$0.885$). SA's advantage was not an artifact of limited training data: the training-free score function provided a structural advantage in both fidelity and diversity.

% ═══════════════════════════════════════════════════════════════════════════════
\section{Principal Component Analysis}\label{app:pca}
% ═══════════════════════════════════════════════════════════════════════════════

Several experiments in this paper use Principal Component Analysis (PCA) for dimensionality reduction: the protein sequence experiment projects one-hot encodings from $\mathbb{R}^{1420}$ to $\mathbb{R}^{59}$ (Appendix~\ref{app:protein}), the GMM-PCA baseline operates in a 50-dimensional PCA subspace (Appendix~\ref{app:gmm-pca}), and the multi-head SA demo partitions PCA subspaces across heads (Appendix~\ref{app:multihead}). We briefly summarize the construction for completeness.

\paragraph{Setup.}
Given a memory matrix $\mathbf{X} = [\mathbf{m}_1,\dots,\mathbf{m}_K] \in \mathbb{R}^{d \times K}$, define the column mean $\bar{\mathbf{m}} = \frac{1}{K}\sum_{k=1}^{K}\mathbf{m}_k$ and the centered matrix $\tilde{\mathbf{X}} = \mathbf{X} - \bar{\mathbf{m}}\mathbf{1}_K^\top \in \mathbb{R}^{d \times K}$.

\paragraph{Singular value decomposition.}
The economy SVD of the centered matrix is
\begin{equation}\label{eq:pca-svd}
    \tilde{\mathbf{X}} = \mathbf{U}\boldsymbol{\Sigma}\mathbf{V}^\top,
\end{equation}
where $\mathbf{U} \in \mathbb{R}^{d \times r}$ has orthonormal columns (the \emph{principal directions}), $\boldsymbol{\Sigma} = \operatorname{diag}(\sigma_1,\dots,\sigma_r)$ with $\sigma_1 \ge \sigma_2 \ge \cdots \ge \sigma_r > 0$ (the \emph{singular values}), and $\mathbf{V} \in \mathbb{R}^{K \times r}$ has orthonormal columns. Here $r = \operatorname{rank}(\tilde{\mathbf{X}}) \le \min(d, K)$.

\paragraph{Variance interpretation.}
The $j$-th principal component of pattern $\mathbf{m}_k$ is $z_{kj} = \mathbf{u}_j^\top(\mathbf{m}_k - \bar{\mathbf{m}})$. The variance of the $j$-th component across the $K$ patterns is $\sigma_j^2 / K$, and the fraction of total variance captured by the first $p$ components is
\begin{equation}\label{eq:pca-variance}
    \rho(p) = \frac{\sum_{j=1}^{p}\sigma_j^2}{\sum_{j=1}^{r}\sigma_j^2}.
\end{equation}

\paragraph{Projection and reconstruction.}
To reduce dimensionality from $d$ to $p \le r$, define the projection matrix $\mathbf{W}_p = \mathbf{U}_{:,1:p}^\top \in \mathbb{R}^{p \times d}$ using the first $p$ principal directions. The projected representation of pattern $\mathbf{m}_k$ is $\mathbf{z}_k = \mathbf{W}_p(\mathbf{m}_k - \bar{\mathbf{m}}) \in \mathbb{R}^p$, and the approximate reconstruction is $\hat{\mathbf{m}}_k = \bar{\mathbf{m}} + \mathbf{W}_p^\top\mathbf{z}_k$. Among all rank-$p$ linear projections, PCA minimizes the mean squared reconstruction error $\frac{1}{K}\sum_{k=1}^{K}\|\mathbf{m}_k - \hat{\mathbf{m}}_k\|_2^2 = \frac{1}{K}\sum_{j=p+1}^{r}\sigma_j^2$.

\paragraph{Usage in this paper.}
In the protein experiment (Appendix~\ref{app:protein}), we select $p$ such that $\rho(p) \ge 0.95$ (retaining 95\% of variance), yielding $p{=}59$ from $d{=}1420$. After projection, the $p$-dimensional representations are $\ell_2$-normalized to form the memory matrix. In the multi-head SA experiment (Appendix~\ref{app:multihead}), we partition the principal directions across attention heads to assign each head a different variance subspace.

% ═══════════════════════════════════════════════════════════════════════════════
\section{Multi-Head Stochastic Attention}\label{app:multihead}
% ═══════════════════════════════════════════════════════════════════════════════

A natural question is whether SA extends to multi-head attention, where each head operates on a learned low-dimensional projection of the memory. Modern transformers run several attention heads in parallel because no single learned projection can simultaneously capture every relevant aspect of the data; the same intuition suggests that running independent Langevin chains in different subspaces should be more effective than a single chain in the full $d$-dimensional space, especially when the data lives near a low-dimensional manifold. We present a toy demonstration using PCA-partitioned subspaces as a proxy for learned projections, sidestepping the need to actually train projection matrices while preserving the architectural structure. In standard multi-head attention, each head $h \in \{1,\dots,H\}$ has learned projection matrices $\mathbf{W}_K^h, \mathbf{W}_Q^h \in \mathbb{R}^{d_{\mathrm{head}} \times d}$ that map keys and queries into a $d_{\mathrm{head}}$-dimensional subspace. Multi-head SA replaces the deterministic attention output with a Langevin sample from each head's energy:
\begin{equation}\label{eq:multihead-energy}
    E_h(\mathbf{z}) = \tfrac{1}{2}\|\mathbf{z}\|_2^2 - \tfrac{1}{\beta}\log\sum_{k=1}^{K}\exp\bigl(\beta\,(\mathbf{W}_K^h\mathbf{m}_k)^\top\mathbf{z}\bigr),
    \quad \mathbf{z} \in \mathbb{R}^{d_{\mathrm{head}}},
\end{equation}
where the projected memories $\mathbf{W}_K^h\mathbf{m}_k$ serve as the stored patterns in head $h$'s subspace. Independent Langevin chains run on each $E_h$, producing per-head samples $\mathbf{z}^{(h)}$. The outputs are concatenated and projected back to the original space via $\boldsymbol{\xi} = \mathbf{W}_O[\mathbf{z}^{(1)};\dots;\mathbf{z}^{(H)}]$, exactly as in standard multi-head attention.

Without learned projections, we use PCA to define meaningful subspaces. Given $K$ stored patterns with $r = \min(d, K)$ non-zero principal components (Appendix~\ref{app:pca}), we assign PCs $\{(h{-}1) \cdot r/H + 1,\,\dots,\,h \cdot r/H\}$ to head $h$. This gives each head a distinct slice of the data's variance structure: head 1 captures the dominant shape modes, while later heads capture progressively finer variation.

We use $K{=}100$ digit-3 images from MNIST ($d{=}784$), the same memory as all other MNIST experiments. With $r{=}100$ non-zero PCs, each head in an $H{=}4$ configuration receives 25 PCs. The sampling protocol is identical to the main experiments: 30 chains, 5{,}000 steps, 2{,}000 burn-in, thinned every 100 steps, 5 samples per chain (150 total). We also test whether simply increasing the number of stored patterns $K$ at fixed $H{=}1$ achieves the same effect.

\begin{table}[h]
\centering
\caption{Multi-head SA with PCA-partitioned subspaces versus single-head SA at larger memory sizes. All configurations at $\beta{=}200$ (generation regime). Metrics computed against each configuration's own memory. MaxCos: mean maximum cosine similarity to the nearest stored pattern.}
\label{tab:multihead}
\small
\begin{tabular}{@{}llccc@{}}
\toprule
Config & $K$ & $\mathcal{N}$ $\uparrow$ & $\bar{\mathcal{D}}$ $\uparrow$ & MaxCos $\uparrow$ \\
\midrule
SA ($H{=}1$) & 100  & $0.547 \pm 0.002$ & $0.885 \pm 0.002$ & $0.453 \pm 0.002$ \\
SA ($H{=}1$) & 500  & $0.548 \pm 0.003$ & $0.892 \pm 0.003$ & $0.452 \pm 0.003$ \\
SA ($H{=}1$) & 1000 & $0.548 \pm 0.002$ & $0.892 \pm 0.003$ & $0.452 \pm 0.002$ \\
\midrule
SA ($H{=}2$, 50 PCs/head)  & 100 & $0.175 \pm 0.006$ & $0.307 \pm 0.002$ & $0.825 \pm 0.006$ \\
SA ($H{=}4$, 25 PCs/head)  & 100 & $0.187 \pm 0.005$ & $0.291 \pm 0.002$ & $0.815 \pm 0.005$ \\
SA ($H{=}5$, 20 PCs/head)  & 100 & $0.196 \pm 0.005$ & $0.263 \pm 0.002$ & $0.804 \pm 0.005$ \\
\midrule
SA ($H{=}4$, 25 PCs/head) & 500  & $0.391 \pm 0.003$ & $0.475 \pm 0.002$ & $0.609 \pm 0.003$ \\
SA ($H{=}4$, 25 PCs/head) & 1000 & $0.469 \pm 0.004$ & $0.582 \pm 0.003$ & $0.531 \pm 0.004$ \\
\bottomrule
\end{tabular}
\end{table}

The PCA variance explained by each head ($H{=}4$, $K{=}100$) is: head~1 (PCs~1--25): 81.2\%, head~2 (PCs~26--50): 12.6\%, head~3 (PCs~51--75): 4.5\%, head~4 (PCs~76--100): 1.6\%. Head~1 captures the dominant digit shape, while later heads contribute stroke-width variation and fine detail.

Three findings emerged from the multi-head sweep (Table~\ref{tab:multihead}). First, multi-head SA at $K{=}100$ substantially improved fidelity at the same $\beta$: $H{=}4$ achieved max-cos~$= 0.815$, approaching single-head retrieval at $\beta{=}2000$ ($0.848$, Table~\ref{tab:full-vae}), while single-head at $\beta{=}200$ achieved only $0.453$. The mechanism is that each head runs Langevin dynamics in a lower-dimensional subspace ($d_{\mathrm{head}} \ll d$), where the same $\beta$ produces a sharper Boltzmann distribution due to higher signal-to-noise ratio; multi-head attention effectively shifts the retrieval-generation phase boundary. Second, increasing $K$ alone did not improve single-head SA at generation temperature: $K{=}1000$ gave identical metrics to $K{=}100$ (max-cos $= 0.452$ in both cases) because at $\beta{=}200$ the energy landscape is sufficiently flat that additional memories do not sharpen the softmax weights, confirming that the multi-head improvement was a genuine structural effect of dimensionality reduction rather than a consequence of the $K/d$ ratio. Third, the multi-head advantage diminished as $K$ grew per head: at $K{=}1000$ with $H{=}4$ (250 patterns per 25-dimensional head subspace), max-cos dropped to $0.531$, approaching the single-head value, consistent with the paper's capacity analysis (as $K/d_{\mathrm{head}}$ increases, the per-head energy landscape flattens and the dimensionality-reduction benefit erodes). In a learned multi-head architecture the projections would adapt to maintain separation even at large $K$, but the PCA-partitioned proxy cannot.

\end{document}